\documentclass{article}
\usepackage{suffix}

\usepackage[preprint]{neurips_2025}

\usepackage[utf8]{inputenc} 
\usepackage[T1]{fontenc}    
\usepackage{hyperref}       
\usepackage{url}            
\usepackage{booktabs}       
\usepackage{amsfonts}       
\usepackage{nicefrac}       
\usepackage{microtype}      
\usepackage{xcolor}         
\usepackage{subfig}
\usepackage{caption} 
\usepackage{graphicx}
\usepackage{booktabs} 
\usepackage{subcaption}

\usepackage{multirow}
\usepackage{wrapfig}

\usepackage{amsmath}
\usepackage{amssymb}
\usepackage{mathtools}
\usepackage{amsthm}
\usepackage{bbm}
\usepackage{enumitem}

\usepackage[capitalize,noabbrev]{cleveref}

\theoremstyle{plain}
\newtheorem{theorem}{Theorem}[section]
\newtheorem{proposition}[theorem]{Proposition}

\theoremstyle{definition}

\theoremstyle{remark}

\definecolor{teaser_red}{RGB}{184, 77, 77}
\definecolor{teaser_green}{RGB}{61, 158, 61}
\definecolor{teaser_blue}{RGB}{31, 119, 180}
\definecolor{appendix_pink}{RGB}{227, 119, 194}

\DeclareMathOperator{\rank}{\operatorname{rank}}

\newcommand{\softmax}{$\mathtt{softmax}$ }

\renewcommand{\a}{{\bf a}}

\newcommand{\m}{{\bf m}}

\newcommand{\s}{{\bf s}}

\def\Am{{\bf A}}
\def\Bm{{\bf B}}

\def\Gm{{\bf G}}

\def\Im{{\bf I}}
\def\Km{{\bf K}}
\def\Mm{{\bf M}}

\def\Sm{{\bf S}}
\def\Um{{\bf U}}
\def\Vm{{\bf V}}
\def\Wm{{\bf W}}
\def\Zm{{\bf Z}}
\def\Ym{{\bf Y}}
\def\Xm{{\bf X}}

\newcommand{\R}{{\mathbb{R}}}

\title{Unpacking Softmax: How Temperature Drives Representation Collapse, Compression and Generalization}

\author{%
  Wojciech Masarczyk\textsuperscript{1,*}
  \And
  Mateusz Ostaszewski\textsuperscript{1}
  \And
  Tin Sum Cheng\textsuperscript{2}
  \And
  Tomasz Trzciński\textsuperscript{1,4,5}
  \And
  Aurelien Lucchi\textsuperscript{2}
  \And
  Razvan Pascanu\textsuperscript{3}
}

\begin{document}

\footnotetext[1]{\small Warsaw University of Technology, Poland}
\footnotetext[2]{\small University of Basel, Switzerland}
\footnotetext[3]{\small Mila, Qu\'ebec AI Institute, Canada}
\footnotetext[4]{\small IDEAS Research Institute, Poland}
\footnotetext[5]{\small Tooploox, Poland}
\renewcommand*{\thefootnote}{\fnsymbol{footnote}}
\footnotetext[1]{\small Corresponding author: \texttt{wojciech.masarczyk@gmail.com}}

\maketitle

\begin{abstract}

The \softmax function is a fundamental building block of deep neural networks, commonly used to define output distributions in classification tasks or attention weights in transformer architectures. Despite its widespread use and proven effectiveness, its influence on learning dynamics and learned representations remains poorly understood, limiting our ability to optimize model behavior.
In this paper, we study the pivotal role of the \softmax function in shaping the model's representation. We introduce the concept of \emph{rank deficit bias} — a phenomenon in which softmax-based deep networks find solutions of rank much lower than the number of classes. This bias depends on the \softmax function's logits norm, which is implicitly influenced by hyperparameters or directly modified by \softmax temperature. Furthermore, we demonstrate how to exploit the \softmax dynamics to learn compressed representations or to enhance their performance on out-of-distribution data.
We validate our findings across diverse architectures and real-world datasets, highlighting the broad applicability of temperature tuning in improving model performance. Our work provides new insights into the mechanisms of \softmax, enabling better control over representation learning in deep neural networks.

\end{abstract}

\section{Introduction}

The \softmax function defined in Equation~\ref{eq:softmax_with_temp}, with a hyperparameter $T>0$ as the \emph{temperature}, is a cornerstone of deep learning and is primarily used in tasks such as classification and text generation to transform raw model outputs into probability distributions with maximum entropy. By amplifying the largest values and diminishing smaller ones, \softmax enables models to make confident predictions, making it essential for decision-making among multiple options. However, its "winner-takes-all" nature \cite{liu1999winner,elfadel1993softmax, Peterson89} can sometimes lead to training challenges \cite{darcet2023vision, hoffmann2023softmax, shen2023softmax,zhai2023softmax}.

\begin{equation}\label{eq:softmax_with_temp}
\operatorname{softmax}_T(\mathbf{e})=\left[\begin{array}{llc}
\frac{\exp \left(e_1 / T\right)}{\sum_k \exp \left(e_k / T\right)} & \cdots & \frac{\exp \left(e_n / T\right)}{\sum_k \exp \left(e_k / T\right)}
\end{array}\right]
\end{equation}

Originally, the \softmax function was used primarily as the final layer in classification tasks to produce normalized class probabilities. However, with the introduction of the self-attention mechanism in transformer models \cite{vaswani2017attention}, \softmax became central to internal computations—particularly in attention blocks used for text generation, and later, for image classification \cite{dosovitskiy2020image}. Despite various alternatives proposed for the attention mechanism \cite{saratchandran2024rethinking, ramapuram2024theory}, \softmax-based attention remains the dominant approach in modern architectures.

This enduring popularity has spurred extensive research into the implicit trade-offs of using the \softmax activation, such as its impact on the sharpness of output distributions \cite{velickovic2024softmax} and its role in normalizing activations and gradients \cite{saratchandran2024rethinking}. Yet, the deeper effects of \softmax on model behavior—particularly its influence on learned representations and generalization—remain poorly understood. This gap motivates our central research question: 

\begin{center}
\textit{How does \softmax shape neural network representations?}    
\end{center}

Our answer to this question unfolds into four main contributions of our work, visualized in Figure~\ref{fig:teaser}:

\begin{enumerate}
    \item We introduce \emph{rank-deficit bias}—a novel phenomenon where the \softmax biases learning to converge to solutions whose representation has rank much lower than the number of classes. This departs from the behavior predicted by Neural Collapse~\cite{papyan2020neuralcollapse}, a framework that describes the emergence of highly structured, full-rank representations at convergence.
    
    \item We show that the logit's norm at initialization drives the difference in network behavior, influencing representations' compactness, out-of-distribution generalization, and out-of-distribution detection.

    \item We illustrate that the \softmax temperature directly changes logits norm and gives direct control over the training trajectory, which can be used to increase the model's OOD performance or compress the model's size.

    \item Beside the \softmax temperature, we identify other hyperparameters that play a similar role. This observation gives practitioners more flexibility to control the outcome of the learning process.

\end{enumerate}

\section{Unpacking Softmax}

\begin{figure}[!t]
    \centering
    {
    \includegraphics[width=0.48\textwidth]{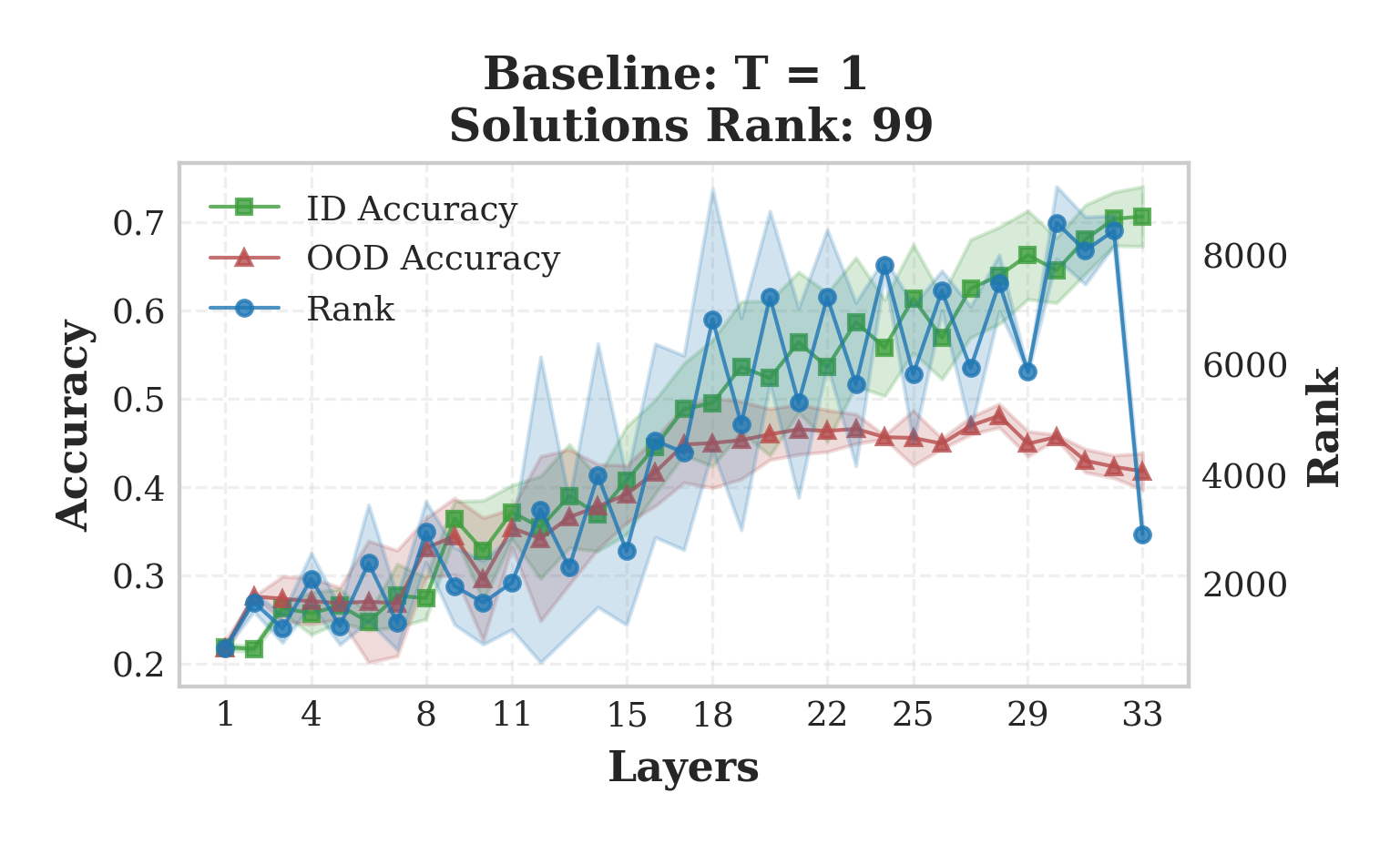}
    }
    {
      \includegraphics[width=0.48\textwidth]{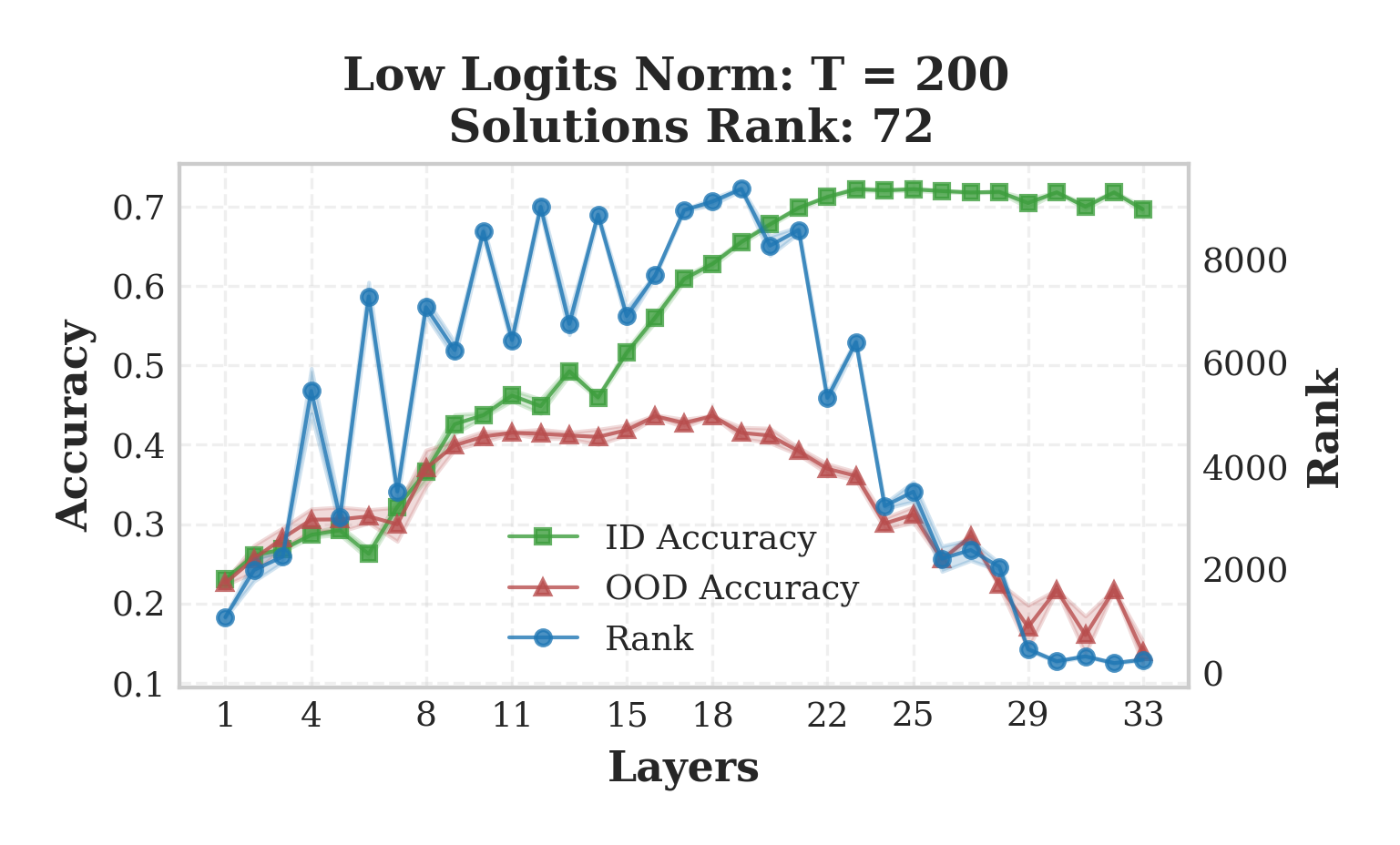}
    }
    \caption{\looseness-1 \small \textbf{Logits norm shapes the representations of neural networks} leading to the best performance achieved at earlier layers for \textcolor{teaser_green}{in distribution data} when trained with low logits norm (right) compared to the baseline (left). These condensed representations trigger \textcolor{teaser_blue}{representation collapse}, which harms generalization on \textcolor{teaser_red}{out-of-distribution data}. The plot presents the accuracy of linear probes attached to ResNet-34 trained on CIFAR-100 with different \softmax temperatures, the OOD dataset is SVHN. Analogous plots with different architectures and datasets can be found in the Appendix~\ref{app:remaining_experiments}.}
    \label{fig:teaser}
\end{figure}

The \softmax function is a ubiquitous component of modern neural networks, yet its role in shaping learned representations—particularly through its temperature parameter—remains underexplored. In this section, we investigate how variations in the logits norm at initialization, induced by the \softmax temperature, systematically affect representation learning and generalization. We present empirical evidence showing that lower logits norms lead to more compressed internal representations, reflected in reduced effective depth and rank, and, surprisingly, worse generalization to out-of-distribution (OOD) data. To quantify these effects, we introduce a suite of diagnostic metrics and evaluate them across diverse architectures and datasets. These observations lay the empirical foundation for our theoretical analysis in Section~\ref{sec:DNN_analysis}.

\subsection{Empirical evidence}

To thoroughly examine the phenomenon, we designed a setup to include the most commonly used architectures and datasets, and designed several metrics to quantify its strength.

\textbf{Architectures and Datasets}
We use 4 different architecture groups: MLP, VGG, ResNet, and VIT, trained on a diversified set of datasets: CIFAR-10, CIFAR-100, ImageNet-100, and ImageNet-1k, to assess how broadly the phenomenon applies. The details about architectures, datasets, and hyperparameters can be found in the Appendix~\ref{app:architectures}.

To quantitatively measure the effect, we define the following estimates: effective depth, OOD generalization loss, and solutions rank based on the linear probing accuracy and numerical rank. 

\textbf{Effective depth ($\kappa$):} measures the ratio between the first layer of the model's that achieves at least 99\% of the last layer's model accuracy using linear probing and the total number of layers in the model. The smaller the value, the more efficient the model.

\textbf{OOD generalization loss ($\rho$):} is defined as the normalized difference between the best OOD accuracy (across the layers) and the OOD accuracy at the final layer using linear probes. The value is non-negative, and the smaller the value, the better.

\textbf{Solutions rank (SR):} We compute the numerical rank of the pre-softmax representations matrix (logits matrix). Using the spectrum of the matrix, we estimate the numerical rank of the given representation matrix as the number of singular values above a threshold $\gamma$.~\footnote{We use the default value of this hyperparameter proposed in PyTorch~\cite{paszke2019pytorch}.}

We begin our analysis by examining how varying the logits norm at initialization affects training dynamics and representation quality. Table~\ref{tab:main-table} compares baseline architectures with their low-logit-norm counterparts, revealing clear shifts in representation structure across different datasets and model families. Even when final classification accuracies are comparable, models trained with lower logits norm exhibit systematically lower solution rank, reduced effective depth ($\kappa$), and increased OOD generalization loss ($\rho$).

These patterns highlight a trade-off: lower logits norm promotes more compressed internal representations—suggesting more efficient use of model capacity—but also leads to performance degradation on OOD data, especially in deeper layers. This empirical evidence sets the stage for the theoretical framework developed in Section~\ref{sec:DNN_analysis}, where we explain why the \softmax function enables such behavior.

\begin{table}[!h]
\centering
\resizebox{0.85\columnwidth}{!}{%
\begin{tabular}{@{}cccccccc@{}}
\toprule
\multirow{2}{*}{Dataset} & \multirow{2}{*}{Architecture} & \multicolumn{3}{c}{Baseline} & \multicolumn{3}{c}{Low Logit Norm} \\ \cmidrule(l){3-8} 
                             &           & $\kappa$ $\downarrow$ & $\rho$ $\downarrow$ & SR & $\kappa$ $\downarrow$ & $\rho$ $\downarrow$ & SR \\ \midrule
\multirow{3}{*}{CIFAR-10}     & MLP       &    75\%   &  67\%   &   9   &    38\%   &  78\%    &   9   \\
                             & ResNet18  &   100\%    &  29\%   &  9     &   67\%    & 54\%    &   9   \\
                             & ResNet20  &  100\%     &  10\%   &   9   &   75\%    &  34\%   &   9   
                             \\ 
                             \midrule
\multirow{2}{*}{CIFAR-100}    & ResNet18  &  100\%     &  19\%   &   99   &    83\%   &  50\%   &   59   \\
                             & ResNet34  &   100\%    &  16\%   &   99   &    68\%   &   51\%  &   72   \\ 
                             \midrule
\multirow{2}{*}{ImagNet-100}  & ResNet34  &   100\%    &   6\%  &  99    &   85\%    &    52\% &   42   \\
                              & ResNet50  &    100\%  &   5\%  &    99  &   78\%    &  50\%   &   27   \\ \midrule
\multirow{2}{*}{ImageNet-1k} & ResNet34  &    100\%   &   5\%  &    512\textsuperscript{\textdagger} &    82\%   &  23\%   &     122   \\
                             & ResNet50  &    100\%   &   5\%  &   947   &    78\%   &  22\%   &    128    \\
                             \bottomrule
\end{tabular}%
}
\caption{Training the models with different logits norm at initialization leads to models with different properties. Lower logits norm leads to more compressed representations ($\kappa$), decreased out-of-distribution generalization ($\rho$), and lower solutions rank (SR).  \textsuperscript{\textdagger} shows when the rank is bounded by the size of the classifier. Fine-grained results in Appendix~\ref{app:more_temperatures}.}
\label{tab:main-table}
\end{table}

\section{Analysis}\label{sec:DNN_analysis}

This section demonstrates how low logit norms lead to \emph{rank-deficit bias} for MLP (CIFAR-10) and ResNet-34 (CIFAR-100), with additional experiments in Appendix~\ref{app:remaining_experiments}.  

\paragraph{Notation} A network \( f_{\theta}: \mathcal{X} \to \mathbb{R}^c \) maps inputs \(\mathbf{x}\) to logits \(\mathbf{e}\) over \(c\) classes, with probabilities computed via \(\operatorname{softmax}_T(\mathbf{e})\) with temperature $T$. The logits matrix \(\mathbf{M} = \mathbf{W}^L\mathbf{A}^L\) of an \(L\)-layer network combines final weights \(\mathbf{W}^L \in \mathbb{R}^{c \times d}\) and penultimate representations \(\mathbf{A}^L \in \mathbb{R}^{d \times n}\). Here, \(\|\mathbf{M}\|\) denotes the Frobenius norm, and \(\sigma_1(\mathbf{A})\) is the top singular value of $\Am$.  

\subsection{Training dynamics}

\paragraph{Logits norm and \softmax temperature equivalence} 

The \softmax function with temperature \( T > 0 \) (Equation~\ref{eq:softmax_with_temp}) is functionally equivalent to applying \softmax with temperature 1 to logits scaled by $1/T$. Conversely, logits norm acts as an inverse temperature for normalized logits.~\footnote{
$\operatorname{softmax}_T(\mathbf{e}) = \operatorname{softmax}_1\left(\mathbf{\frac{e}{T}}\right) = \operatorname{softmax}_{\frac{1}{\|\mathbf{e}\|}}\left(\mathbf{\frac{e}{\|\mathbf{e}\|T}}\right)$} Thus, the effects observed in our work can arise either from explicit high-temperature training or implicit architectural logit norm reduction, as detailed in Section~\ref{sec:architecture_choices}.

\textbf{High temperature induces loss symmetry} 

Increasing the \softmax temperature flattens output distributions, maximizing entropy and approaching uniform class probabilities asymptotically~\cite{velickovic2024softmax}. This sample uniformity creates a symmetric loss landscape, where each sample incurs a CrossEntropy~\footnote{The same holds for MSE that also exhibit \textit{rank-deficit bias} when paired with \softmax (see Appendix~\ref{app:mse_results}).} loss of $\ln(\frac{1}{n})$ for $n$ classes. This symmetry might hinder learning by suppressing gradient diversity~\cite{ziyin2024symmetryinducesstructureconstraint}.

\paragraph{Breaking loss symmetry requires increasing logits norm}
\begin{wrapfigure}[16]{r}{0.5\textwidth}
    \centering
        \vspace{-0.6cm}
        \includegraphics[width=0.5\textwidth]{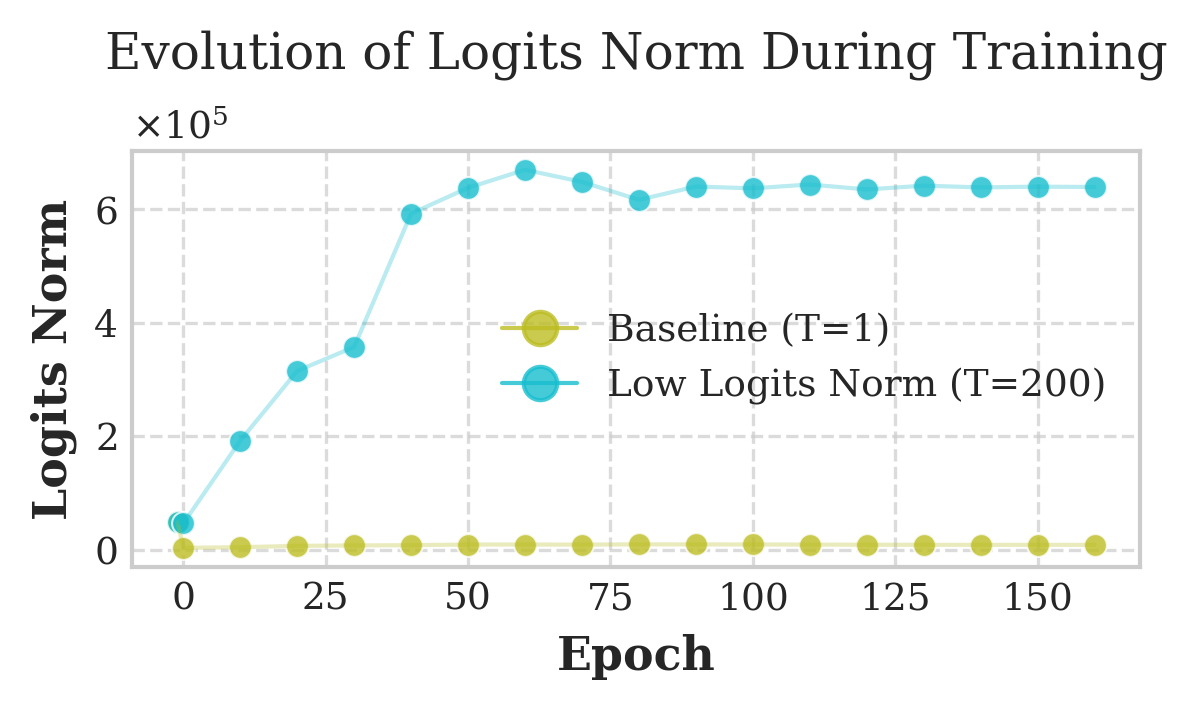}
        \caption{\looseness-1 \small High-temperature training forces networks to increase logits norm to escape the loss symmetry induced by the \softmax temperature compared to baseline. Experiment: ResNet-34 on CIFAR-100. Additional experiments in Appendix~\ref{app:growing_logits_norm}. \label{fig:logits_norm_over_training}}
\end{wrapfigure}

To minimize the loss, the model must break symmetry by reducing \softmax entropy. This can be achieved either by decreasing the temperature (as in~\cite{velickovic2024softmax}) or increasing the norms of the logits. With fixed temperature during training, networks must grow the logits' norms to escape symmetry. Figure~\ref{fig:logits_norm_over_training} validates this claim, showing rapid growth of the logits' norm when trained with high temperature compared to baseline. 

\paragraph{Two mechanisms for logit norm growth} 

The logits matrix $\Mm = \Wm\Am \in \mathbb{R}^{c \times n}$ decomposes into the last weight matrix $\Wm \in \mathbb{R}^{c \times d}$ and penultimate representations $\Am \in \mathbb{R}^{d \times n}$, where $c$, $n$, and $d$ denote class count, sample size, and representation dimension respectively\footnote{Our analysis generalizes to convolutional layers via tensor reshaping (see Appendix~\ref{app:pabs})}. Taking the SVD of $\Wm$ and $\Am$, we obtain:

\begin{equation}
\begin{aligned}
\|\Mm\| = \|\Wm\Am\| 
 & = \|\Um_{\Wm}\Sigma_{\Wm}\Vm^{\top}_{\Wm}\Um_{\Am}\Sigma_{\Am}\Vm^{\top}_{\Am}\| \\
 & = \|\Sigma_{\Wm}\Vm^{\top}_{\Wm}\Um_{\Am}\Sigma_{\Am}\| \leq \|\Sigma_{\Wm}\Sigma_{\Am}\|,
\end{aligned}
\end{equation}

where the orthonormal SVD components satisfy $\Wm = \Um_{\Wm}\Sigma_{\Wm}\Vm^{\top}_{\Wm}$ and $\Am = \Um_{\Am}\Sigma_{\Am}\Vm^{\top}_{\Am}$. The second equality holds because singular vectors are orthonormal. Then, the norm grows through:
\begin{enumerate}
    \item \textit{Singular value scaling}: Increasing $\Sigma_{\Wm}$ and $\Sigma_{\Am}$ diagonal entries.
    \item \textit{Singular vector alignment}: Maximizing $\|\Mm\|$ when $\Vm^{\top}_{\Wm}\Um_{\Am} = \Im$, achieving maximal alignment of the singular vectors.
\end{enumerate}

To determine whether the alignment is exploited during training, we track alignment by measuring cosine similarity between the top-15 singular vectors of $\Wm^i$ and $\Am^i$ across layers $i=1,\ldots, L$. Figure~\ref{fig:alignment} (top) shows high-temperature models develop strong alignment early in the training, while baseline models exhibit negligible alignment throughout the whole training.

\paragraph{Partial alignment causes representation collapse}

\begin{wrapfigure}[33]{r}{0.54\textwidth}
    \centering
        \includegraphics[width=0.54\textwidth]{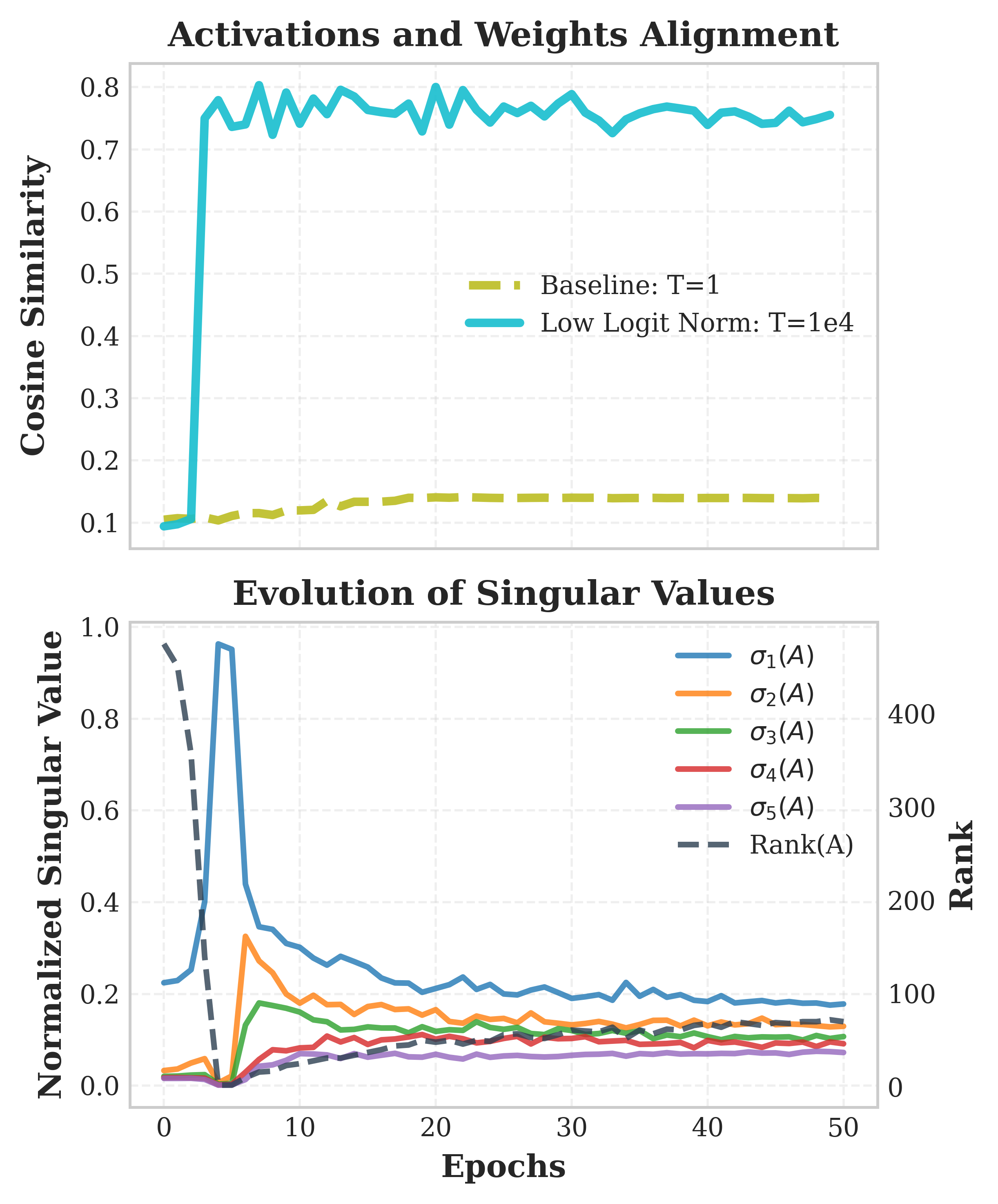}
        \caption{\looseness-1 \small Strong alignment forms during initial training epochs for the model trained with high temperature, leading to representation collapse caused by top singular values dominating the spectrum. \textbf{Top:} maximum alignment (cosine similarity) between the top-15 singular vectors of weights and representations averaged across all layers for the network. \textbf{Bottom:} Experiment: MLP trained on CIFAR-10. Fine-grained results presented in Figure~\ref{fig:finegrained_alignment}; additional experiments in the Appendix~\ref{app:pabs}.\label{fig:alignment}}
\end{wrapfigure}

The alignment phenomenon occurs simultaneously across multiple layers (Figure~\ref{fig:alignment} (top)), creating a compounding effect that exponentially increases representations' norm with network depth. For maximally aligned networks, the top singular value of the final logits ($\sigma_{1}^{L}$) becomes the product of each layer's top singular values:
\[
\sigma_{1}^{L} = \prod_{i=1}^{L-1} \sigma_{1}^{i} \geq (\sigma_{1}^{k})^L,
\]
where $\sigma_{1}^{i}$ is layer $i$'s top singular value and $k$ identifies the layer with the smallest $\sigma_{1}^{i}$. Crucially, when $\sigma_{1}^{k} > 1$, this creates exponential growth with depth $L$. However, this growth is highly selective - initially, only the top singular vector across layers aligns (Figure~\ref{fig:alignment} (top)) to be later joined by a few additional ones (Figure~\ref{fig:finegrained_alignment}), while others remain nearly orthogonal, creating an imbalance in learning dynamics.

This imbalanced alignment leads to representation collapse, as evidenced in Figure~\ref{fig:alignment} (bottom). Early in training, only the top singular value grows significantly, dominating the spectrum. Around epoch 8, two additional values begin growing, coinciding with the emergence of strong alignment among the top three singular vectors (Figure~\ref{fig:finegrained_alignment}). We quantify this collapse using numerical rank: $\rank(\Am) \coloneqq \Sigma_{i=1}^{n} \mathbbm{1}_{\sigma_i(\Am) > \gamma}$, where $\gamma$ is typically set relative to $\sigma_1(\Am)$. As top singular values grow exponentially (dashed line in Figure~\ref{fig:alignment} (bottom)), the numerical rank drops sharply, demonstrating clear representation collapse in the learned features.

\paragraph{Representations collapse triggers gradient collapse}
Representation collapse creates a feedback loop in neural networks. The initial collapse stems from uneven learning speeds across different singular values. But once representations collapse, they restrict gradient diversity through a fundamental rank limitation, making recovery even harder. To understand this mechanism, consider layer representations $\Am^{i} = \phi(\Wm^{i}\Am^{i-1})$. The gradient rank at each layer obeys a crucial bound:

\begin{equation}\label{eq:gradient_rank_bound}
    \rank\left(\frac{\partial \mathcal{L}}{\partial \Wm^{i}}\right)  = \rank\left(\frac{\partial \mathcal{L}}{\partial \Zm^{i}}\Am^{i-1}\right) \leq \min\left(\rank\left(\frac{\partial \mathcal{L}}{\partial \Zm^{i}}\right), \rank\left(\Am^{i-1}\right)\right) \leq \rank(\Am^{i-1}),
\end{equation}

meaning collapsed representations directly limit gradient diversity. While this bound uses exact rank, our experiments with numerical rank in Figure~\ref{fig:activations_gradients_ranks} confirm the effect: high-temperature training (right) shows simultaneous collapse in both representations and gradients, unlike normal training (left).

\begin{figure}[!h]
    \centering
    {
    \includegraphics[width=0.48\textwidth]{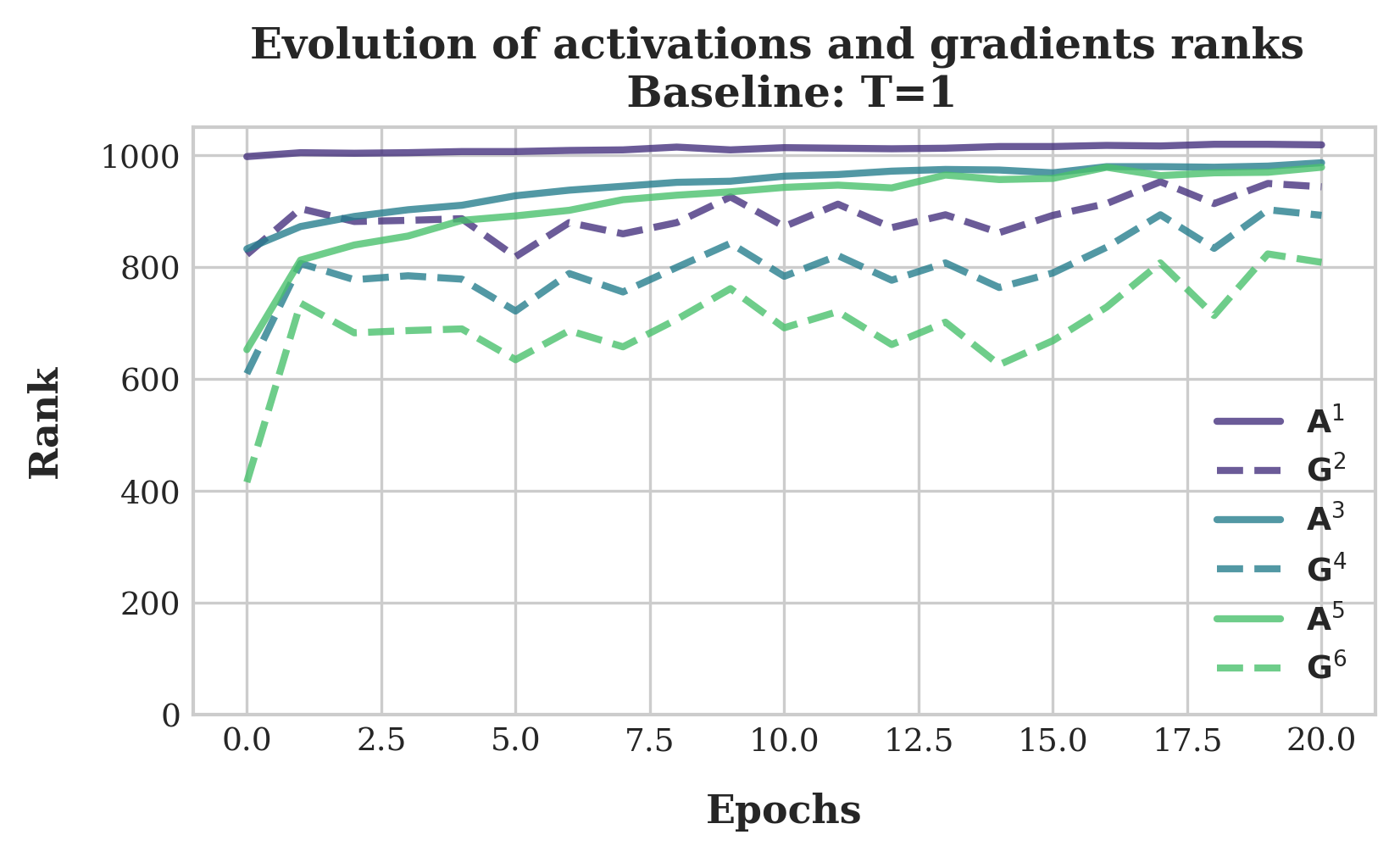}
    }
    {
    \includegraphics[width=0.48\textwidth]{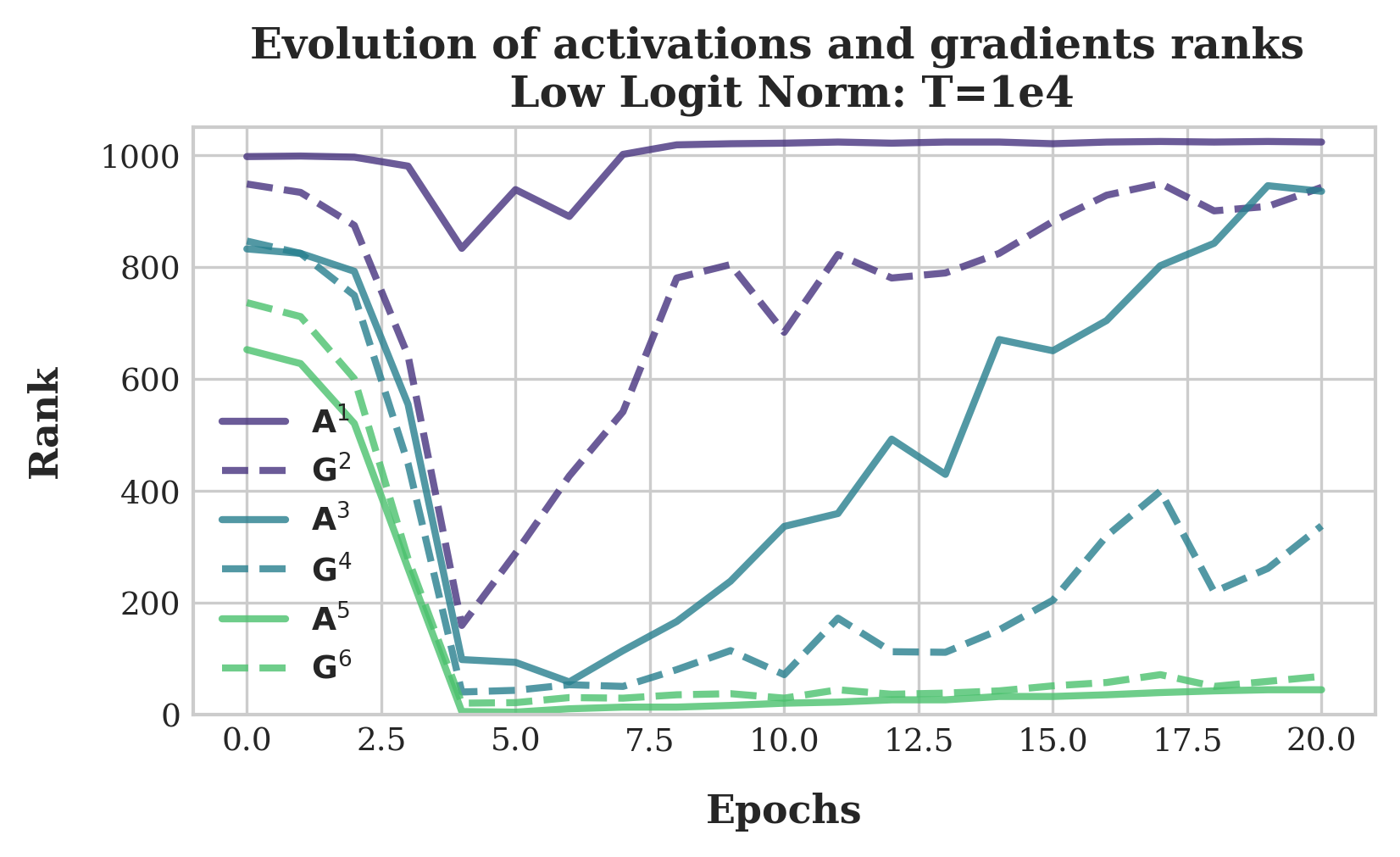}
    }
    \caption{\looseness-1 \small Rank of activations $\Am^{i}$ bounds the rank of the gradients $\Gm^{i+1}$, leading to low-rank gradients at deeper layers when trained with high temperature (right). Experiment: MLP trained on CIFAR-10.
}
    \label{fig:activations_gradients_ranks}
\end{figure}

\paragraph{Summary}
Low logit norm or high-temperature training triggers a cascade of effects leading to representation collapse, demonstrating the fundamental role of \softmax temperature in shaping neural networks' representations.

We now bridge our observations to Neural Collapse~\cite{papyan2020neuralcollapse}, a broader framework predicting the rigid geometric structure of learned representations. We demonstrate how high-temperature training leads to a novel \textit{rank-deficit bias} that fundamentally alters the expected geometric structure of learned representations. The extended comparison can be found in the Appendix~\ref{app:neural_collapse}.

\subsection{From Neural Collapse to Rank-Deficit Bias}

\begin{wrapfigure}[14]{r}{0.45\textwidth}
    \centering
        \vspace{-1.3cm}
        \includegraphics[width=0.45\textwidth]{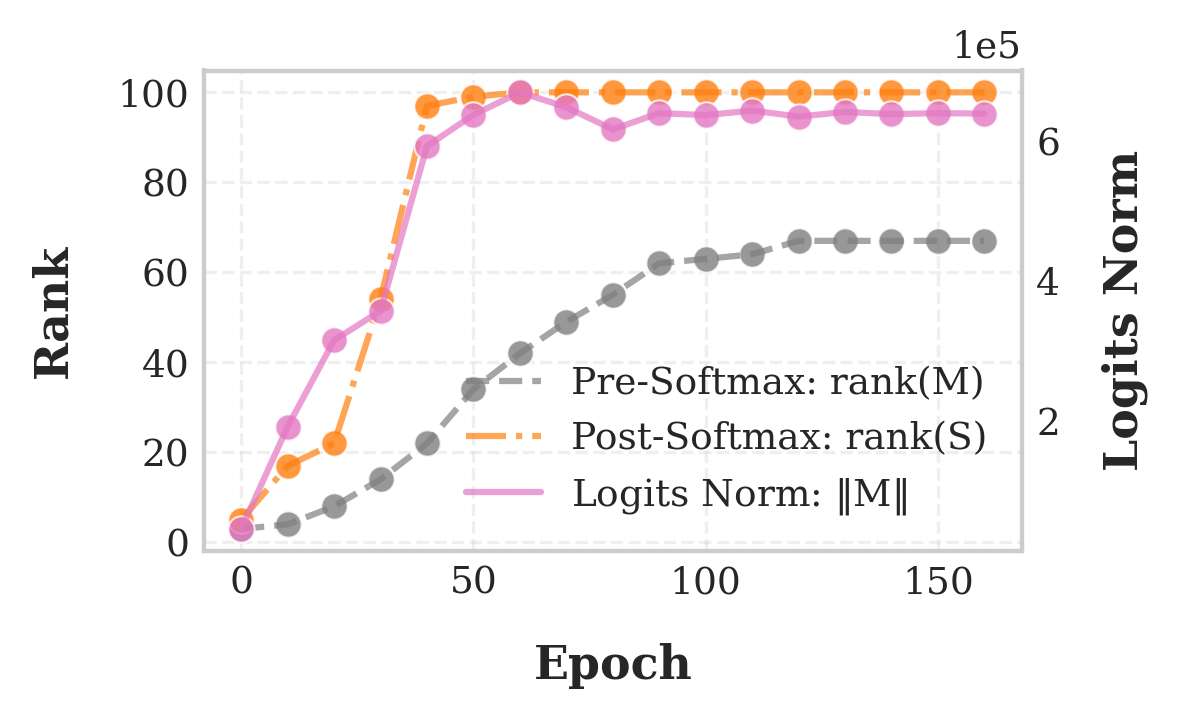}
        \caption{\looseness-1 \small Evolution of logits rank before and after applying a \softmax and logits norm before applying the \softmax function. Experiment: ResNet-34 trained on CIFAR-100.\label{fig:pre_post_softmax_rank_evolution}}
\end{wrapfigure}

In supervised classification, a neural network $f_{\theta}(\Xm) = \operatorname{softmax}_T(\Mm)$ must approximate target outputs $\Ym \in \mathbb{R}^{c\times n}$. Typically, $\rank(\Ym) = c$ since $n \gg c$. The Neural Collapse (NC) phenomenon~\cite{papyan2020neuralcollapse} reveals that successful training leads to a specific geometric structure where $\rank(\Mm) = c-1$ (see Appendix~\ref{app:neural_collapse}). Subsequent works have shown that NC solutions influence generalization, robustness, and transfer learning~\cite{kothapalli2023neuralcollapsereviewmodelling,munn2024impactgeometriccomplexityneural,papyan2020neuralcollapse}.

Our key discovery breaks from this established pattern: high-temperature training induces \textit{rank-deficit bias}, where $\rank(\Mm) \ll c-1$. To understand this paradox, we track $\|\Mm\|$, $\rank(\Mm)$, $\rank(\Sm)$, where $\Sm = \operatorname{softmax}_T(\Mm)$, and  during training.

\begin{wrapfigure}[16]{r}{0.45\textwidth}
    \centering
        \vspace{-1.2cm}
        \includegraphics[width=0.45\textwidth]{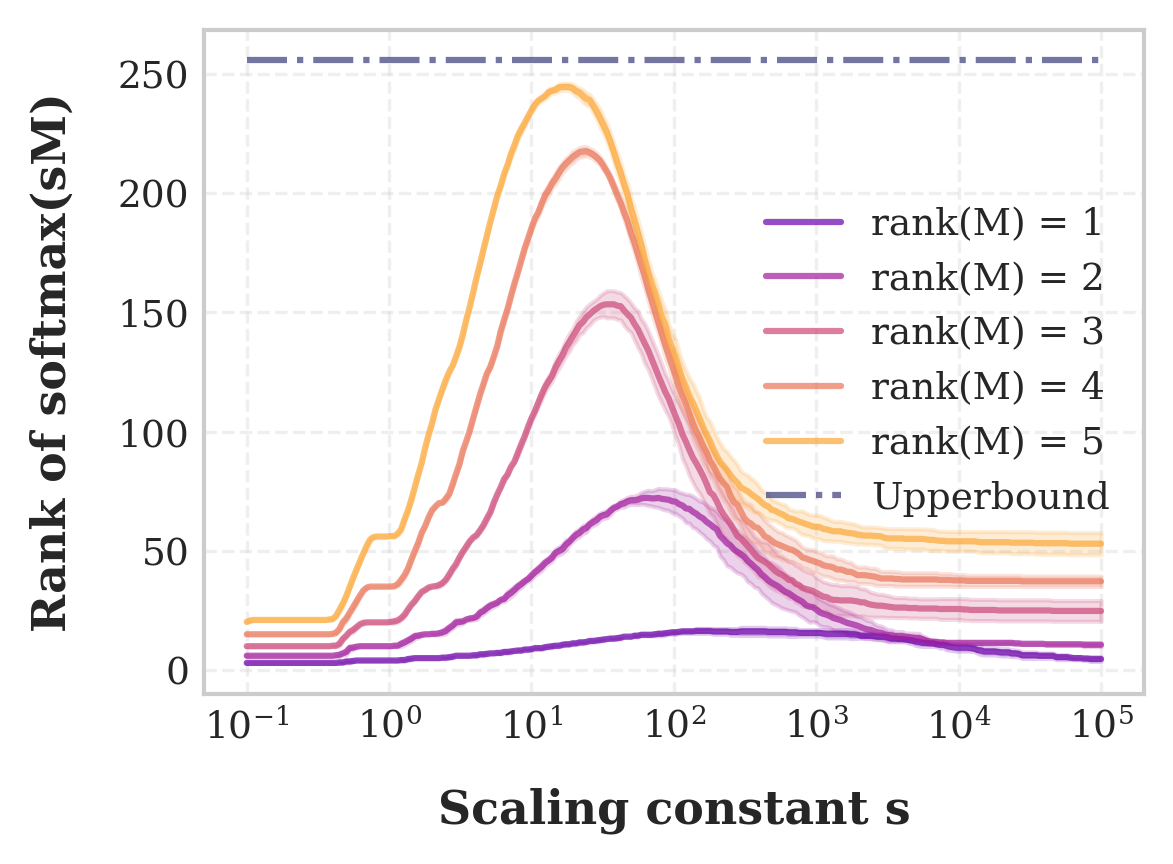}
        \caption{\looseness-1 \small The evolution of post-softmax rank of random matrices $\Mm$ scaled by a constant $c$ (inverse of temperature). Each matrix $\Mm \coloneqq \Am\Bm^\top \in \mathbb{R}^{n \times n}$, where $\Am, \Bm \in \mathbb{R}^{n\times k}$, and $k = 1, \ldots, 5$, making $\Mm$ low-rank. The elements $a_{ij}, b_{ij} ~\sim \mathcal{U}(-1, 1)$.\label{fig:softmax_scaling}}
\end{wrapfigure}

Figure~\ref{fig:pre_post_softmax_rank_evolution} reveals a pattern: while both ranks start small, the post-softmax rank quickly saturates while the pre-softmax rank remains low. The sudden growth of post-softmax rank aligns with the growth of the logits norm. Is this norm-rank relationship fundamental? To test this, we randomly generate low-rank matrices $\Mm \in \mathbb{R}^{n \times n}$ and compute their post-softmax rank when increasing their norms by scaling them with a constant $s$. Since multiplying a matrix by a scalar does not change its rank, the pre-softmax rank stays the same throughout the whole experiment; however, Figure~\ref{fig:softmax_scaling} reveals the key insight: Increasing $\|\Mm\|$ alone boosts $\rank(\operatorname{softmax}(\Mm))$, even when $\rank(\Mm)$ stays constant.

To formally analyze this effect, we propose the following upper bound on the difference between top ($\sigma_1(\Sm)$) and bottom ($\sigma_n(\Sm)$) singular values of matrix $\Sm$, where $\Sm = \operatorname{softmax}(\Mm)$ for any real $\Mm \in \mathbb{R}^{n \times n}$:

\begin{proposition} \label{proposition:B2}
Consider any matrix $\Sm\in\R^{n\times n}$ with columns $\s_j\in\R^n$ being probability vectors. Then the gap between the largest $\sigma_1(\Sm)$ and the smallest singular value $\sigma_n(\Sm)$ is bounded by the following tight inequality:
\[
0 \leq \sigma_1(\Sm) - \sigma_n(\Sm) \leq \sqrt{1+r}-\sqrt{\max\left\{\frac{1}{n}-r,0\right\}}
\]
where $r:= \max_i \sum_{j\neq i}\langle \s_i,\s_j\rangle$.
\end{proposition}

The bound shows that the rank of a post-softmax matrix depends on the mutual similarity between the columns of the softmaxed logits, which increases as the logits norm increases, boosting the rank of $\Sm$. The proof is in the Appendix~\ref{app:softmax_analysis}.

Given the \softmax’s surprising ability to amplify rank and its connection to high-temperature training and \textit{rank-deficit bias}, we investigate the theoretically minimal rank that still permits full-rank post-softmax outputs. In Appendix~\ref{app:rank2_to_full_rank}, we prove that the minimal rank is 2 irrespective of the number of classes in the dataset. However, as Table~\ref{tab:main-table} shows, empirical results remain far from this theoretical limit. We encourage further research into training methods that push toward lower-rank solutions, given their direct impact on OOD performance explored in the next Section.

\paragraph{Summary} \softmax enables \textit{rank-deficit bias} by mapping low-rank inputs to full-rank outputs via norm amplification. However, the rank of solutions found by training is still much higher than the theoretical limit.

We now turn our attention to quantifying the consequences of the \textit{rank-deficit bias}, showing that the \softmax temperature has a direct impact on the model's performance on downstream OOD tasks.

\section{Generalization is at odds with detection}\label{sec:generalization_vs_detection}

\begin{wrapfigure}[13]{r}{0.45\textwidth}
    \centering
        \vspace{-0.6cm}
        \includegraphics[width=0.45\textwidth]{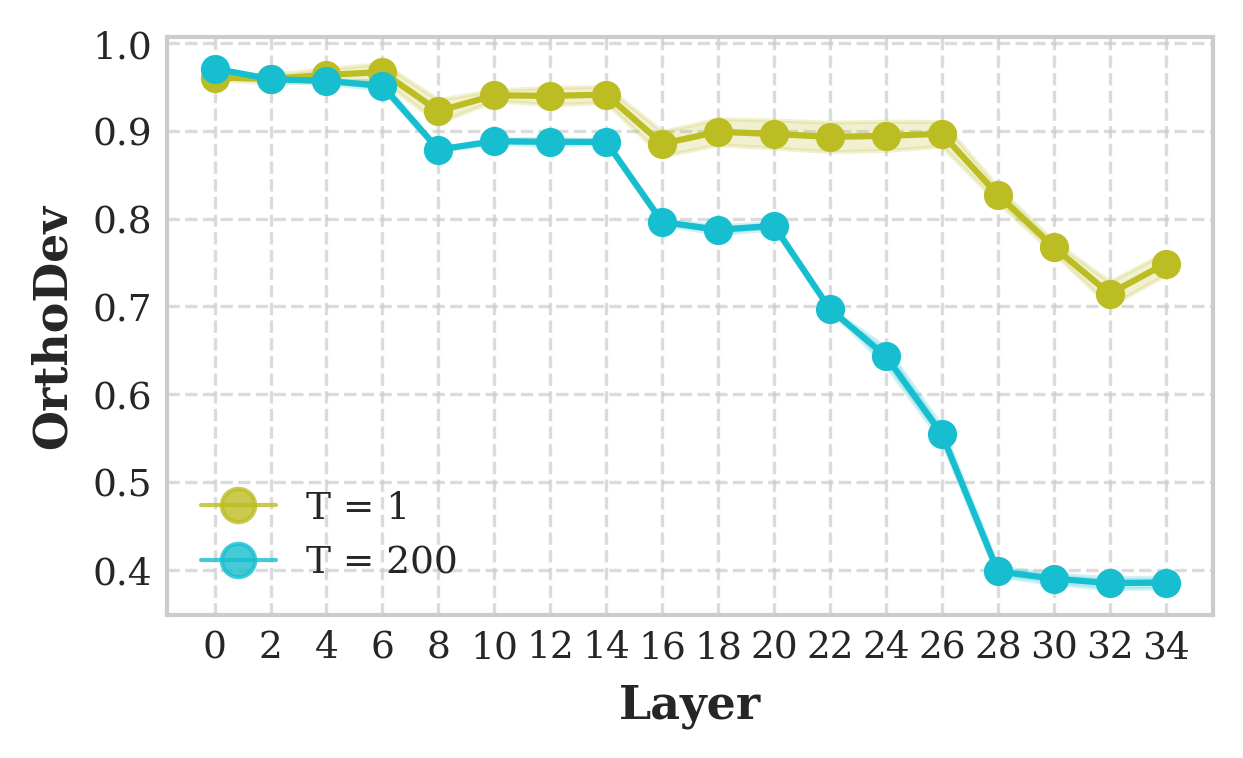}
        \caption{\looseness-1 \small Training with high temperature results in lower $\operatorname{OrthoDev}$ values compared to baseline model. Experiment: ResNet-34 CIFAR-100. Additional experiments in Appendix~\ref{app:NC5}.\label{fig:ood_score}}
\end{wrapfigure}

Studies have explored the relationship between Neural Collapse (NC) and a model's ability to generalize to OOD data~\cite{haas2023linkingneuralcollapsel2,ammar2024neconeuralcollapsebased,harun2025controllingneuralcollapseenhances}. Notably, \cite{haas2023linkingneuralcollapsel2} observed that models with stronger Neural Collapse tendencies exhibit improved OOD detection performance. Further, \cite{ammar2024neconeuralcollapsebased} demonstrated that NC leads to an OOD orthogonality condition measured in practice as the convergence of the following quantity:


\begin{equation}\label{eq:ood_metric}
\operatorname{OrthoDev} \coloneqq \operatorname{Avg}_c\left|\frac{\left\langle\mu_c, \mu_G^{\mathrm{OOD}}\right\rangle}{\left\|\mu_c\right\|_2\left\|\mu_G^{\mathrm{OOD}}\right\|_2}\right| \rightarrow 0,
\end{equation}

where $\mu_c$ denotes the class average on penultimate layer representations, and $\mu_G^{\mathrm{OOD}}$ is the global average representation on OOD data. Building on this orthogonality condition, \cite{ammar2024neconeuralcollapsebased} proposed NECO, a state-of-the-art OOD detection method. However, empirical results show that $\operatorname{OrthoDev}$ rarely converges to zero. Given authors~\cite{ammar2024neconeuralcollapsebased} links $\operatorname{OrthoDev}$ NC, we hypothesize that models with \textit{rank-deficit bias} should achieve lower $\operatorname{OrthoDev}$ values and improve OOD detection using NECO. 

\begin{table}[!h]
\centering
\resizebox{0.8\columnwidth}{!}{%
\begin{tabular}{@{}ccccc|cccc@{}}
\toprule
  ID $\rightarrow$ OOD Dataset        & \multicolumn{4}{c|}{CIFAR-100 $\rightarrow$ CIFAR-10} & \multicolumn{4}{c}{CIFAR-10 $\rightarrow$ CIFAR-100} \\ \midrule
 & \multicolumn{2}{c}{Baseline} & \multicolumn{2}{c|}{Low Logit Norm} & \multicolumn{2}{c}{Baseline} & \multicolumn{2}{c}{Low Logit Norm} \\ \midrule
          & AUROC $\uparrow$  & FPR $\downarrow$  & AUROC $\uparrow$  & FPR $\downarrow$ & AUROC $\uparrow$  & FPR $\downarrow$ & AUROC $\uparrow$ & FPR $\downarrow$ \\ \midrule
ResNet-18 &   $63.66\%$     &   $90.44\%$   &    $\textbf{ 77.04\%}$   &  $\textbf{ 80.13\%}$  &    $77.75\%$   &   $73.29\%$   &    $\textbf{87.53\%}$    &   $\textbf{55.87\%}$   \\
ResNet-34 &    $62.36\%$    &   $91.13\%$   &   $\textbf{76.37\%}$      &   $\textbf{79.27\%}$   &   $73.96\%$     &   $73.60\%$   &    $\textbf{86.18\%}$    &   $\textbf{57.23\%}$   \\ 
VIT-B &    $92.38\%$    &   $40.01\%$   &   $\textbf{ 93.96\%}$    &   $\textbf{30.92\%}$   &   $98.09\%$     &   $9.51\%$   &     $\textbf{98.33\%}$   &  $\textbf{8.87\%}$    \\ 

\bottomrule
\end{tabular}%
}
\caption{\looseness-1 \small{NECO method for OOD detection on baseline (or low logit norm) models. ID/OOD dataset: CIFAR-100/CIFAR-10 (left) and reversed (right). Metrics: AUROC and FPR (at $95\%$ TPR), both in $\%$. ViT-B (pretrained on ImageNet-21k) shows stronger baseline results than ResNet (trained from scratch). Best results are bolded.}}
\label{tab:OOD_detection}
\end{table}

Our experiments confirm this: Figure~\ref{fig:ood_score} reveals that high-temperature training reduces $\operatorname{OrthoDev}$, while Table~\ref{tab:OOD_detection} demonstrates significantly better OOD detection performance. However, Table~\ref{tab:main-table} and Figure~\ref{fig:teaser} show that these same models suffer from poor OOD generalization in the last layers. This trade-off is further explored in Appendix~\ref{app:ood_vs_generalization}, showing that stronger OOD generalization correlates with weaker OOD detection. This observation is in line with recent work~\cite{harun2025controllingneuralcollapseenhances} where authors proposed a regularization term to force NC and increase the model's OOD detection.

\paragraph{Summary} There exists a fundamental trade-off between OOD generalization and OOD detection directly controlled by the \softmax temperature.

\section{What implicitly change logits norm?}\label{sec:architecture_choices}

\begin{wraptable}[10]{r}{80mm}
\vspace{-0.4cm}

\centering
\centering
\resizebox{\linewidth}{!}{%
\begin{tabular}{@{}ccccccc@{}}
\toprule
\multirow{2}{*}{Dataset} & \multicolumn{3}{c}{Baseline} & \multicolumn{3}{c}{Low Logit Norm} \\ 

& $\kappa$ $\downarrow$ & $\rho$ $\downarrow$ & SR & $\kappa$ $\downarrow$ & $\rho$ $\downarrow$ & SR \\ 
\midrule
CIFAR-10      &  47\%     &  64\%   &   9    &   53\%    &  64\%   &   9  \\  
CIFAR-100     &  58\%     &  58\%   &   73   &   58\%    &  53\%   &   45 \\
\bottomrule
\end{tabular}%
}
\caption{\small{Training baseline VGG-19 models on CIFAR-10/CIFAR-100 already leads to collapsed models, and decreasing further logit norm does not lead to substantial differences except for solutions rank.}}
\label{tab:vgg_table}

\end{wraptable}
While our previous sections established softmax temperature's importance in shaping representations, we now demonstrate that architectural choices equally regulate logit norms through three primary mechanisms: initialization strategies, network dimensions, and normalization layers.

\paragraph{Empirical Motivation}
The necessity of this analysis becomes evident when comparing models across architectures. Tables~\ref{tab:vgg_table} and~\ref{tab:main-table} contrasts each other. While baseline ResNets find full rank solutions and training with high temperature leads to collapse, for VGG-19, the baseline models are already collapsed, and high-temperature training shows minimal effects. This occurs because VGG-19 networks inherently produce lower initial logit norms than ResNets, a phenomenon we empirically validate in Figure~\ref{fig:initial_norms_vgss_resnets}.

\paragraph{Initialization Strategies}
Weight initialization critically impacts both activation scales and early training dynamics through its effect on logit norms. The relationship can be described through the following inequalities: $\|\Mm\| = \|\Wm\Am\| \leq \|\Wm\| \|\Am\| \approx \sigma \sqrt{nm}\|\Am\|$, where $\Wm \in \mathbb{R}^{n \times m}$ with $w_{ij} \sim \mathcal{N}(0, \sigma)$. In practice, different initialization schemes lead to varying behaviors. For instance, contrastive reinforcement learning often initializes final layers from $\mathcal{U}(-10^{-12},10^{-12})$ \cite{zheng2024stabilizing}, which keeps initial representations tightly clustered. While this approach stabilizes optimization, it may slow early learning progress compared to more conventional initialization methods.

\paragraph{Network Dimensions}
The width of network layers significantly influences representation learning through its effect on logit norms. Wider layers naturally produce larger representation norms due to increased dimensionality. This effect becomes particularly pronounced in self-supervised learning frameworks, where projection heads are typically designed much wider than their backbone networks \cite{chen2020simple}. Empirical studies of SimCLR demonstrate that removing the projector leads to collapsed representation spectra, resulting in reduced feature expressivity and degraded downstream performance \cite{jing2022understanding}, highlighting the crucial role of width in maintaining effective representations.

\paragraph{Normalization Effects}
Normalization layers provide another architectural mechanism for controlling logit norm growth through their mean subtraction and variance scaling operations \cite{ba2016layer}. The placement of these layers relative to the logits proves particularly crucial: applying normalization after the logits effectively prevents norm inflation and consequently mitigates \textit{rank-deficit bias}. Recent research has shown that batch normalization not only stabilizes training but also preserves layer-wise representation diversity, actively preventing rank collapse \cite{daneshmand2020batch}. These findings underscore the importance of careful normalization layer placement in controlling representation quality.

Collectively, these architectural elements establish a model's implicit "temperature" bias, directing learning trajectories independent of explicit softmax temperature settings. We provide comprehensive experimental validation of these effects in Appendix~\ref{app:architectural_hyperparameters}.

\section{Related works}
Our work connects several research areas: Neural Collapse, representation learning, \softmax temperature effects, and the interplay between out-of-distribution (OOD) generalization, detection, and model compression. We contextualize our contributions within these fields.

Neural Collapse (NC)~\cite{papyan2020neuralcollapse} describes a rigid geometric structure where penultimate layer representations converge to an equiangular tight frame (ETF) simplex when trained to the terminal phase (TPT). Subsequent work~\cite{rangamani23intermediate,parker2023neuralcollapseintermediatehidden} identified similar structures in intermediate layers (Intermediate Neural Collapse), though without explaining the underlying mechanisms. 

Our findings reveal key differences from NC: (1) rank collapse occurs early in training, well before TPT; (2) it emerges without specific regularization or hyperparameters; and (3) solutions exhibit ranks significantly lower than NC predictions. While related to NC, our results demonstrate that networks with collapsed rank only partially satisfy NC conditions, suggesting NC cannot fully explain our observations. Crucially, we provide a mechanistic explanation for rank collapse and show how to control it via hyperparameters or \softmax temperature (see Appendix~\ref{app:neural_collapse}).

Unconstrained Feature Models (UFMs) provide a simplified theoretical framework for studying Neural Collapse by treating both model parameters and input features as optimizable variables. Prior work has used UFMs to analyze Neural Collapse under class imbalance~\cite{hong2024neural} and extended them to deep architectures (DUMFs), where~\cite{sukenik2024neural} identified low-rank solutions in DUMFs before the last ReLU under high weight decay. In contrast, we study standard architectures, measuring rank directly from pre-softmax logits—the classification-relevant space—without relying on weight decay. Our results demonstrate that \textit{rank-deficit bias} is inherent in practical networks, not just theoretical models, and we offer actionable insights for controlling it via temperature and hyperparameters.

Temperature scaling has been used at inference time for sharpening output distribution for OOD tasks~\cite{velickovic2024softmax} or improving model calibration~\cite{guo2017calibration}. On the other hand, temperature tuning turned out to be crucial in self-supervised learning~\cite{chen2020simple} or private LLM inference~\cite{jha2024aerosoftmaxonlyllmsefficient}. While temperature tuning has been used in multiple different areas, to the best of our knowledge, prior work has not examined its impact on representation learning and OOD performance—a key focus of our study.

Recent work explores the compression-OOD performance trade-off~\cite{harun2024variables,masarczyk2024tunnel}, with~\cite{evci2022head2toe,masarczyk2024tunnel} demonstrating improved transfer learning via intermediate representations. Others~\cite{ammar2024neconeuralcollapsebased,haas2023linkingneuralcollapsel2,harun2025controllingneuralcollapseenhances} link stronger NC to better OOD detection, though at the cost of generalization~\cite{haas2023linkingneuralcollapsel2}. We extend these findings by showing how low logit norms, induced by temperature or architecture, affect OOD generalization and detection.

Our work unifies these perspectives, offering new insights into how \softmax shapes neural representations and suggesting directions for improving deep learning models. A more detailed related works section can be found in the Appendix~\ref{app:related_works}.

\section{Conclusions}

In this work, we have presented a systematic investigation of how the \softmax function shapes the learned representations in deep neural networks. Our primary contribution is the identification and analysis of \emph{rank-deficit bias}—a previously unrecognized phenomenon where networks trained with \softmax consistently converge to solutions whose representations have significantly lower rank than the number of classes. This finding contrasts with the full-rank solutions predicted by Neural Collapse theory, revealing an important aspect of \softmax's influence on learning dynamics.

Through extensive theoretical analysis and empirical validation across multiple architectures and datasets, we have demonstrated that this behavior is fundamentally governed by the norm of the logits at initialization. Importantly, we have shown that the \softmax temperature parameter serves as a powerful control mechanism, enabling practitioners to explicitly manage the trade-off between representation compactness and model performance. Our results provide both theoretical insights and practical tools for optimizing neural network training.

While our findings provide novel insights into \softmax's role in representation learning, this work has several limitations that point to valuable future research directions. Our analysis has focused primarily on supervised image classification, leaving open questions about how \textit{rank-deficit bias} manifests in other architectures and learning paradigms.

For instance, investigating \softmax in intermediate Transformer layers could yield insights into training stability and efficiency, particularly in NLP, where these layers share structural similarities with classification tasks yet remain understudied in this context~\cite{wu2024linguisticcollapseneuralcollapse}. Similarly, self-supervised learning methods, many of which rely on \softmax-based losses (e.g., contrastive learning), could benefit from our framework.

Finally, while our work analyzes fixed-temperature training, an important next step would be to explore dynamic temperature scheduling, which could combine the benefits of both standard and high-temperature regimes during optimization.

\begin{ack}
MO was partially supported by the National Science Centre, Poland, under a grant 2023/51/D/ST6/01609.
We gratefully acknowledge Polish high-performance computing infrastructure PLGrid (HPC Center: ACK Cyfronet AGH) for providing computer facilities and support within computational grant no. PLG/2025/017963

\end{ack}

\bibliography{bibliography}
\bibliographystyle{unsrt}

\newpage

\appendix

\section{Experimental details}\label{app:architectures}
\subsection{Architectures}

In this section, we detail the model architectures examined in the experiments and list all hyperparameters used in the experiments.

\paragraph{VGG~\cite{Simonyan15}} In the main text, we use VGG-19. The architecture consists of five stages, each consisting of a combination of convolutional layers with ReLU activation and max pooling layers. The VGG-19 has 19 layers, including 16 convolutional layers and three fully connected layers. The first two fully connected layers are followed by ReLU activation. The base number of channels in consecutive stages for VGG architectures equals 64, 128, 256, 512, and 512.

\paragraph{ResNet~\cite{he2016deep}} 
In experiments, we utilize three variants of the ResNet family of architectures, i.e., ResNet-18, ResNet-34, and ResNet-50. ResNet-$N$ is a five-staged network characterized by depth, with a total of $N$ layers. The initial stage consists of a single convolutional layer -- with kernel size $7\times7$ and 64 channels and ReLU activation, followed by max pooling $2\times2$, which reduces the spatial dimensions. The subsequent stages are composed of residual blocks. Each residual block typically contains two or three convolutional layers and introduces a shortcut connection that skips one or more layers. Each convolutional layer in the residual block is followed by batch normalization and ReLU activation. The remaining four stages in ResNet-18 and ResNet-34 architectures consist of 3x3 convolutions with the following number of channels: 64, 128, 256, and 512. ResNet-50 uses bottleneck blocks with 1x1, 3x3, and 1x1 convolutions, with channel dimensions of 256, 512, 1024, and 2048 in the four main stages. When training ResNets on CIFAR-10/CIFAR-100, we modify the kernel size of the first layer ($7 \times 7 \rightarrow 3 \times 3$) and do not use the max pooling layer. 

\paragraph{VIT-B~\cite{dosovitskiy2020image}}
The Vision Transformer (ViT-B) architecture processes images by dividing them into fixed-size patches (16x16), which are then linearly embedded. These patch embeddings are combined with positional embeddings and fed into a standard Transformer encoder. The ViT-B variant consists of 12 transformer layers, each with a hidden size of 768 and 12 attention heads. The Multi-Layer Perceptron (MLP) in each transformer block has a dimension of 3072. The model uses layer normalization before each block and residual connections around each sub-layer. Unlike convolutional networks, ViT-B relies entirely on self-attention mechanisms to model relationships between image patches, allowing it to capture both local and global dependencies in the image. When fine-tuned on CIFAR-10/CIFAR-100, we resize the images to $224 \times 224$.

\paragraph{MLP~\cite{bebis1994feed}} 
An MLP (Multi-Layer Perceptron) network is a feedforward neural network architecture type. It consists of multiple layers -- in our experiments, 8 hidden layers with ReLU activations (except the last layer, which has \softmax activation). In our experiments, the underlying architecture has 2048 neurons per layer.

\subsection{Datasets}

In this article, we present the results of experiments conducted on the following datasets:

\paragraph{CIFAR-10~\cite{Krizhevsky09learningmultiple}} 
CIFAR-10 is a widely used benchmark dataset in the field of computer vision. It consists of 60,000 color images in 10 different classes, with each class containing 6,000 images. The dataset is divided into 50,000 training images and 10,000 test images. The images in CIFAR-10 have a resolution of $32\times 32$ pixels. The dataset is released under a custom license that grants all rights to users with the only obligation being proper citation of the original work.

\paragraph{CIFAR-100~\cite{Krizhevsky09learningmultiple}} 
CIFAR-100 is a dataset commonly used for image classification tasks in computer vision. It contains 60,000 color images, with 100 different classes, each containing 600 images. The dataset is split into 50,000 training images and 10,000 test images. The images in CIFAR-100 have a resolution of $32\times 32$ pixels. Unlike CIFAR-10, CIFAR-100 offers a higher level of granularity, with more fine-grained categories such as flowers, insects, household items, and various types of animals and vehicles. The license terms for CIFAR-100 are identical to those of CIFAR-10.

\paragraph{ImageNet-100~\cite{DenDon09Imagenet}}
ImageNet-100 is a subset of the ImageNet-1k dataset, consisting of 100 randomly sampled classes while maintaining the original dataset's distribution. It contains approximately 130,000 training images and 5,000 validation images, with each class having roughly 1,300 training and 50 validation examples on average. The images vary in resolution, but we preprocess them to $224\times 224$ pixels. ImageNet-100 provides a more manageable scale for experimentation while preserving the diversity and complexity of the full ImageNet dataset. The dataset is released under the same terms as ImageNet-1k, which allows for non-commercial research use. 

\paragraph{ImageNet-1k~\cite{DenDon09Imagenet}}
ImageNet-1k (also known as ILSVRC 2012) is a large-scale dataset containing 1.2 million training images and 50,000 validation images across 1,000 object categories. Each category contains approximately 1,300 training images and 50 validation images. The images have varying resolutions but we preprocess them to $224\times 224$ pixels. ImageNet-1k has been instrumental in advancing computer vision research and serves as a standard benchmark for image classification tasks. The dataset is available for non-commercial research purposes under the terms specified by the ImageNet organization, which requires proper attribution and prohibits commercial use without additional permissions.

\subsection{Data Augmentations}

For CIFAR-10 and CIFAR-100, we applied standard preprocessing with normalization, along with the following data augmentations:
\begin{itemize}
    \item Random cropping:
    \begin{itemize}
        \item Size: $32 \times 32$
        \item Padding: $4$
    \end{itemize}
    \item RandomHorizontalFlip: Applied with $p = 0.5$
\end{itemize}

For models trained on ImageNet-100 and ImageNet-1k, we used a more extensive augmentation pipeline, implemented via the \texttt{Timm} library~\cite{rw2019timm}. The full specifications can be found in the library's repository~\url{https://github.com/rwightman/timm}. The augmentations included:
\begin{itemize}
    \item RandomResizedCropAndInterpolation:
    \begin{itemize}
        \item Size: $224 \times 224$
        \item Scale range: $(0.08, 1.0)$
        \item Aspect ratio range: $(0.75, 1.3333)$
        \item Interpolation: Bilinear/Bicubic
    \end{itemize}
    \item RandomHorizontalFlip: Applied with $p = 0.5$
    \item ColorJitter:
    \begin{itemize}
        \item Brightness adjustment: $(0.6, 1.4)$
        \item Contrast adjustment: $(0.6, 1.4)$
        \item Saturation adjustment: $(0.6, 1.4)$
        \item Hue adjustment: None
    \end{itemize}
\end{itemize}

When fine-tuning a pretrained ViT-B for the experiments in Section~\ref{sec:generalization_vs_detection}, we used the augmentations proposed by the authors of the \texttt{NECO} method~\cite{ammar2024neconeuralcollapsebased}. The data transformations were implemented as follows:

\begin{itemize}
    \item RandomResizedCropAndInterpolation:
    \begin{itemize}
        \item Output size: $224 \times 224$
        \item Scale range: $(0.05, 1.0)$
    \end{itemize}
    \item Normalization with mean $[0.5, 0.5, 0.5]$ and std $[0.5, 0.5, 0.5]$
\end{itemize}

\subsection{Hyperparameters}

The hyperparameters used for neural network training are listed in Table~\ref{tab:cifar10_100_hparams}. Each column corresponds to a different combination of architecture and dataset. The values provided are optimized for achieving the best performance on the CIFAR-10 dataset~\cite{liu2019rethinkingvaluenetworkpruning}. For ImageNet, we used the default hyperparameters recommended by the \texttt{timm} library~\cite{rw2019timm}.

Notably, for all experiments except those on ImageNet, the high-temperature models were trained using the same hyperparameters as their baseline counterparts. This ensures that any performance differences can be attributed solely to the temperature adjustment. However, we anticipate that further improvements could be achieved by fine-tuning hyperparameters specifically for high-temperature training—a promising direction for future work.

In the case of ImageNet experiments with high-temperature training, we adjusted only the number of epochs and learning rate to mitigate the slowdown caused by increased temperature, while keeping other hyperparameters unchanged.

\begin{table}[!h]
    \centering
    \begin{tabular}{c|c|c|c|c}
        \hline
        & \multicolumn{3}{c}{CIFAR-10/CIFAR-100} & ImageNet \\
        \hline
        Parameter & VGG & ResNet & MLP & ResNet \\
        \hline
        Learning rate (LR) & $0.1$ & $0.1$  &  $0.05$ & $1.6$ \\ 
        SGD momentum & $ 0.9$& $0.9$ & $0.0$ & $0.9$ \\
        Weight decay & $10^{-4}$ & $10^{-4}$ & $0$ & $10^{-4}$ \\
        Number of epochs & 160& 164  &  100 & 120 \\
        Mini-batch size & 128& 128 & 128 & 1024\\
        LR-decay-milestones & 80, 120& 82, 123 &  - & 30, 60, 90 \\
        LR-decay-gamma & 0.1& 0.1& 0.0 & 0.1 \\
        \hline
    \end{tabular}
    \caption{Hyperparameters used during training the models on different datasets.}
    \label{tab:cifar10_100_hparams}
\end{table}

\subsection{Experimental methodology}

Our analysis primarily relies on quantities derived from linear probing accuracies and numerical rank estimates. To obtain these, we collect representations from each layer following the nonlinear activation function, with the exception of the final layer, where we compute statistics both before and after applying the \softmax operation.

\begin{wraptable}[10]{r}{40mm}
\centering
\centering
\resizebox{\linewidth}{!}{%
\begin{tabular}{c|c}
\hline
Parameter & Value \\
\hline
Learning rate & $0.001$ \\ 
Weight decay & $10^{-3}$ \\
Number of epochs & 50 \\
Mini-batch size & 4096 \\
\hline
\end{tabular}
}
\caption{Hyperparameters for training linear probes on representations.}
\label{tab:linear_probes_hparams}
\end{wraptable}

We note significant variability in numerical rank computations depending on experimental design. Consistently, we observe that the rank of pre-activations (prior to ReLU) is substantially lower than that of post-activation representations. Furthermore, while some researchers compute the numerical rank of the representation covariance matrix rather than the representations themselves, this approach typically yields lower rank estimates while losing the interpretability of the rank as a proxy for the number of linearly independent features or samples at a given layer. This interpretation is particularly crucial for our analysis of \textit{rank-deficit bias}, which is why we consistently report the rank of the activation matrix itself rather than its covariance.

For both linear probing and numerical rank computations, we used a feature subset with size equal to the minimum of 10,000 or the full feature dimension (of flattened representations in case of convolutions). When computing solution ranks, we evaluated this quantity on the complete dataset, with the sole exception of ImageNet-1k, where we used a subset of 100,000 samples.

The hyperparameters for our linear probing experiments are detailed in Table~\ref{tab:linear_probes_hparams}. All probes were trained using the Adam optimizer.

All reported statistics represent averages across 3 runs with identical hyperparameters but different random seeds. Where applicable, we include the standard deviation across these runs.

\subsection{Computational resources}

Given the extensive scope of our experimental evaluation, we refrain from specifying exact computational requirements for individual experiments. Instead, we provide an overview of the infrastructure used across our study. 

Our investigation proceeded in two phases:
\begin{itemize}
    \item \textbf{Preliminary exploration} was conducted using a single NVIDIA GeForce RTX 2080 Ti GPU (11GB), accumulating approximately 500 GPU-hours of compute time.
    \item \textbf{Full-scale experiments} utilized two multi-GPU configurations with the following GPUs:
    \begin{itemize}
        \item NVIDIA RTX A5000 (24GB)
        \item NVIDIA GH200 (96GB)
    \end{itemize}

\end{itemize}
The complete experimental setup required approximately 50,000 GPU-hours of compute time.

\subsection{Logits norm comparison of randomly initialized models}\label{app:initial_norms_vgss_resnets}

\begin{wrapfigure}[23]{r}{0.6\textwidth}
    \centering
    \includegraphics[width=\linewidth]{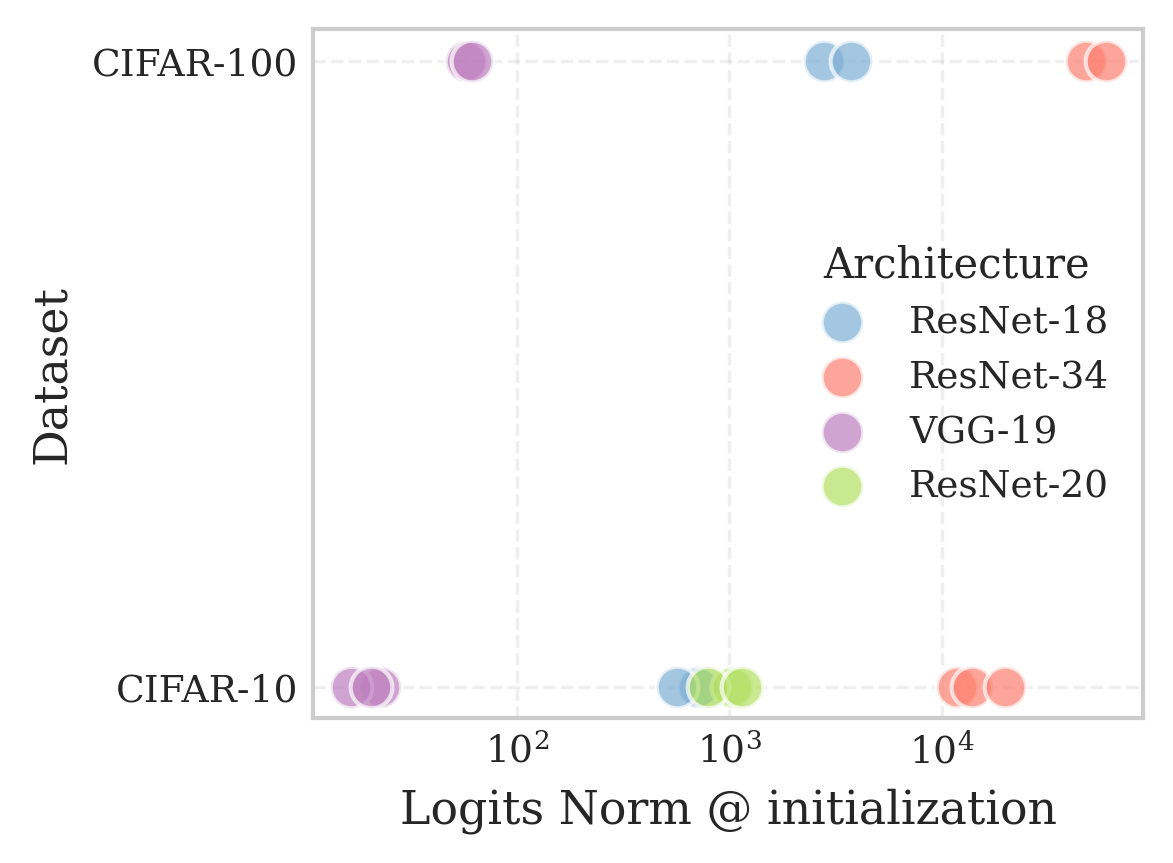}
    \caption{VGG-19 has orders of magnitude lower logit norm compared to models from the ResNet family, leading to a \textit{rank-deficit bias} for baseline models trained with $T=1$. Each point represents a single randomly initialized network.}
    \label{fig:initial_norms_vgss_resnets}
\end{wrapfigure}

Figure~\ref{fig:initial_norms_vgss_resnets} confirms our hypothesis that VGG-19 exhibits lower logit norms than ResNet variants on both CIFAR-10 and CIFAR-100. To investigate whether increasing the logit norm via low-temperature training affects representation learning, we calibrate the temperature for VGG-19 to match ResNet-18's logit norm on CIFAR-10. As shown in Figure~\ref{fig:low_temperature_vgg}, this adjustment yields the expected improvement in reducing the collapse of representations (or at least pushing it to deeper layers). However, even with matched logit norms, VGG-19 and ResNet-18 learn distinct solutions, underscoring the architectural influence on representation learning. We would like to stress the fact that while the logit norms serve as a useful proxy for comparing models of the same architecture and their post-training behavior, they do not generalize across different backbones.

\begin{figure}[!h]
    \centering
    {
    \includegraphics[width=0.48\textwidth]{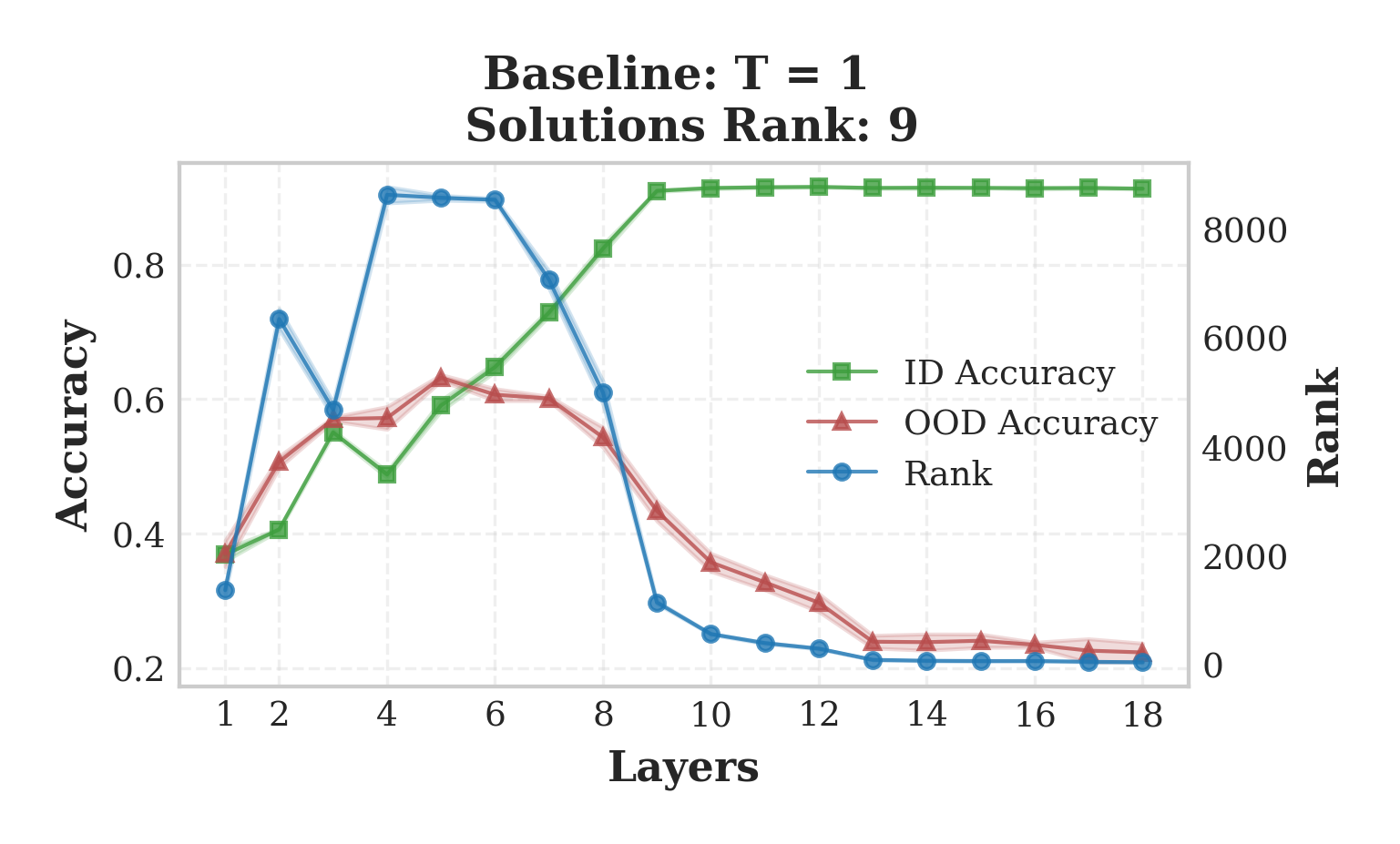}
    }
    {
    \includegraphics[width=0.48\textwidth]{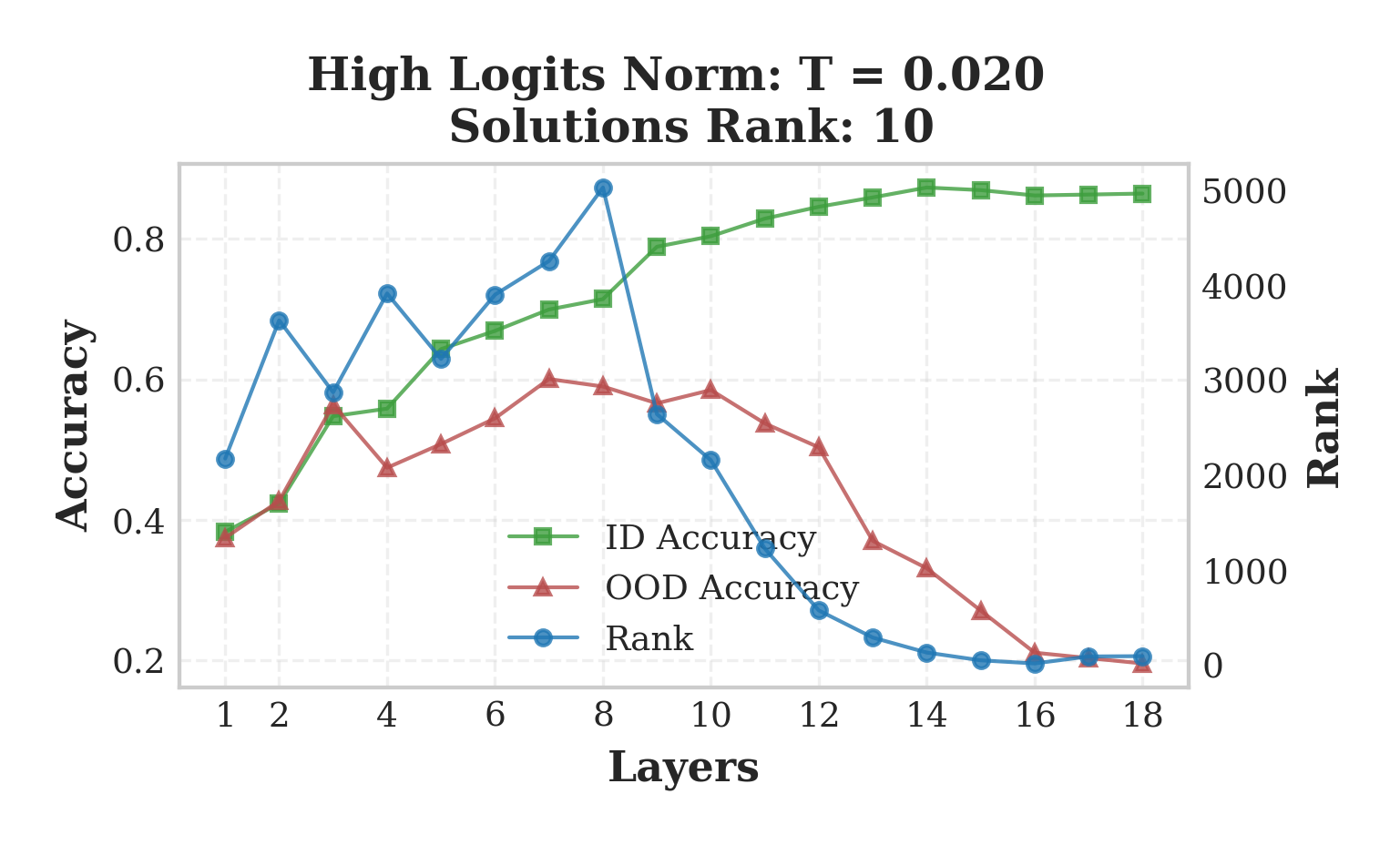}
    }
    \caption{VGG-19 trained with low temperature (right) reduces the effect of collapse, making deeper layers more active in building representations compared to a baseline model (left), which collapses due to low logits norm. } 
    \label{fig:low_temperature_vgg}
\end{figure}

\section{Finegrained alignment}\label{app:pabs}

Figure~\ref{fig:finegrained_alignment} reveals a striking difference in layer alignment dynamics between standard and high-temperature training. The visualization tracks the evolution of inter-layer alignment (y-axis: training progress, x-axis: layer depth) through cosine similarity matrices of the top-15 singular vectors between consecutive layers' representations and weights.

The baseline model shows negligible alignment throughout training (Figure~\ref{fig:alignment}, top), while the high-temperature model develops significant alignment early in training. This alignment emerges first among the dominant singular vectors and gradually propagates to others, though even after 50 epochs it remains concentrated in the top singular vectors - a pattern that appears to constrain learning as other components stay largely orthogonal.

\begin{figure}[!h]
    \centering
    {
    \includegraphics[width=0.48\textwidth]{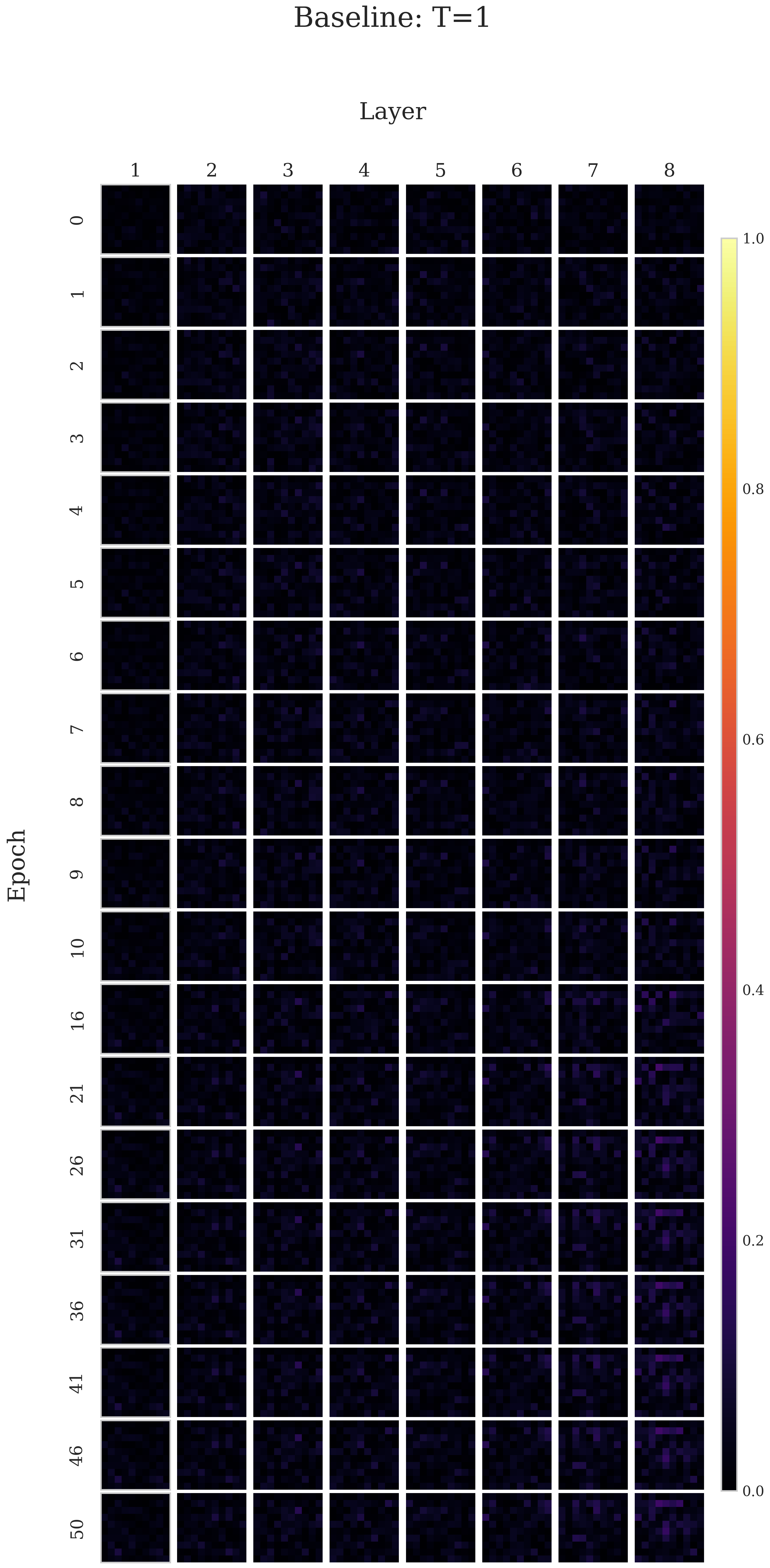}
    }
    {
    \includegraphics[width=0.48\textwidth]{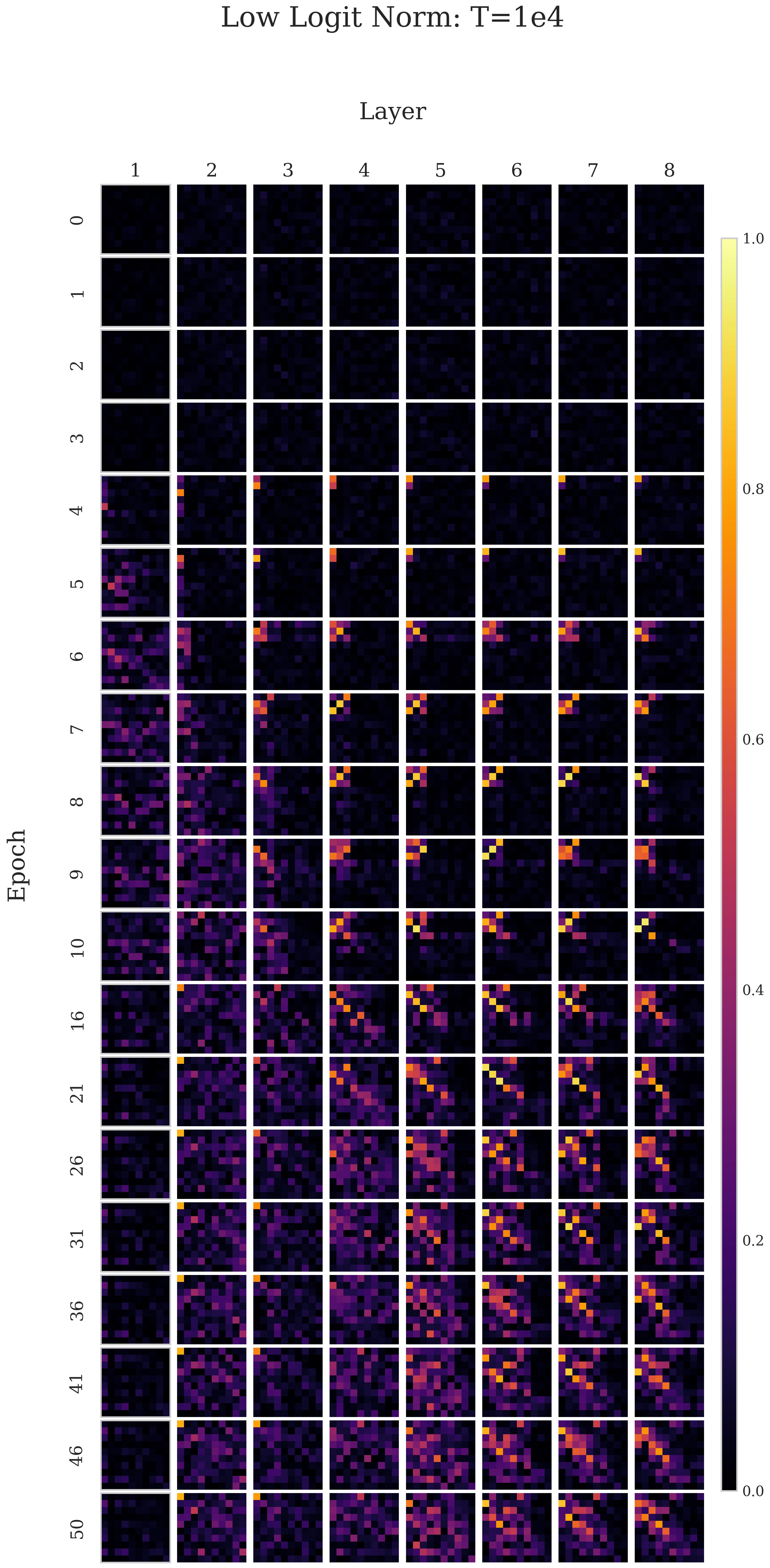}
    }
    \caption{\looseness-1 \small \textbf{The alignment (cosine similarity) between singular vectors of weights and representations} forms during initial training epochs (rows) at deeper layers of the model (columns) trained with high temperature (right), in contrast to the baseline model (left). The plot presents the training process of an MLP network trained on CIFAR-10.} 
    \label{fig:finegrained_alignment}
\end{figure}

For ResNet-18 on CIFAR-100 (Figure~\ref{fig:finegrained_alignment_resnet_high} and Figure~\ref{fig:finegrained_alignment_resnet_baseline}) while the specific alignment pattern differs, the fundamental trend persists: high-temperature training induces substantially stronger alignment, particularly in the final layer and among top singular vectors across all layers, compared to the weaker alignment in baseline models. This consistent observation across architectures strongly suggests that temperature scaling systematically influences neural network learning dynamics at a fundamental level.

\newpage

\section{Additional experiments}\label{app:remaining_experiments}

\subsection{ResNet-18}

\begin{figure}[!h]
    \centering
    {
    \includegraphics[width=0.49\textwidth]{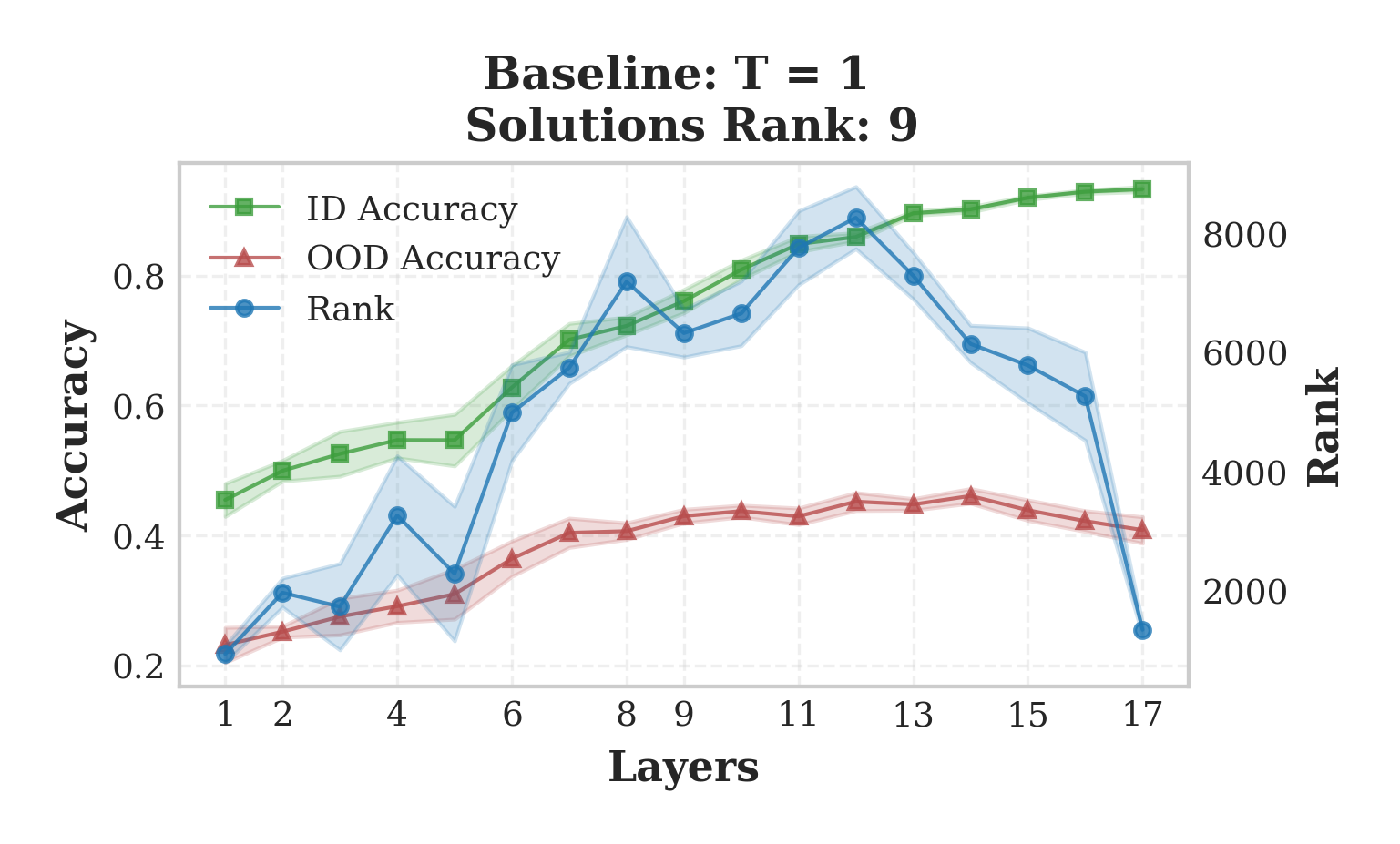}
    }
    {
    \includegraphics[width=0.49\textwidth]{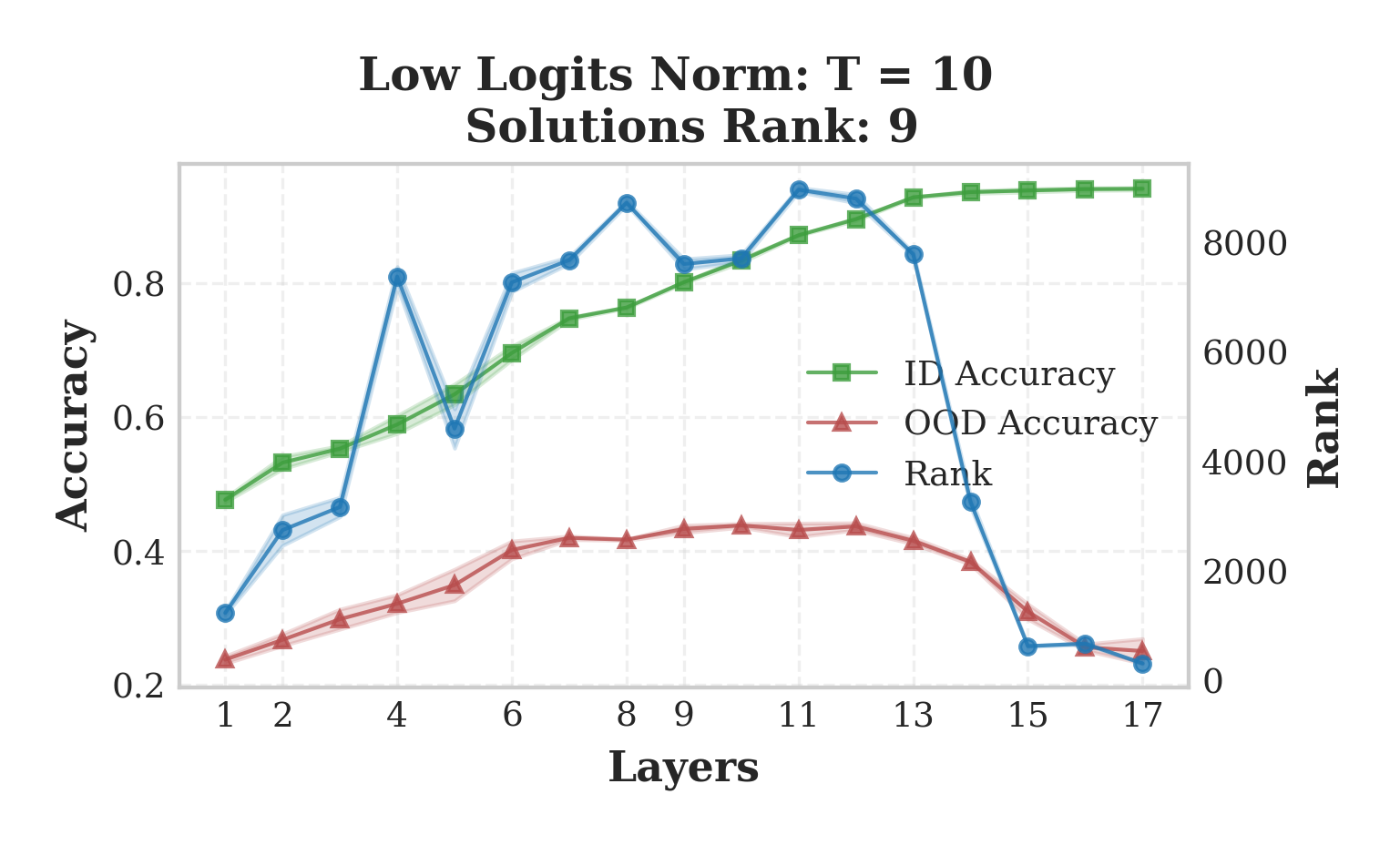}
    }
        {
    \includegraphics[width=0.49\textwidth]{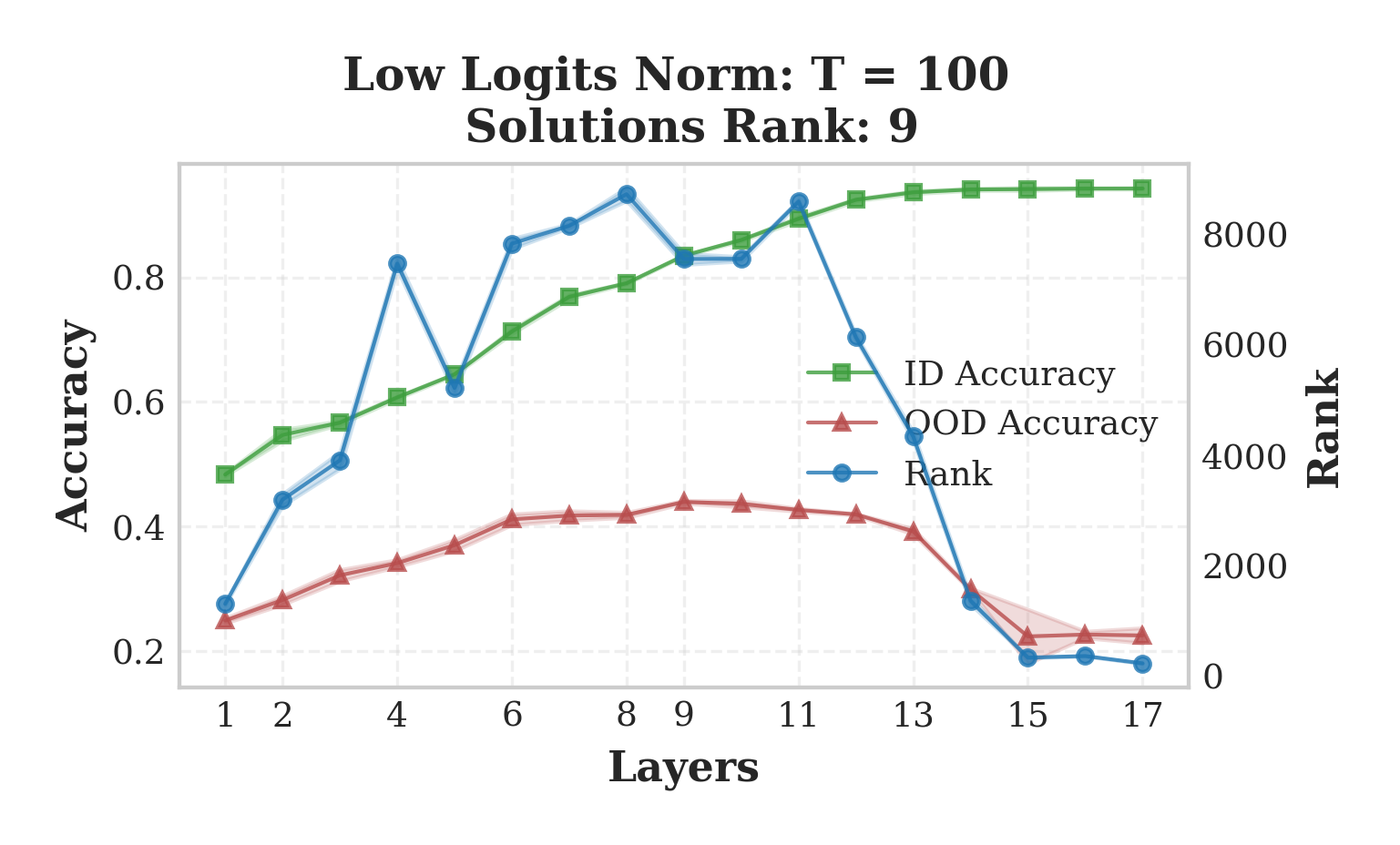}
    }
    {
    \includegraphics[width=0.49\textwidth]{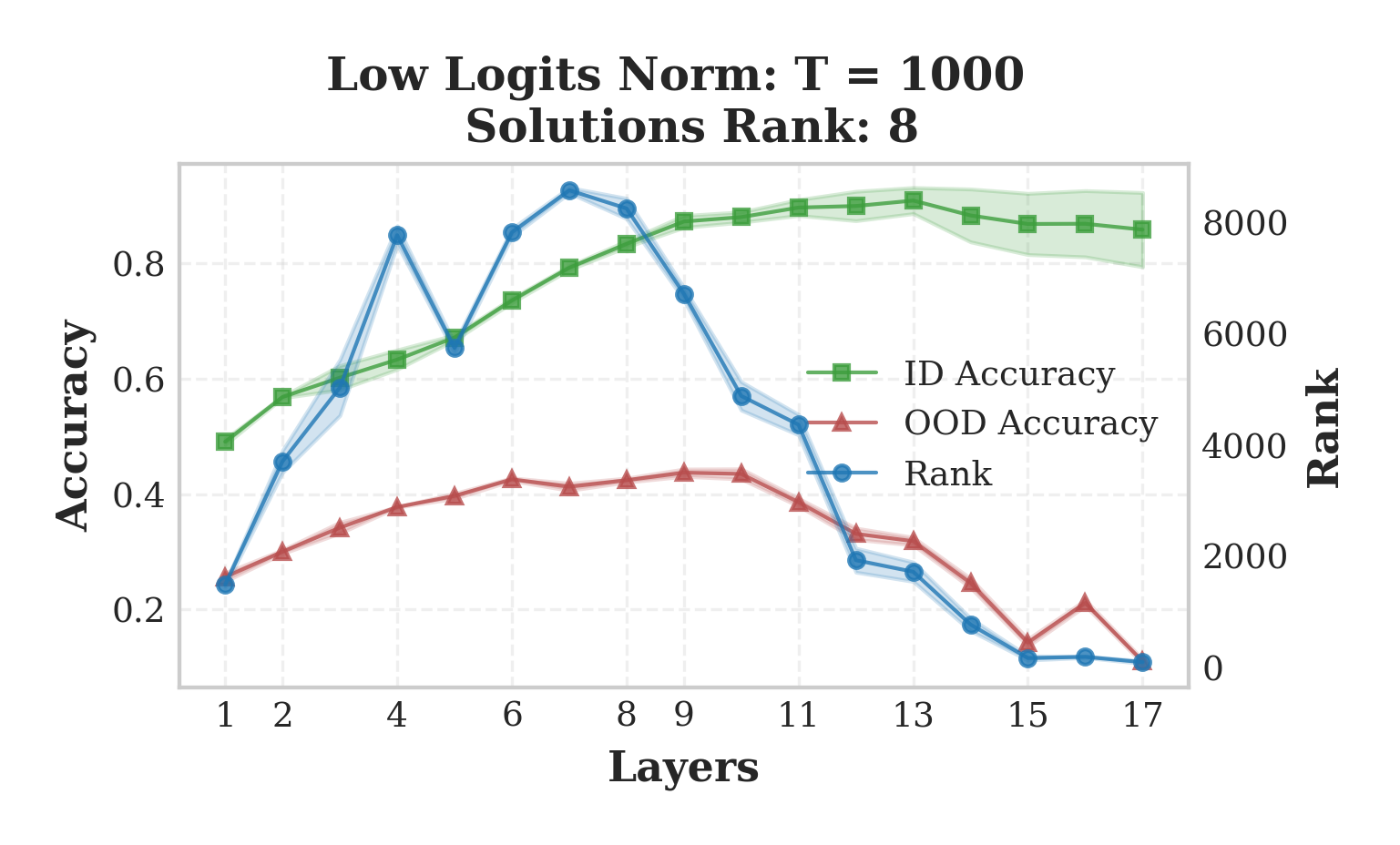}
    }
    \caption{Plot presents the impact of training with high temperature on learned representations and the model's ability to generalize OOD. The higher the temperature, the lower the solutions' rank found by the model. Experiment: ResNet-18 trained on CIFAR-10.} 
\end{figure}

\begin{figure}[!h]
    \centering
    {
    \includegraphics[width=0.49\textwidth]{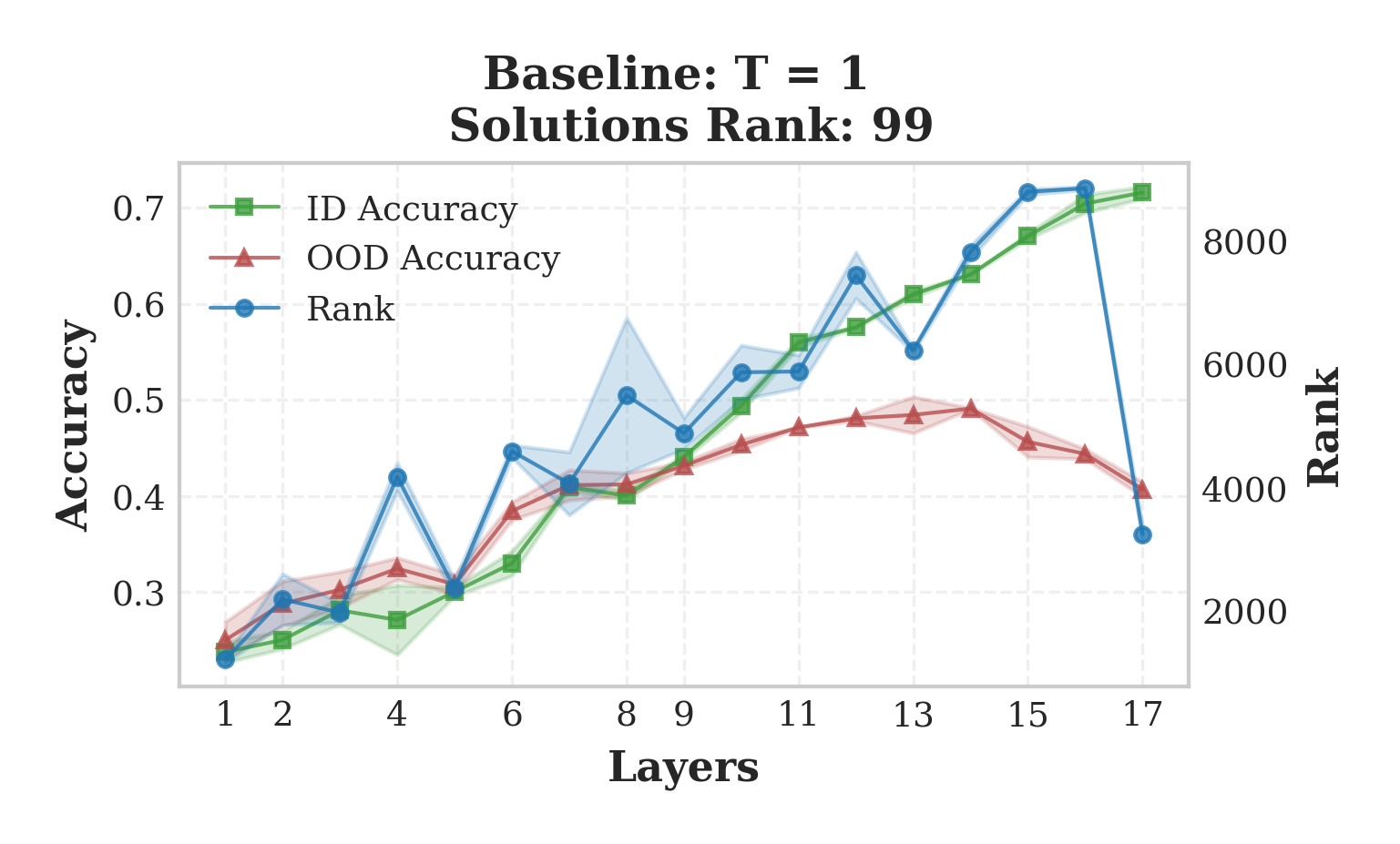}
    }
    {
    \includegraphics[width=0.49\textwidth]{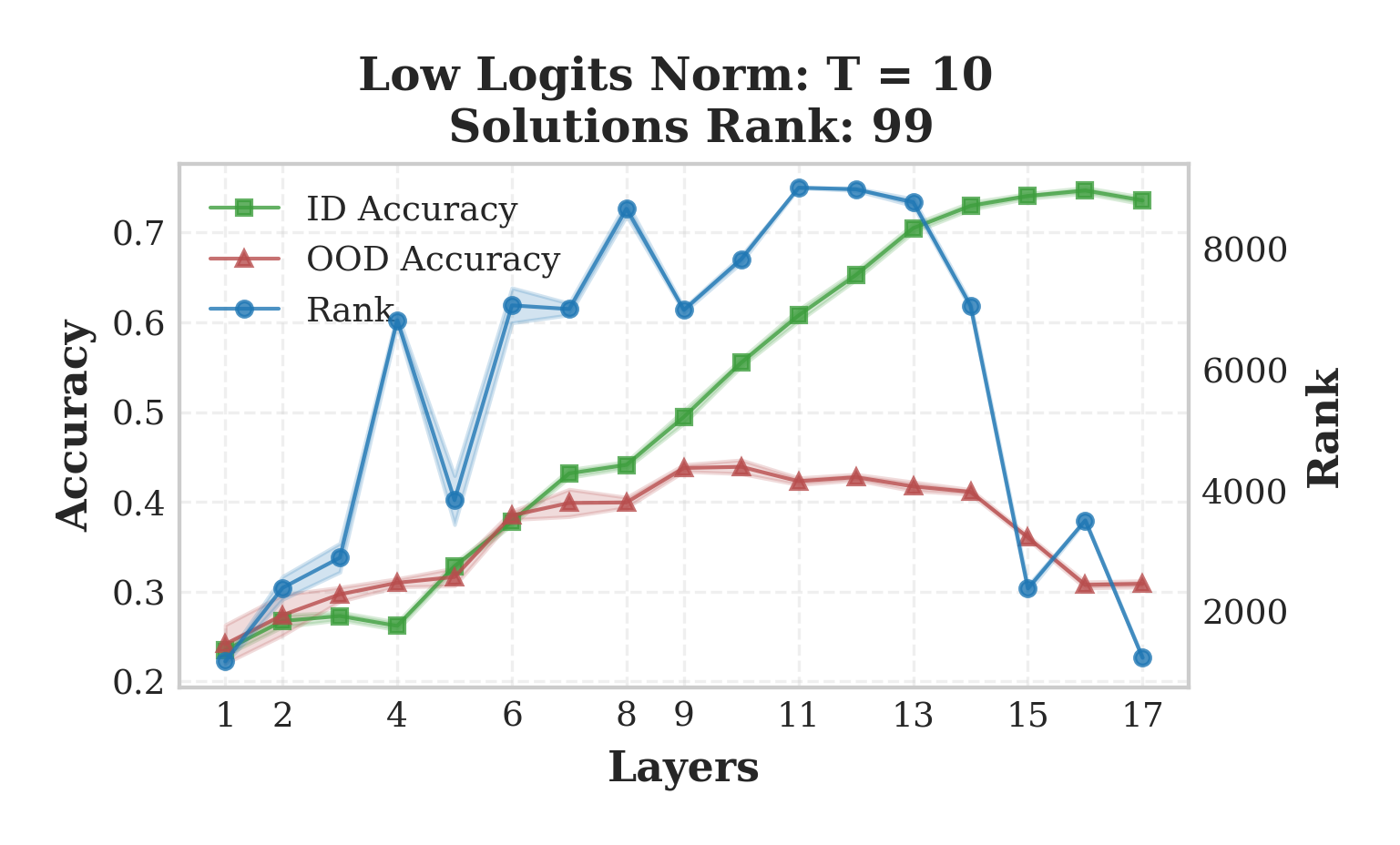}
    }
        {
    \includegraphics[width=0.49\textwidth]{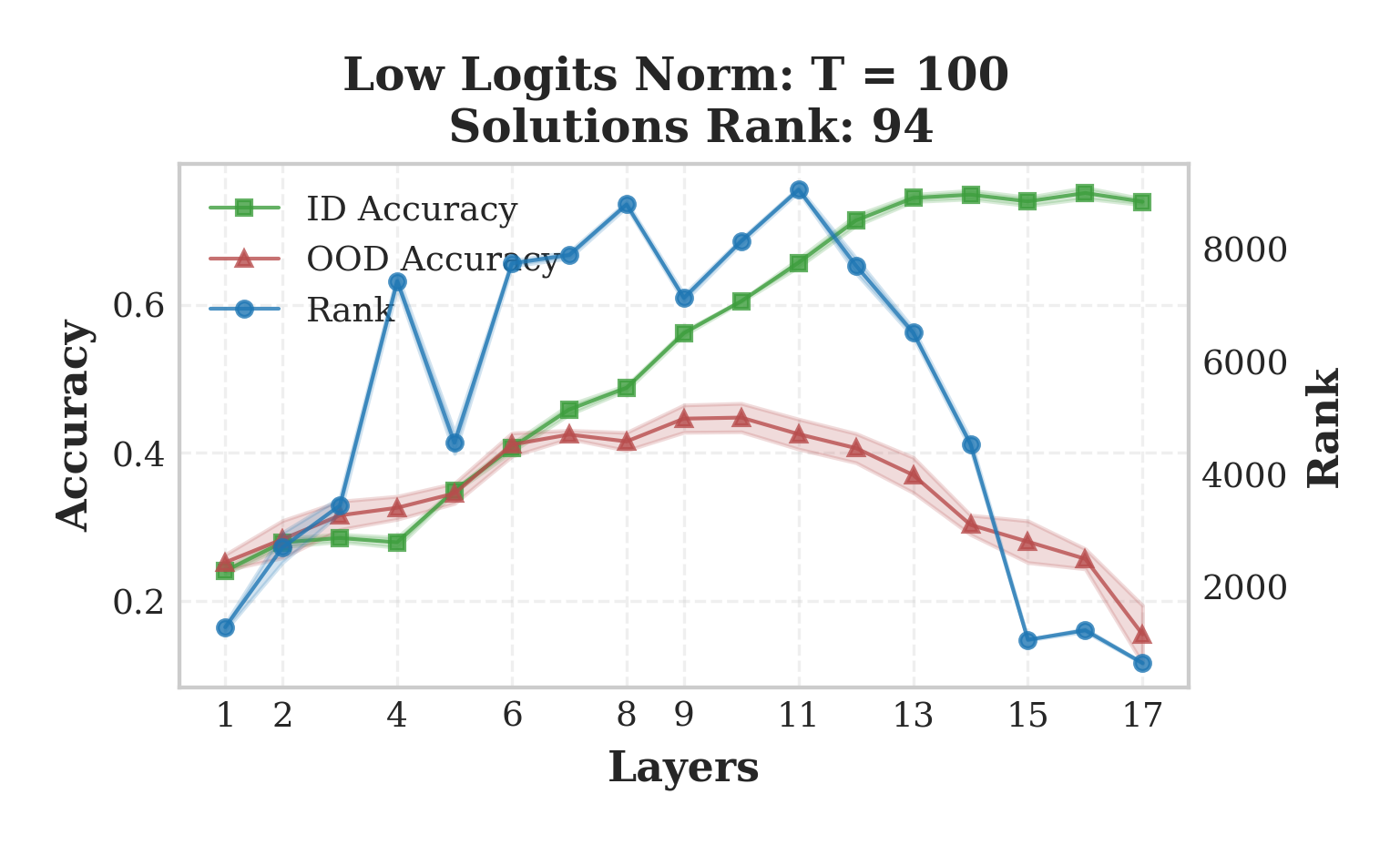}
    }
    {
    \includegraphics[width=0.49\textwidth]{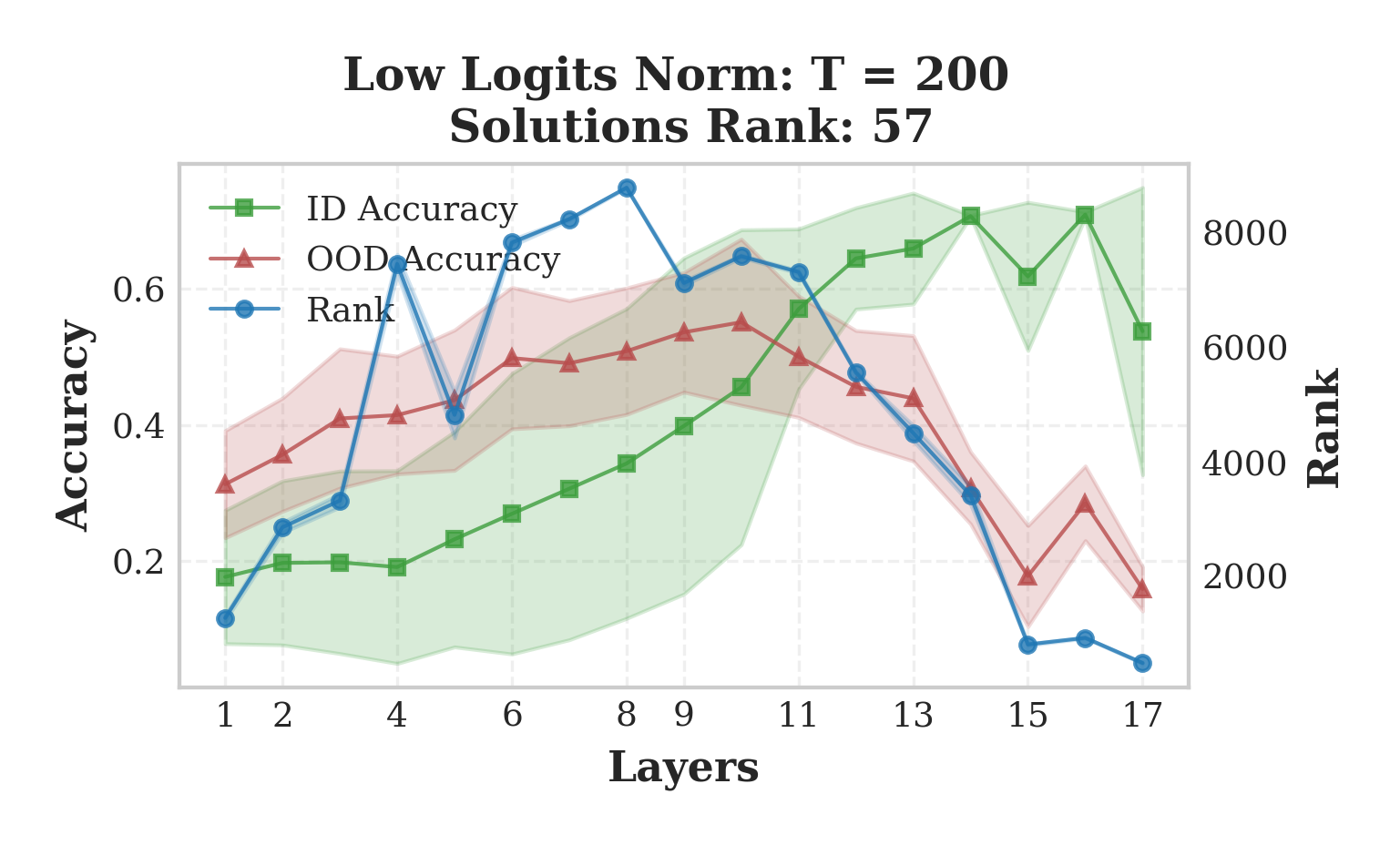}
    }
    \caption{Plot presents the impact of training with high temperature on learned representations and the model's ability to generalize OOD. The higher the temperature, the lower the solutions' rank found by the model. Experiment: ResNet-18 trained on CIFAR-100.} 
\end{figure}

\subsection{ResNet-20}

\begin{figure}[!h]
    \centering
    {
    \includegraphics[width=0.49\textwidth]{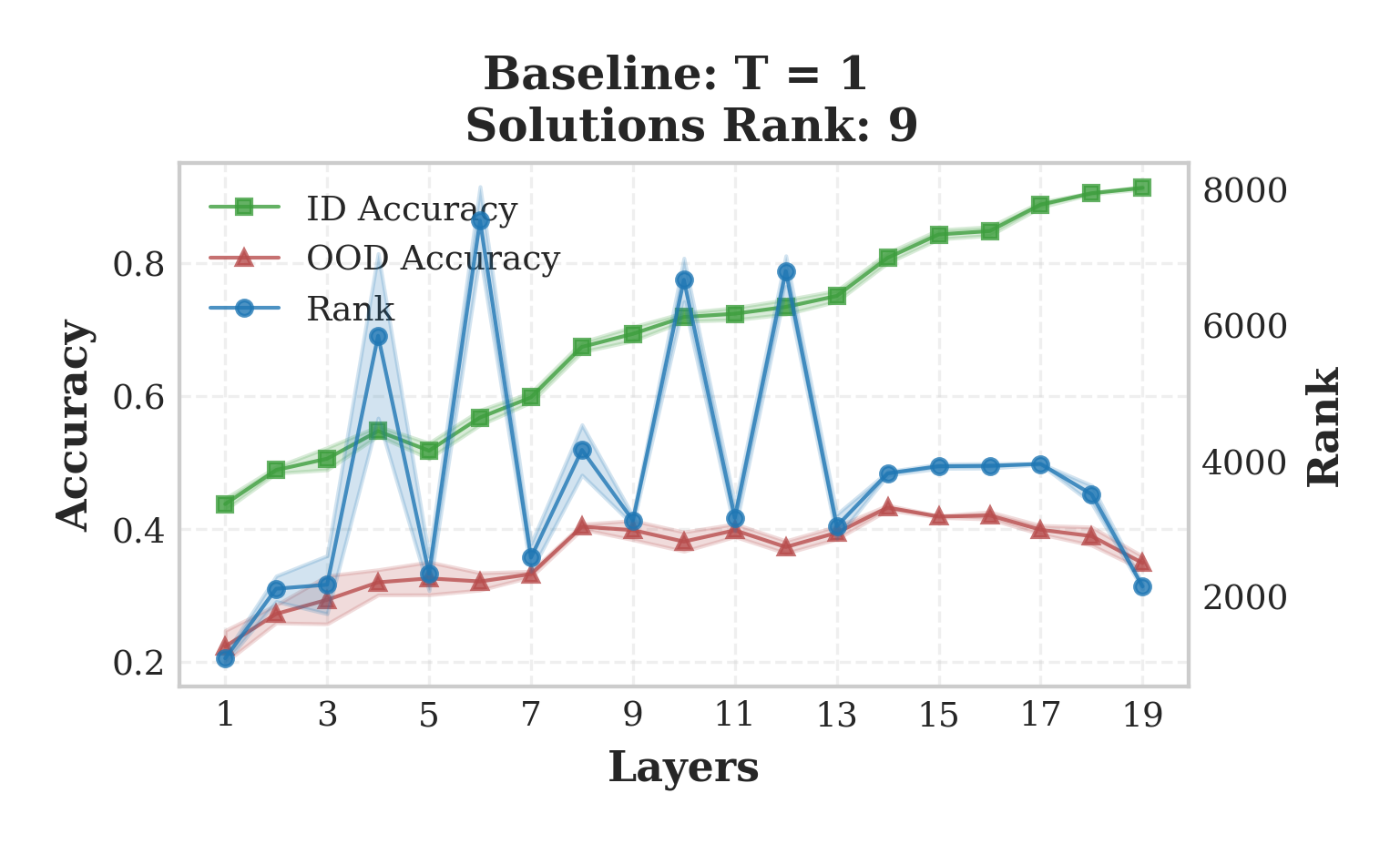}
    }
    {
    \includegraphics[width=0.49\textwidth]{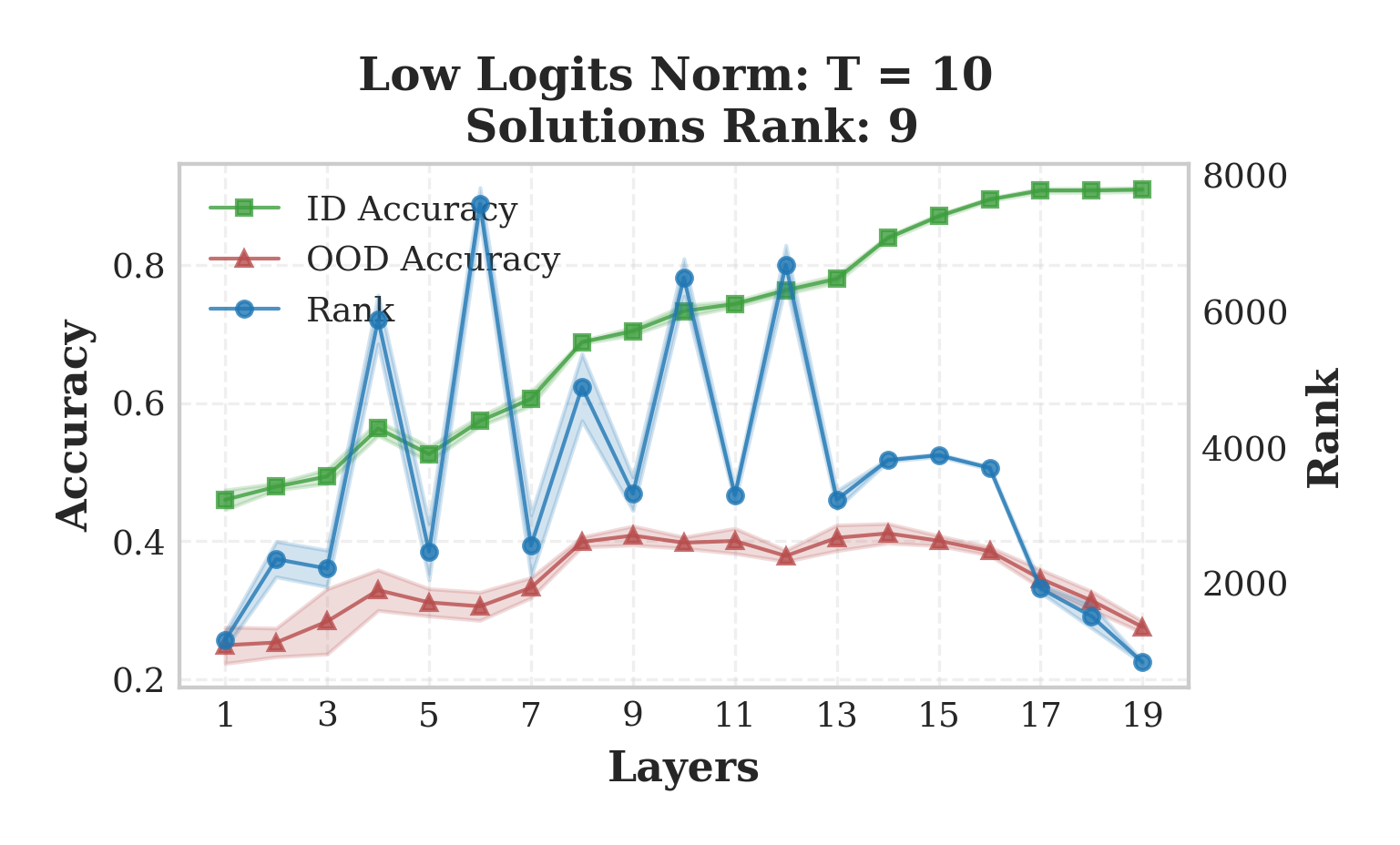}
    }
        {
    \includegraphics[width=0.49\textwidth]{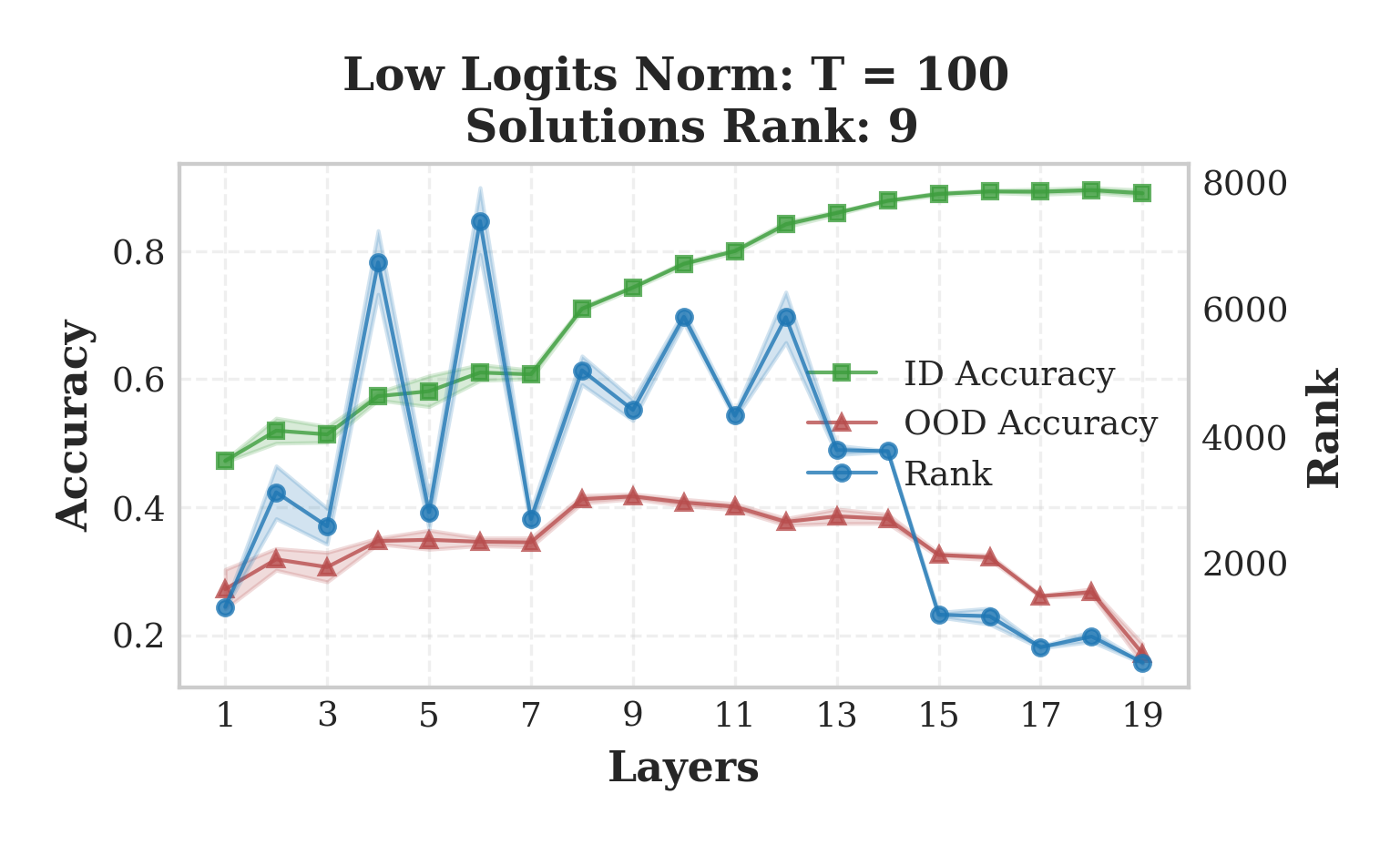}
    }
    {
    \includegraphics[width=0.49\textwidth]{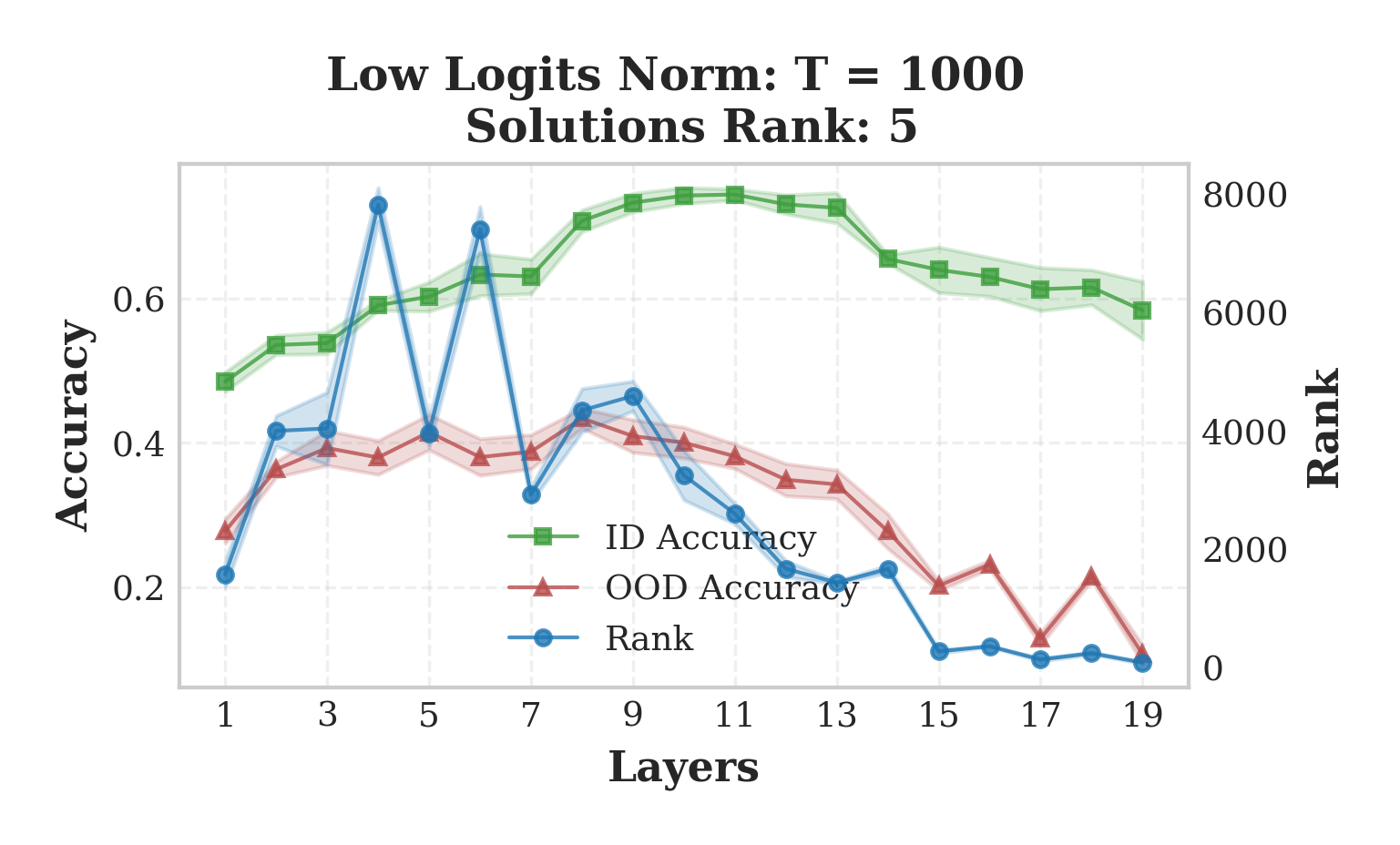}
    }
    \caption{\small{Plot presents the impact of training with high temperature on learned representations and the model's ability to generalize OOD. The higher the temperature, the lower the solutions' rank found by the model. Experiment: ResNet-20 trained on CIFAR-10. The right bottom plot presents an unsuccessful training case.}}
\end{figure}

\subsection{ResNet-34}

\begin{figure}[!h]
    \centering
    {
    \includegraphics[width=0.49\textwidth]{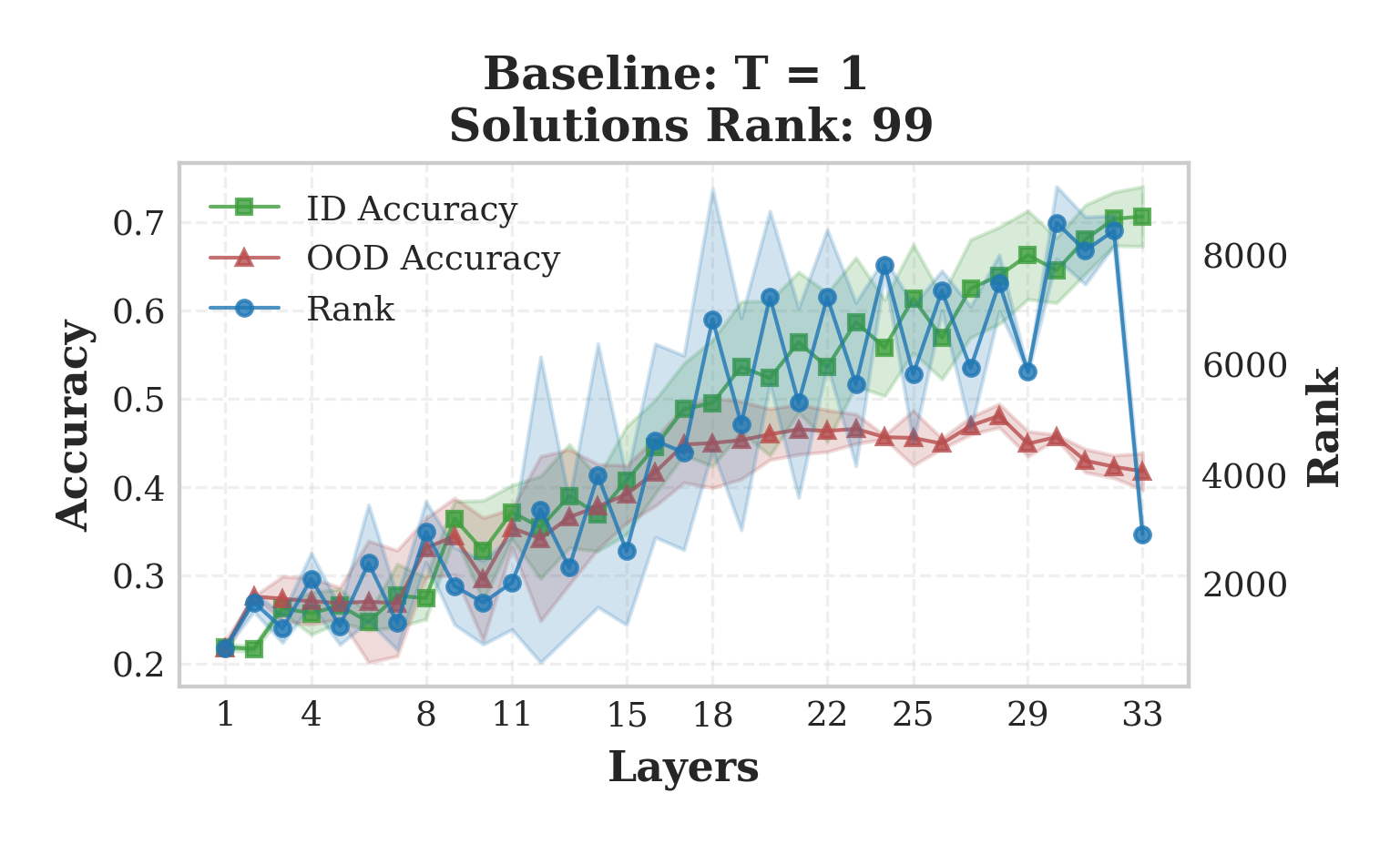}
    }
    {
    \includegraphics[width=0.49\textwidth]{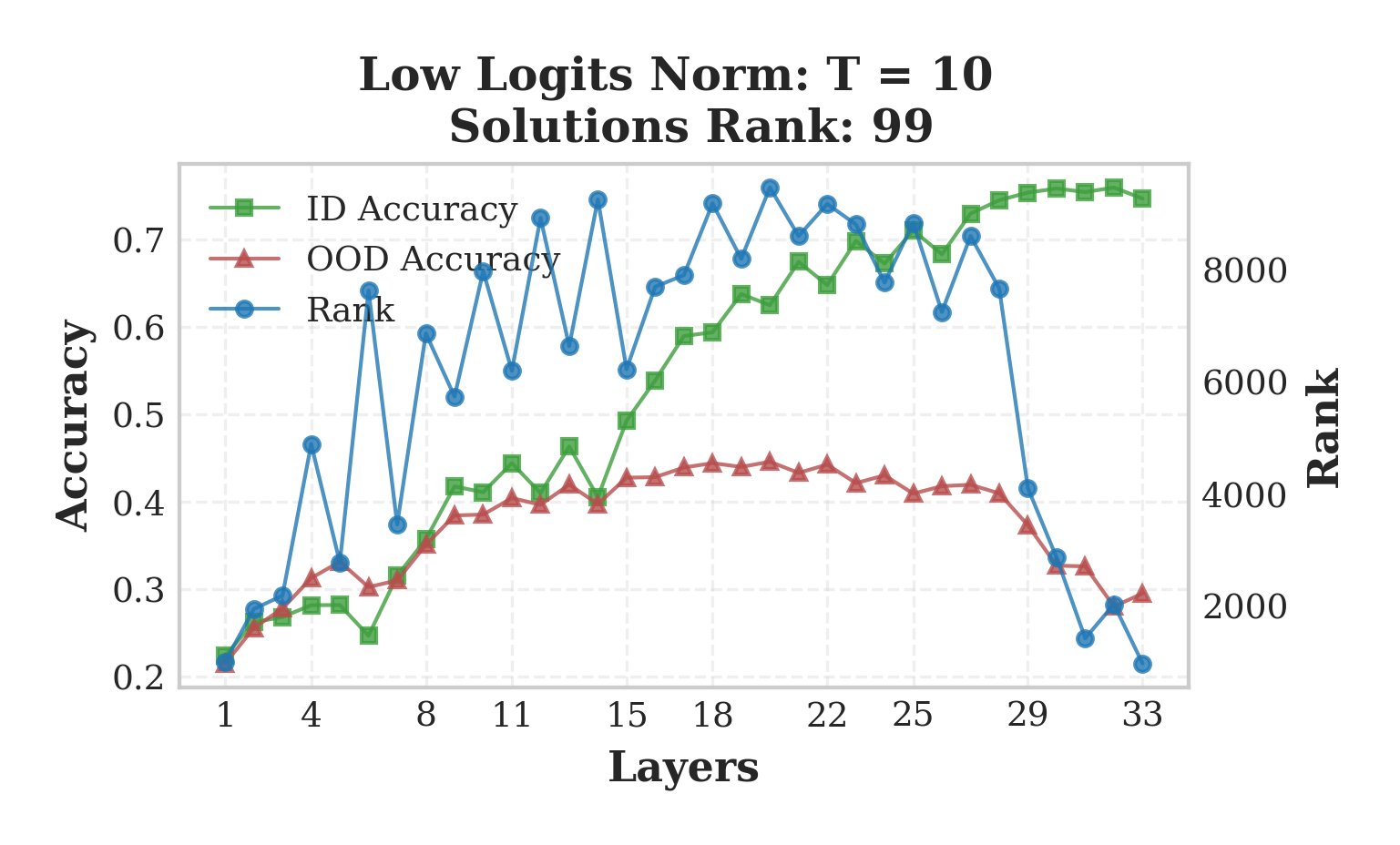}
    }
        {
    \includegraphics[width=0.49\textwidth]{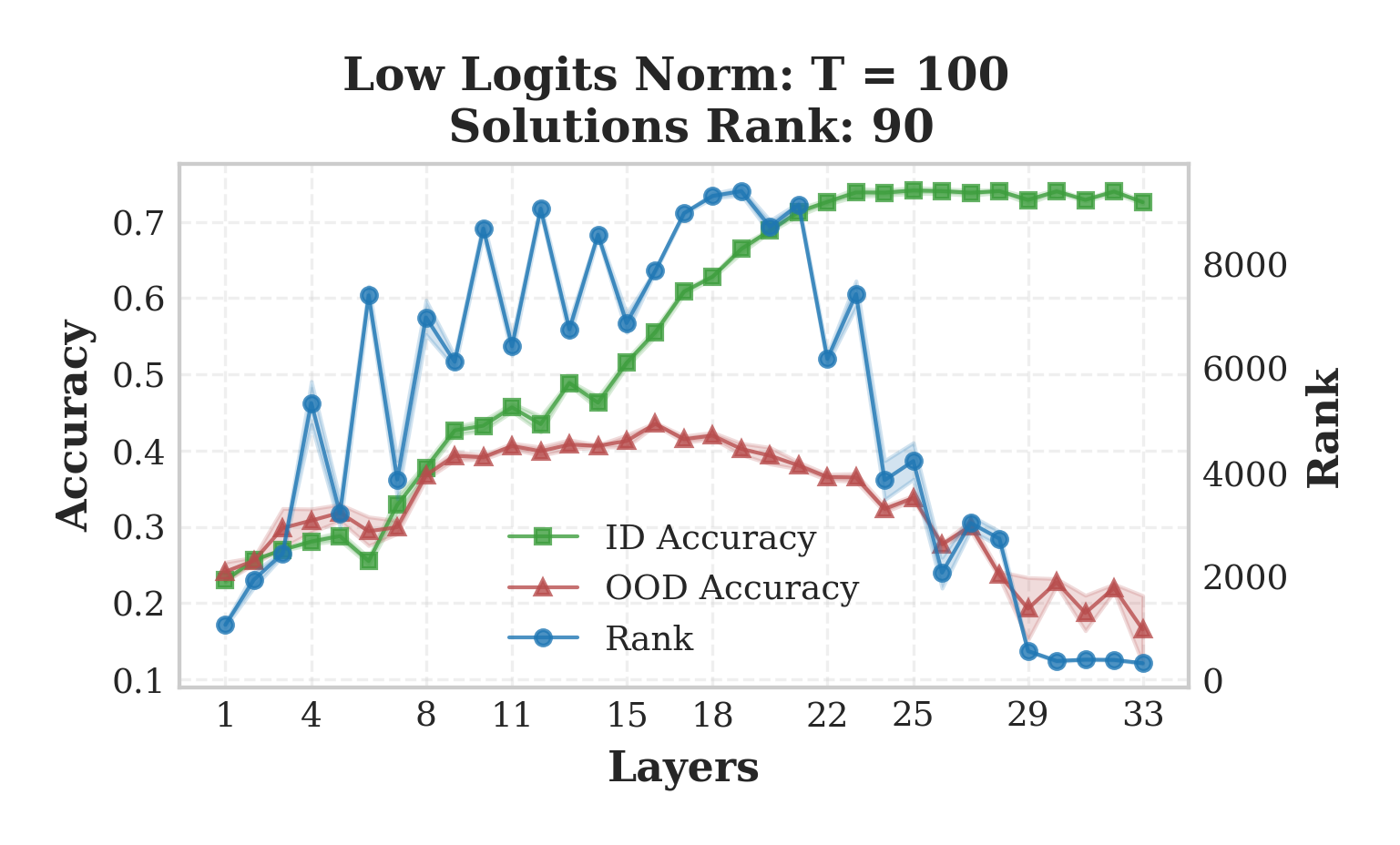}
    }
    {
    \includegraphics[width=0.49\textwidth]{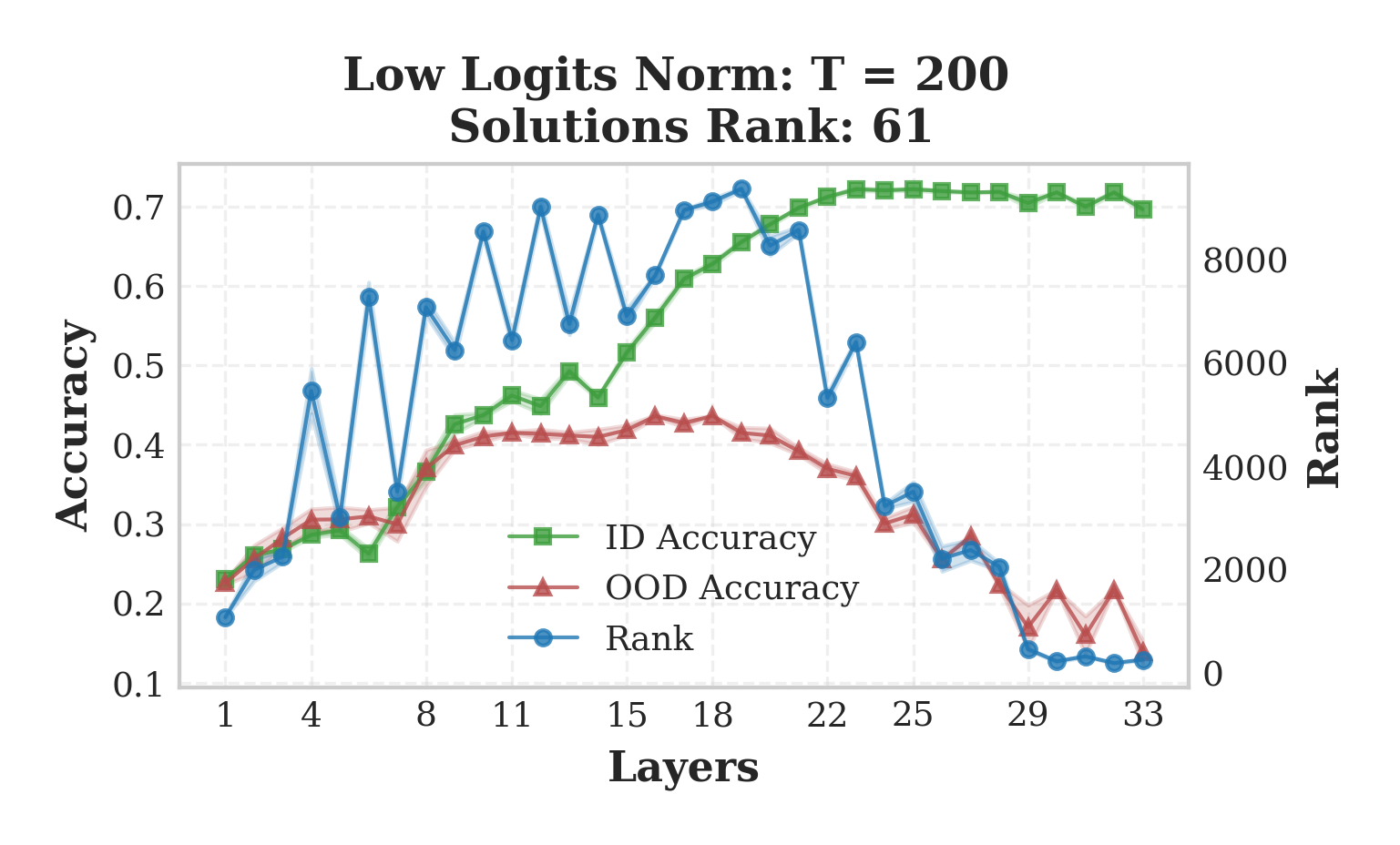}
    }
    \caption{\small{Plot presents the impact of training with high temperature on learned representations and the model's ability to generalize OOD. The higher the temperature, the lower the solutions' rank found by the model. Experiment: ResNet-34 trained on CIFAR-100.}}
\end{figure}

\subsection{VGG-19}

\begin{figure}[!h]
    \centering
    {
    \includegraphics[width=0.49\textwidth]{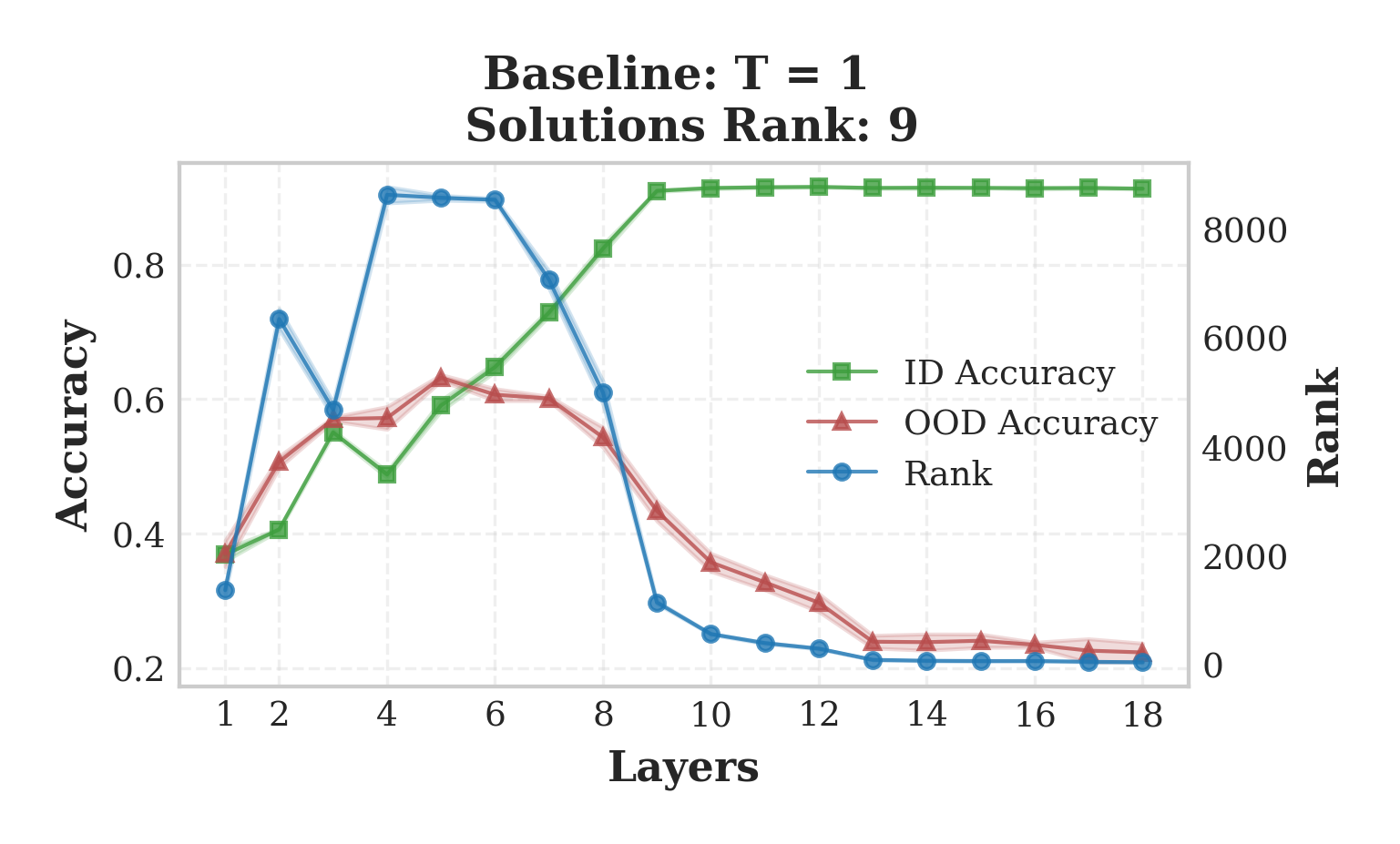}
    }
    {
    \includegraphics[width=0.49\textwidth]{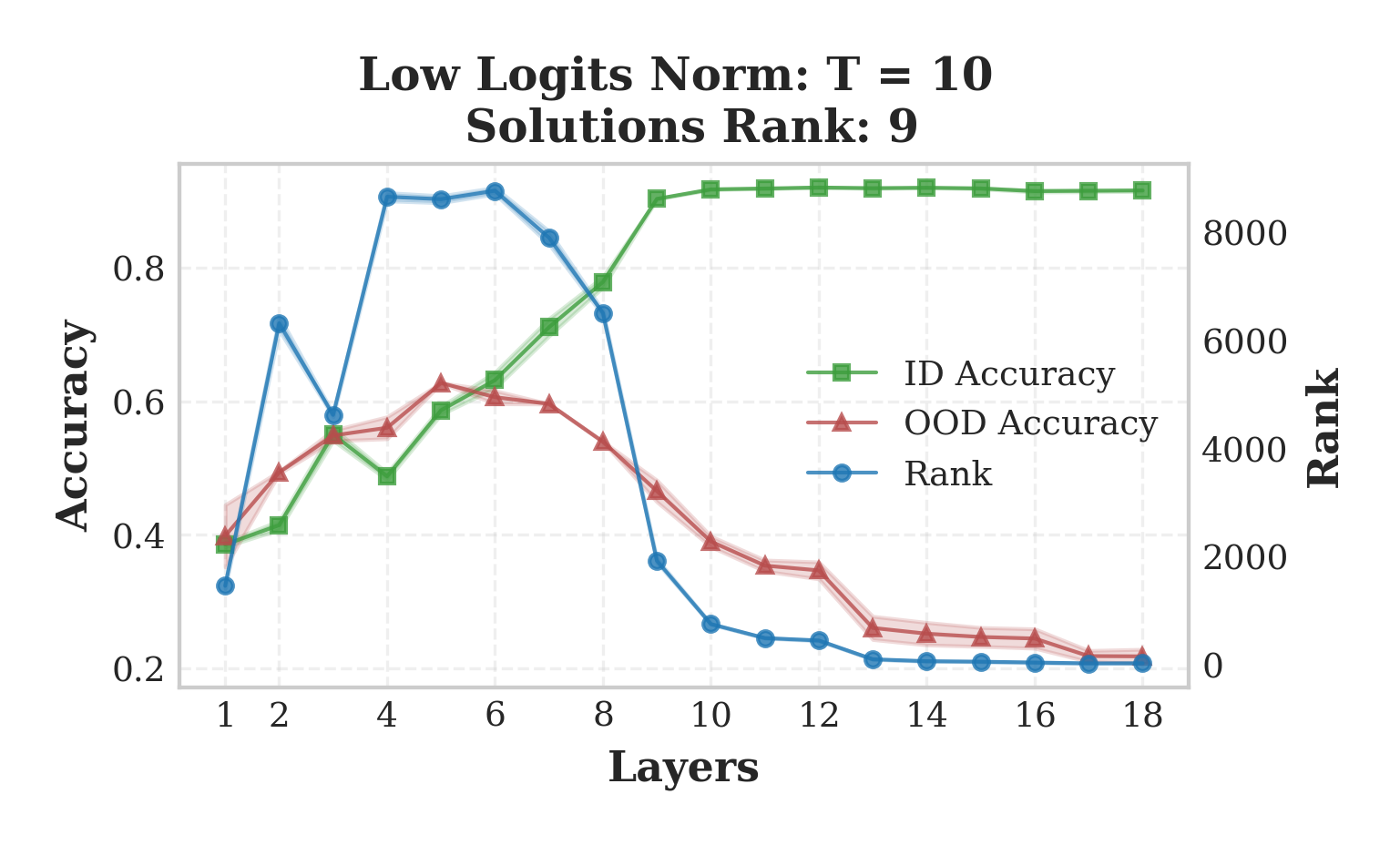}
    }
        {
    \includegraphics[width=0.49\textwidth]{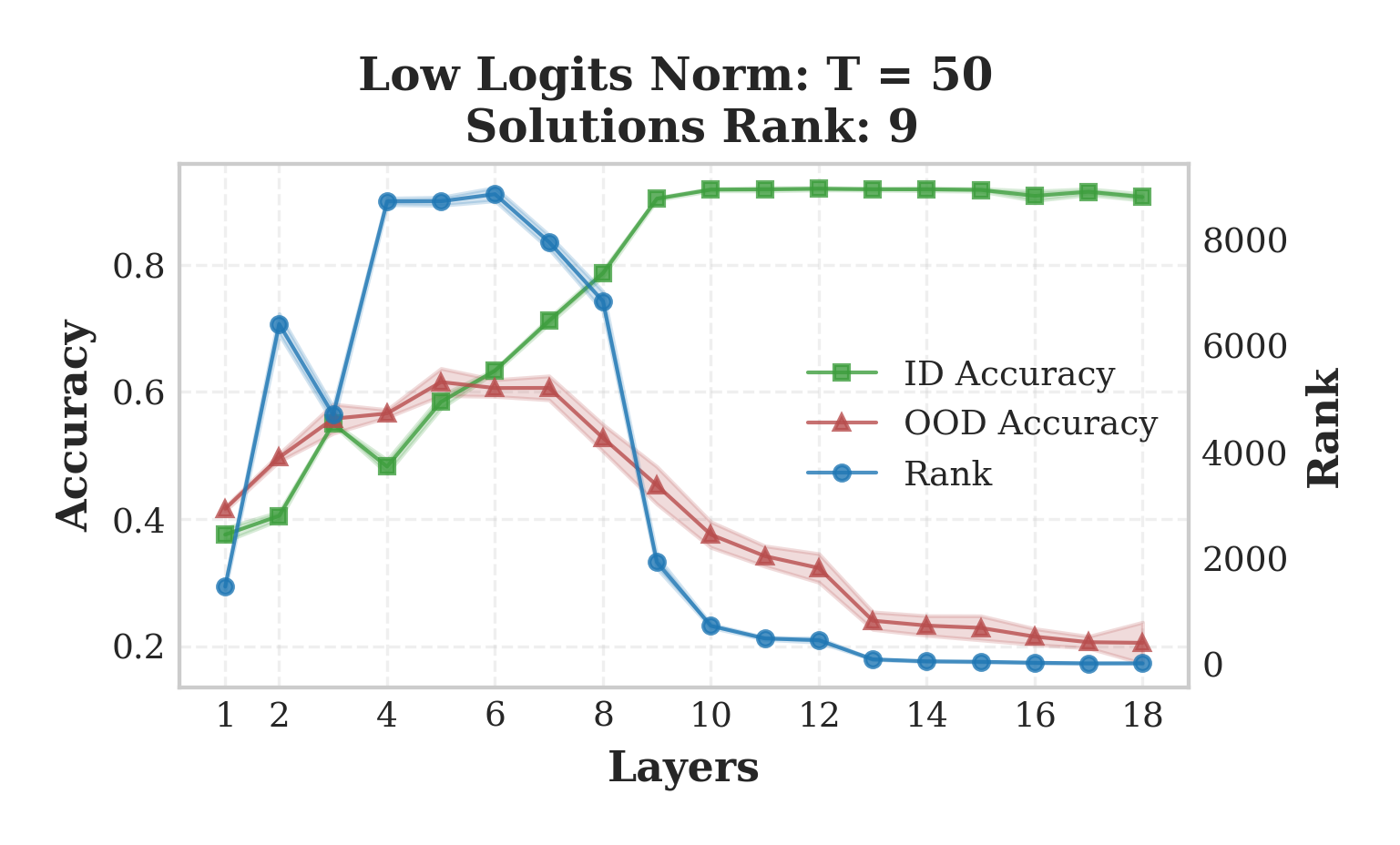}
    }
    {
    \includegraphics[width=0.49\textwidth]{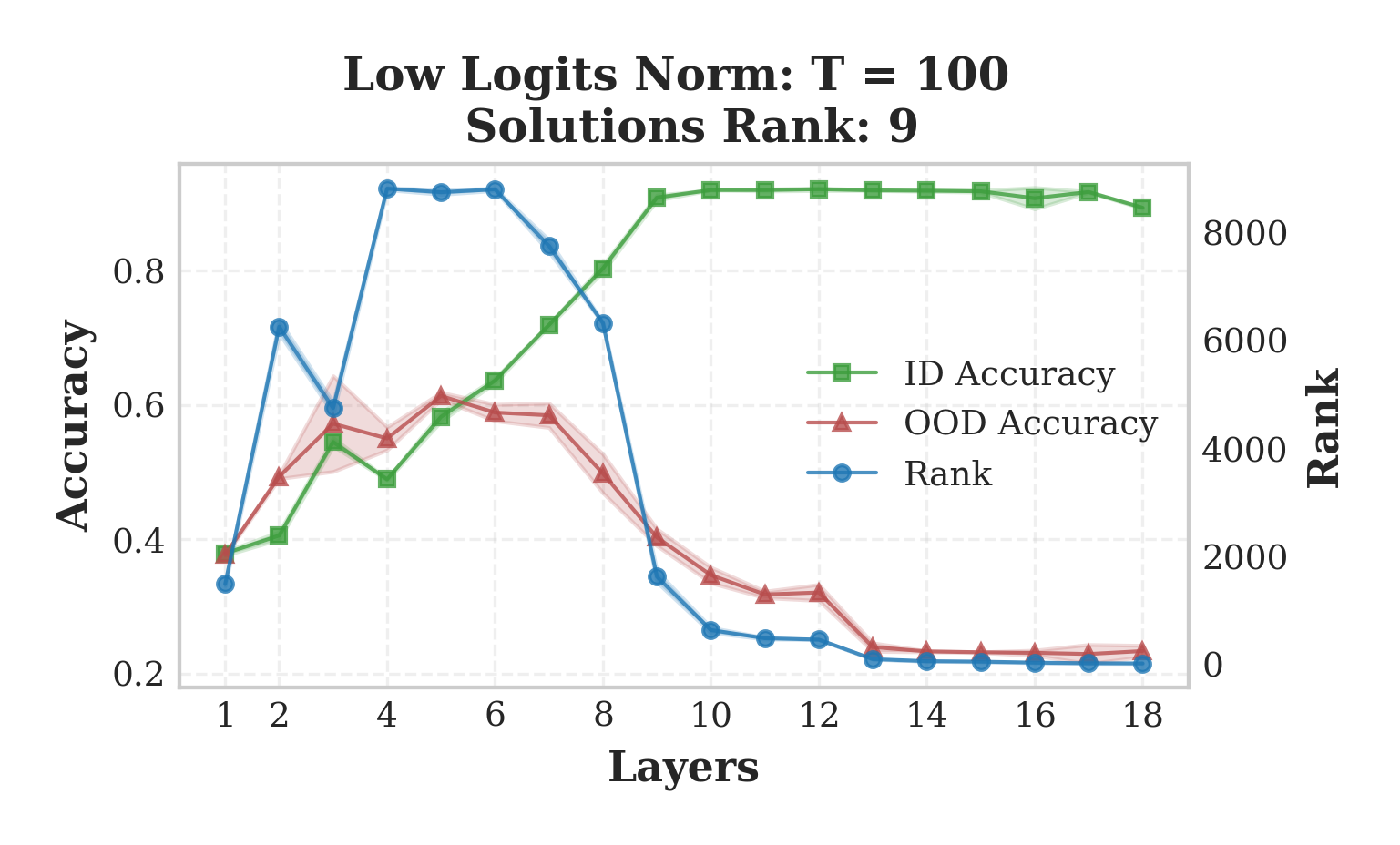}
    }
    \caption{Plot presents the impact of training with high temperature on learned representations and the model's ability to generalize OOD. The higher the temperature, the lower the solutions' rank found by the model. Experiment: VGG-19 trained on CIFAR-10.} 
\end{figure}

\begin{figure}[!h]
    \centering
    {
    \includegraphics[width=0.49\textwidth]{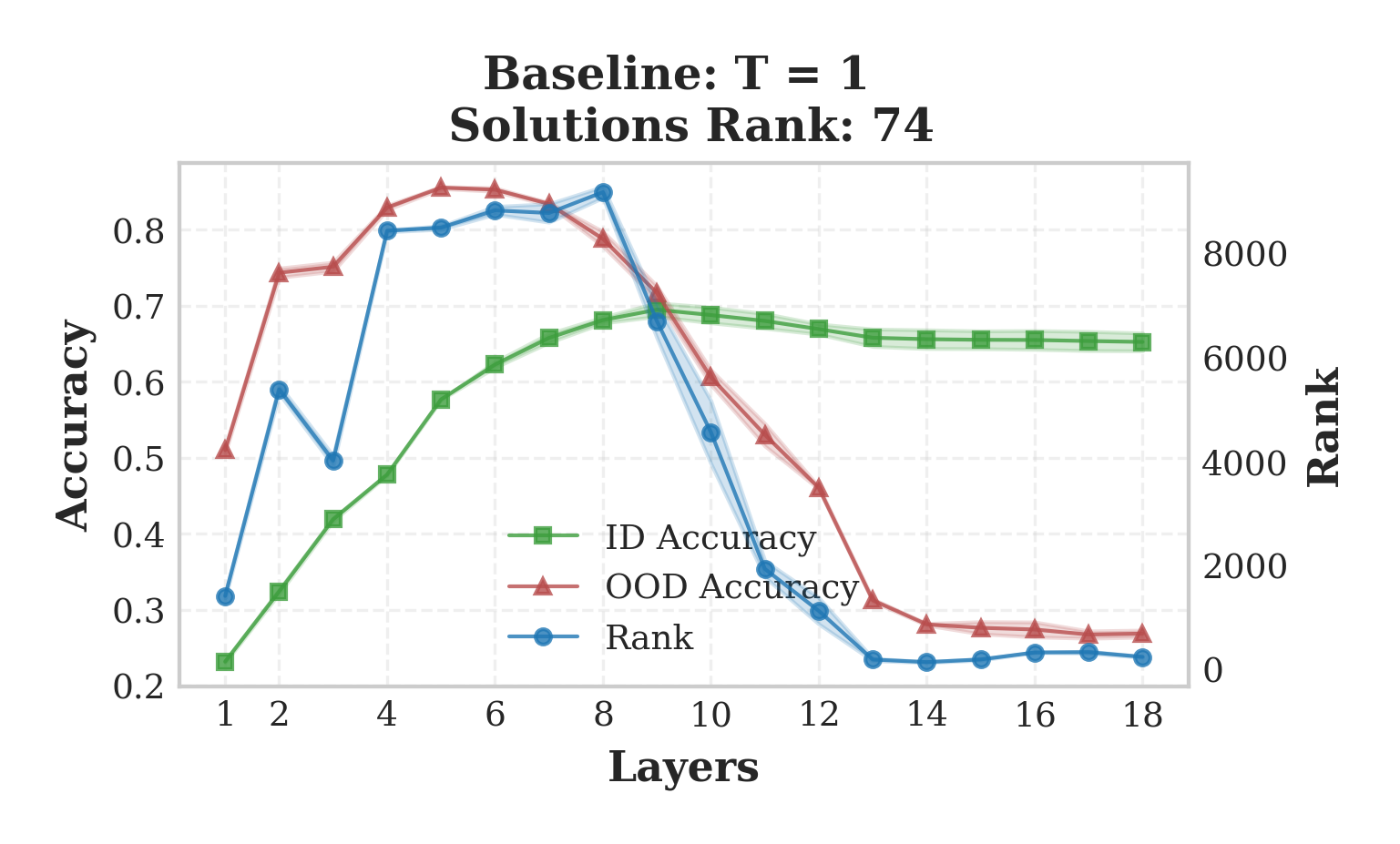}
    }
    {
    \includegraphics[width=0.49\textwidth]{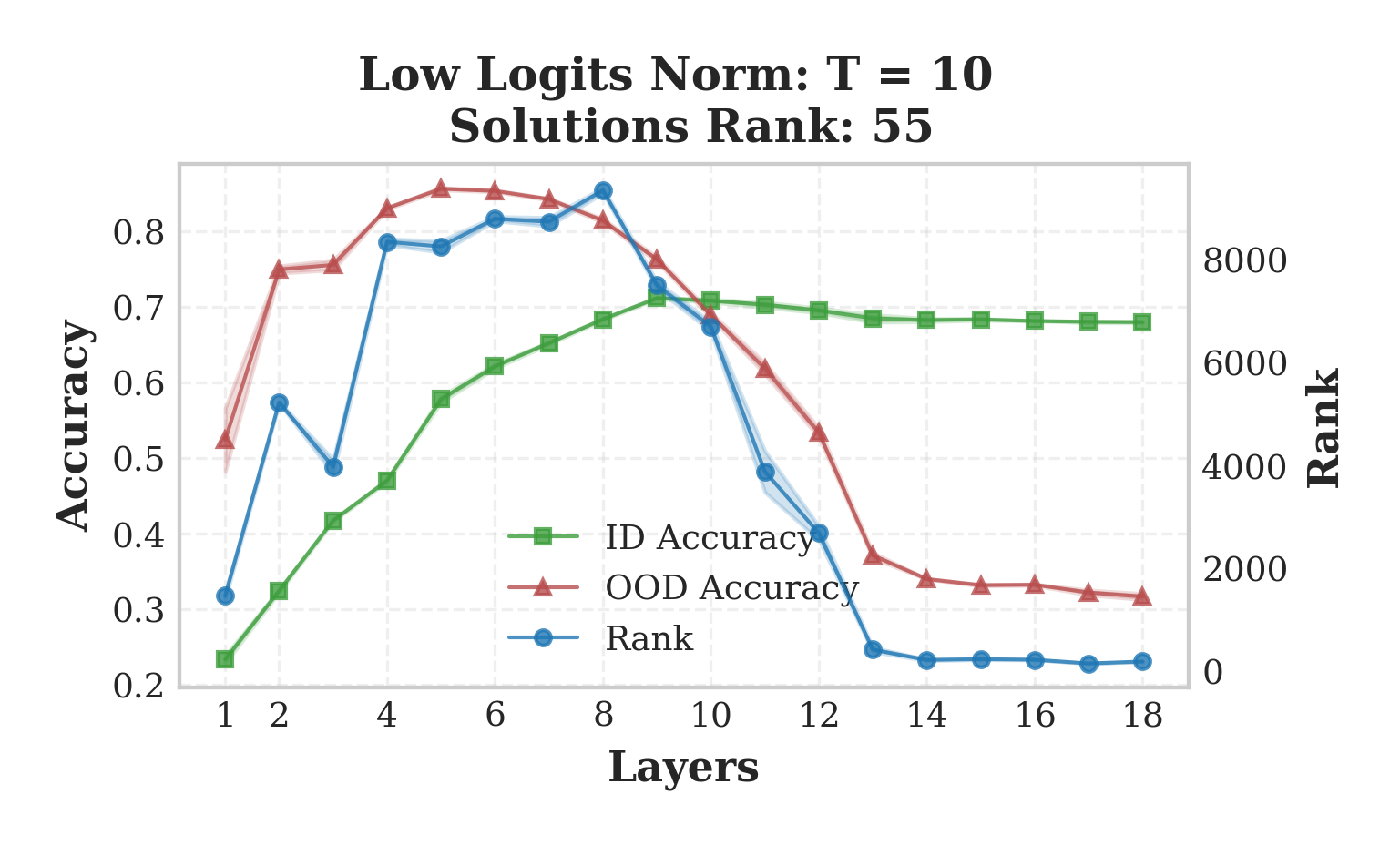}
    }
        {
    \includegraphics[width=0.49\textwidth]{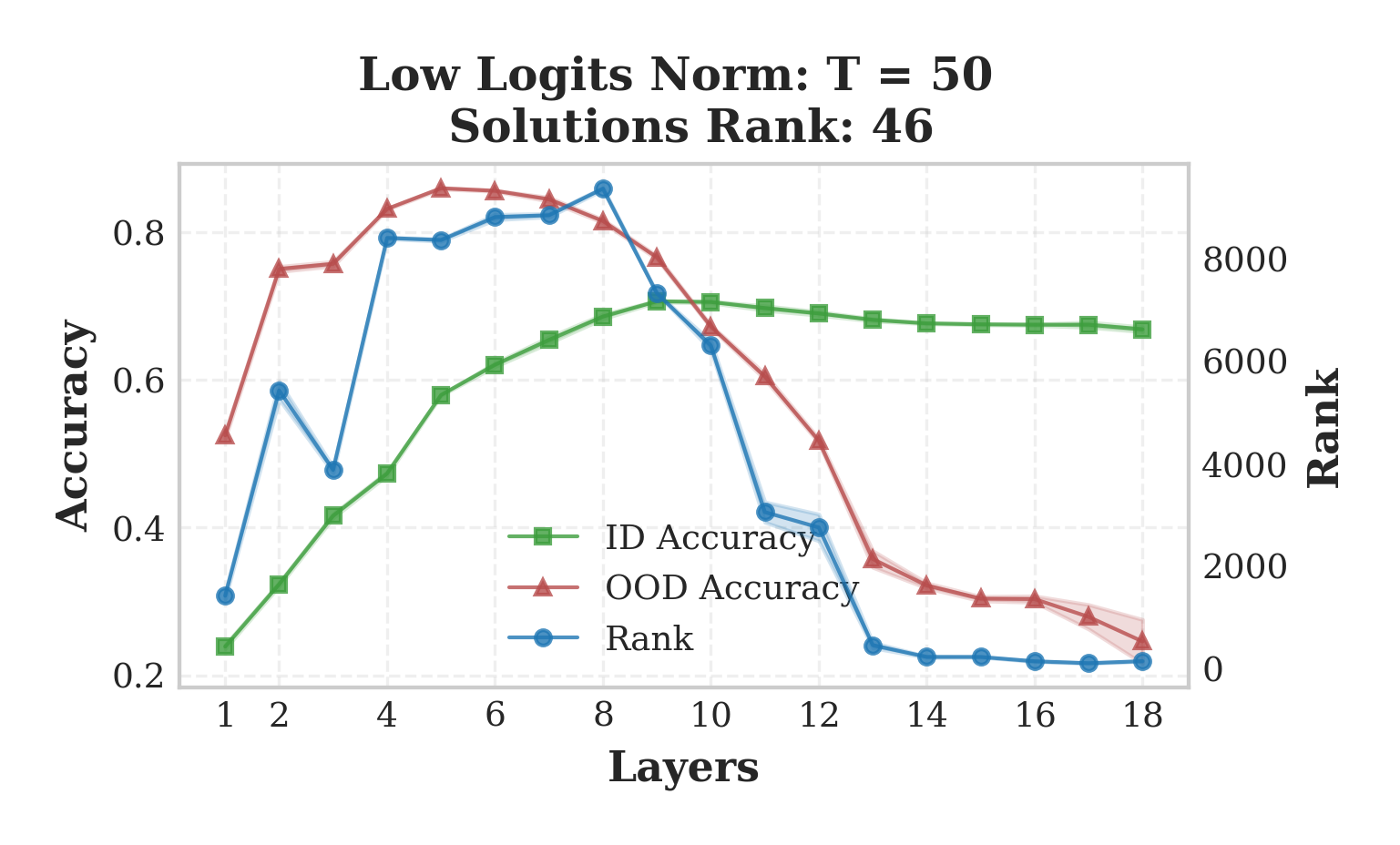}
    }
    {
    \includegraphics[width=0.49\textwidth]{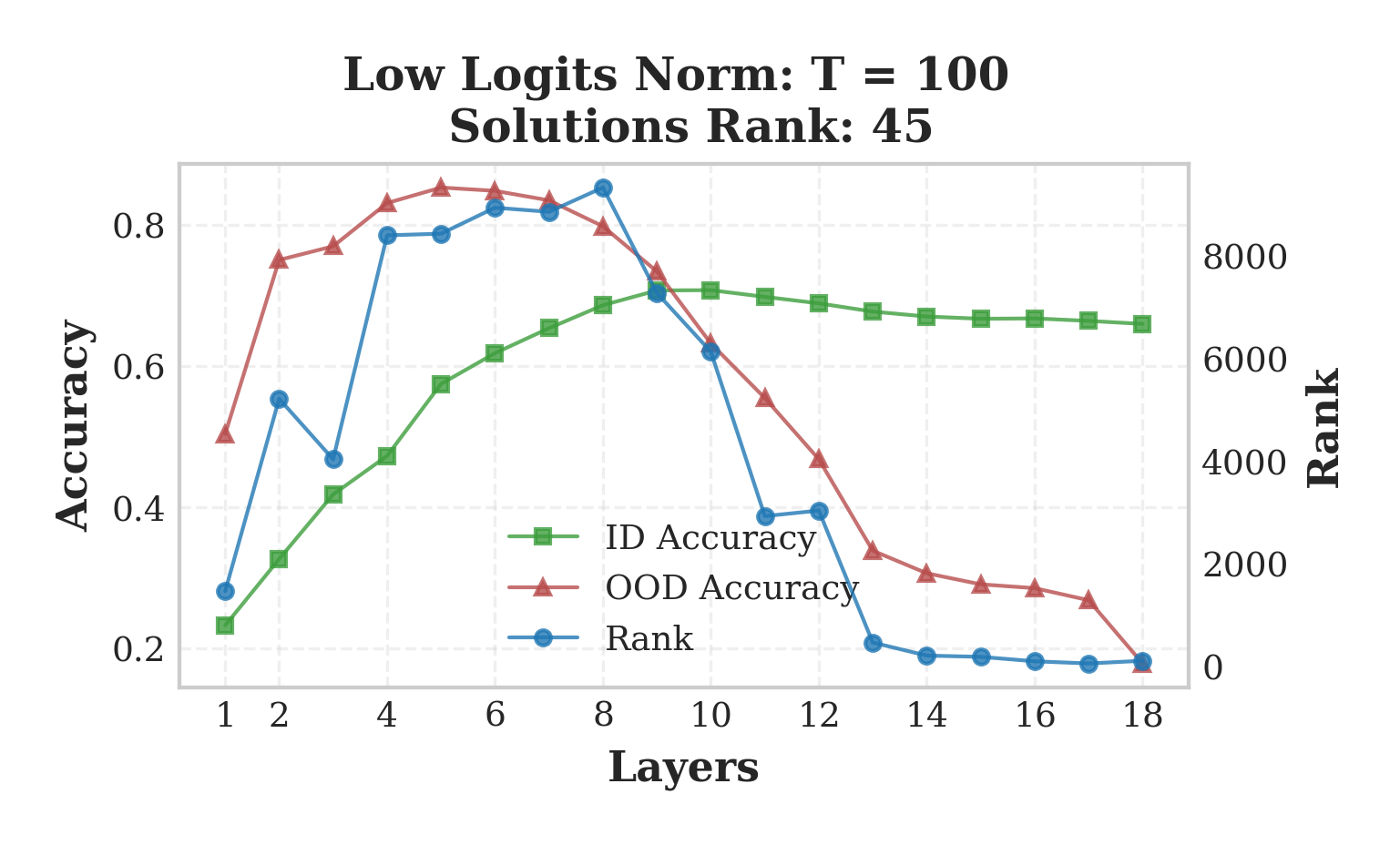}
    }
    \caption{Plot presents the impact of training with high temperature on learned representations and the model's ability to generalize OOD. The higher the temperature, the lower the solutions' rank found by the model. Experiment: VGG-19 trained on CIFAR-100.} 
\end{figure}

\newpage

\section{Extended related works}\label{app:related_works}
Our research intersects with four key areas of deep learning: (1) Neural Collapse and its variants, (2) unconstrained feature models, (3) \softmax temperature effects, and (4) the relationship between model compression and out-of-distribution performance. We provide a comprehensive analysis of each area before situating our contributions.

\subsection{Neural Collapse and Representation Learning}
The Neural Collapse (NC) phenomenon~\cite{papyan2020neuralcollapse} describes a surprising geometric regularity that emerges in the terminal phase of training (TPT), defined as when models achieve 100\% training accuracy. Under balanced class conditions without data augmentation, the penultimate layer representations converge to form an equiangular tight frame (ETF) simplex with rank equal to $C-1$, where $C$ is the number of classes. This structure exhibits three key properties: (i) class means form a simplex ETF, (ii) representations collapse to their class means, and (iii) classifiers align with the class means.

Subsequent work has identified similar geometric structures in intermediate layers. \cite{rangamani23intermediate} and \cite{parker2023neuralcollapseintermediatehidden} demonstrated Intermediate Neural Collapse, where hidden layers develop ETF-like structures before the penultimate layer. However, these works share two limitations with the original NC theory: (1) they provide no mechanistic explanation for why networks converge to these solutions, and (2) they require training to reach TPT to observe the phenomena.

Our findings reveal several fundamental differences from NC:
\begin{itemize}
    \item \textbf{Training Dynamics:} The rank collapse we observe begins in early training phases and persists throughout, without requiring TPT (in fact, most of our models never reach perfect training accuracy)
    \item \textbf{Augmentation Robustness:} Our results hold under strong data augmentation, while NC requires unaugmented datasets~\cite{papyan2020neuralcollapse}
    \item \textbf{Rank Behavior:} We demonstrate networks can find effective solutions with ranks significantly lower than the $C-1$ predicted by NC
    \item \textbf{Mechanistic Control:} We provide both theoretical and empirical evidence showing how rank collapse can be directly controlled through \softmax temperature and other hyperparameters
\end{itemize}

As we detail in Appendix~\ref{app:neural_collapse}, while our observations share similarities with NC, the underlying mechanisms and implications differ substantially. Most notably, none of our models simultaneously satisfy all NC conditions, suggesting our rank collapse phenomenon operates through a distinct pathway.

\subsection{Unconstrained Feature Models}
The unconstrained feature model (UFM) framework~\cite{hong2024neural} was developed to analyze NC under more general conditions, particularly for imbalanced class distributions. This approach treats both network parameters and input features as optimizable variables. The deep UFM (DUMF) extension~\cite{sukenik2023deep,sukenik2024neural} incorporates multiple layers and reveals solutions with rank lower than NC predictions.

However, these models present three key limitations our work addresses:
\begin{enumerate}
    \item \textbf{Measurement Protocol:} DUMF studies measure rank \textit{before} the ReLU activation, observing that ReLU can restore the full rank. This differs fundamentally from our direct measurement of pre-\softmax logits, which directly impact model decisions.
    \item \textbf{Architectural Constraints:} To observe low-rank solutions, these works stack multiple linear layers atop standard backbones and employ high weight decay. This setup was recently shown to induce a low-rank bias~\cite{kobayashi2024weight}.
    \item \textbf{Practical Relevance:} Our experiments demonstrate rank collapse occurs in standard architectures (MLPs, ResNets, VGGs) across multiple datasets without specialized regularization or architectural modifications.
\end{enumerate} 

\subsection{Softmax Temperature Effects}
Temperature scaling has been employed in two distinct contexts:

\paragraph{Inference-Time Adjustment}
\cite{velickovic2024softmax} identified temperature scaling as crucial for obtaining sharp predictions on OOD data, while \cite{guo2017calibration} used it for model calibration. Both approaches only adjust temperature during inference, leaving the learned representations unchanged.

\paragraph{Training-Time Optimization}
In self-supervised learning~\cite{chen2020simple}, temperature acts as a critical hyperparameter controlling representation quality. For LLMs, \cite{jha2024aerosoftmaxonlyllmsefficient} demonstrated temperature's role in private inference scenarios. However, these works neither examine temperature's impact on representation rank nor its relationship to OOD performance.

Our work bridges this gap by demonstrating how \softmax temperature during training affects representation geometry and model behavior. We provide the first systematic study showing temperature's dual role in controlling rank collapse and modulating OOD performance.

\subsection{Model Compression and OOD Performance}

Recent work has revealed complex relationships between model compression, OOD generalization, and detection:

\paragraph{Transfer Learning}
\cite{evci2022head2toe} and \cite{masarczyk2024tunnel} showed intermediate representations can enhance transfer learning, while \cite{harun2024variables} demonstrated width-depth tradeoffs in compressed models.

\paragraph{OOD Detection}
Several works~\cite{ammar2024neconeuralcollapsebased,haas2023linkingneuralcollapsel2,harun2025controllingneuralcollapseenhances} link stronger NC to improved OOD detection. Notably, \cite{haas2023linkingneuralcollapsel2} showed $L_2$ regularization can improve detection at the cost of generalization.

Our work extends these findings by demonstrating how logit norm reduction—whether through architectural choices or temperature scaling—creates a tunable tradeoff between OOD generalization and detection. This provides practitioners with new knobs to optimize models for specific deployment scenarios.

\subsection{Synthesis of Contributions}
By unifying insights from these diverse areas, our work:
\begin{itemize}
    \item Establishes rank collapse as a fundamental phenomenon distinct from NC
    \item Provides mechanistic explanations and control strategies through temperature scaling
    \item Reveals new connections between representation geometry and OOD behavior
    \item Offers practical guidelines for model optimization via temperature tuning
\end{itemize}

\section{What hyperparameters act as softmax temperature?}\label{app:architectural_hyperparameters}

To systematically investigate the impact of hyperparameters on logits norm and model behavior, we conduct a series of controlled experiments using an 8-layer MLP with 2048 hidden units per layer trained on CIFAR-10. The models were initialized using distinct schemes: \{Kaiming initialization~\cite{he2015delvingdeeprectifierssurpassing}, PyTorch default initialization~\cite{paszke2019pytorch}, or Normal distribution with specified standard deviation $\sigma$\}.

\begin{figure}[!h]
    \centering
    {
    \includegraphics[width=0.96\textwidth]{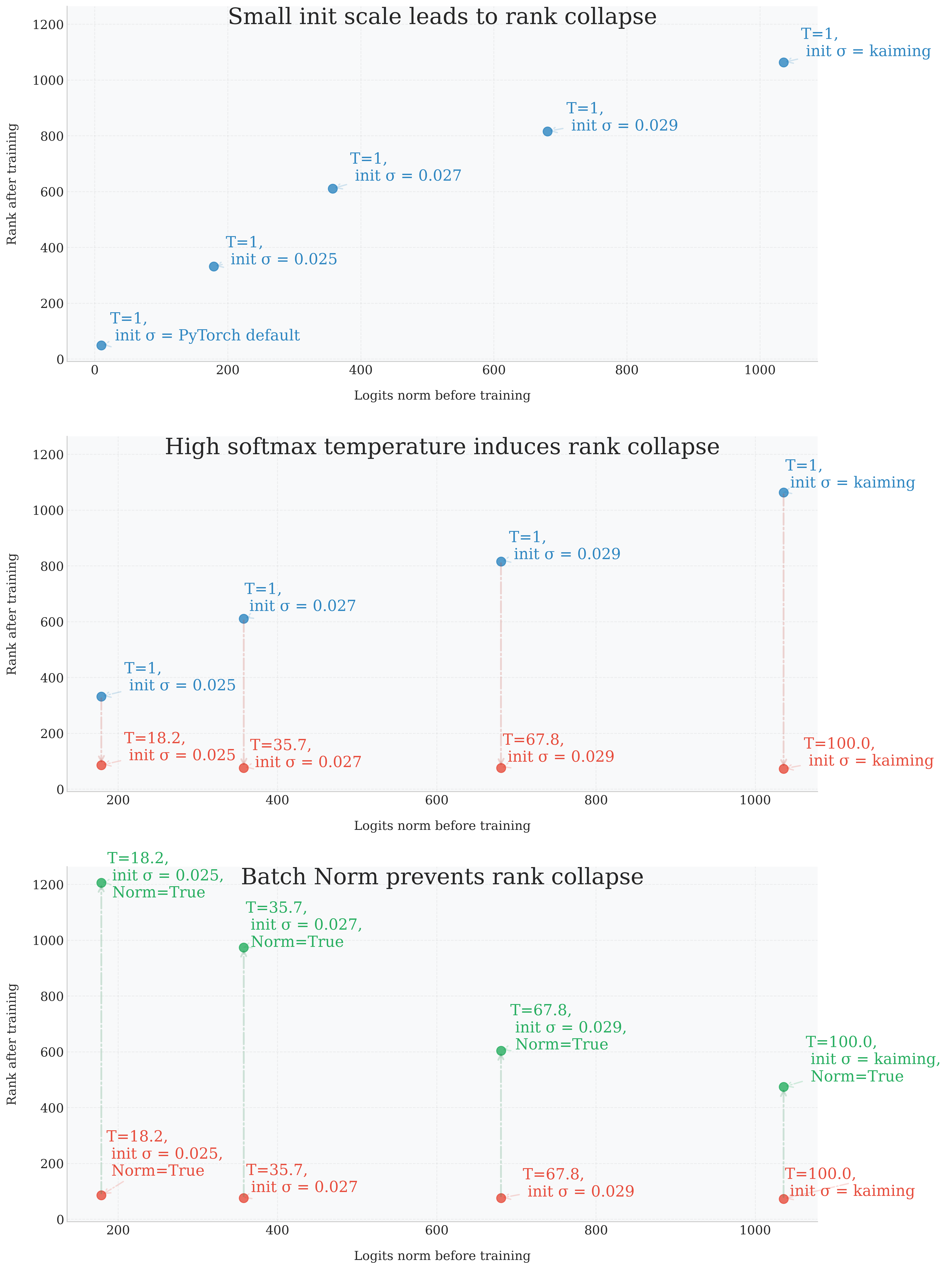}
    }
    \caption{
\textbf{(Top)} Initialization effects on pre-training logits norm (x-axis) demonstrate a clear relationship between initialization scale and model behavior. Smaller $\sigma$ values yield reduced logit norms, effectively equivalent to training with higher temperature. This relationship persists post-training, as shown by the nearly linear correlation with penultimate layer rank (y-axis) -- models with lower initial logits norms consistently achieve lower final ranks.
\textbf{(Middle)} When logits norms are normalized across initialization schemes through temperature adjustment, all models converge to similar rank solutions, confirming the dominant role of logits norm in representation learning.
\textbf{(Bottom)} Introduction of layer normalization prevents representation collapse despite low logits norm, demonstrating the critical importance of normalization layers in maintaining representation quality.
\label{fig:init_temp_batchnorm}
}
\end{figure}

Our first experiment establishes a fundamental relationship between initialization scale and model behavior. As shown in Figure~\ref{fig:init_temp_batchnorm} (top), the choice of initialization directly controls the pre-training logits norm (x-axis), with lower $\sigma$ values producing proportionally smaller norms. Crucially, this initial condition determines the model's final state, as evidenced by the strong linear relationship between initial logits norm and post-training rank (y-axis).

To isolate the effect of logits norm, we conduct a second experiment where we normalize this quantity across initialization schemes through precise temperature adjustment. Remarkably, as demonstrated in Figure~\ref{fig:init_temp_batchnorm} (middle), this normalization causes all models to converge to nearly identical rank solutions, regardless of their initialization. This compelling result confirms that logits norm is the primary determinant of representation quality.

We emphasize that this relationship is asymmetric. Attempting to match larger initial logits norms through temperature reduction fails due to fundamental limitations. This comes from the fact that lower init (especially in deeper networks) decreases the rank of the representations, even for randomly initialized models, and as highlighted in Section~\ref{sec:DNN_analysis}, collapsed representations collapse the gradients, making it much harder for the model to escape this collapsed regime. This critical insight explains why careful initialization and other techniques preserving the rank of the representation remain essential in deep learning.

Our final experiment reveals the protective effect of normalization layers. By applying layer normalization to models that would otherwise collapse (Figure~\ref{fig:init_temp_batchnorm}, bottom), we demonstrate that normalization layers enable rank improvement through alternative mechanisms that don't depend on logits norm amplification. This finding has important practical implications: proper normalization can prevent representation collapse even in unfavorable optimization conditions.

\section{Results for MSE}\label{app:mse_results}

\begin{wraptable}[9]{r}{65mm}

\centering
\centering
\resizebox{\linewidth}{!}{%
\begin{tabular}{@{}ccccccc@{}}
\toprule
\multirow{2}{*}{Architecture} & \multicolumn{3}{c}{Baseline} & \multicolumn{3}{c}{Low Logit Norm}\\ 
& $\kappa$ $\downarrow$ & $\rho$ $\downarrow$ & SR & $\kappa$ $\downarrow$ & $\rho$ $\downarrow$ & SR \\ \midrule
 ResNet34       &    100\%   &  10\%   &   99   & 90\% & 20\% & 95 \\
 \bottomrule
\end{tabular}%
}
\caption{The results for ResNet-34 trained on CIFAR-100 with MSELoss instead of CrossEntropy.}
\label{tab:mse_comparison}

\end{wraptable}

To determine whether our findings generalize beyond Cross Entropy loss, we conducted additional experiments using MSELoss with ResNet-34 on CIFAR-100. By applying the \softmax transformation to model outputs prior to loss computation, we maintained control over the temperature parameter. Remarkably, as shown in Table~\ref{tab:mse_comparison}, we observe the same fundamental trends with MSELoss as with Cross Entropy, without any hyperparameter optimization. However, we note that the differences are milder than in the case of CrossEntropy; we believe this difference should be attributed to the poorly designed parameters. This strongly suggests that our core observations are not loss-function specific, but rather reflect fundamental properties of deep learning optimization. While further tuning could potentially improve absolute performance metrics, the consistent patterns across different loss functions provide compelling evidence for the robustness of our findings.

\newpage

\section{Finegrained temperature experiments}\label{app:more_temperatures}

\begin{table}[!h]
\centering
\resizebox{1\columnwidth}{!}{%
\begin{tabular}{@{}ccccccccccccc@{}}
\toprule
\multirow{2}{*}{Model/Temp} & \multicolumn{3}{c}{T=1} & \multicolumn{3}{c}{T=10} & \multicolumn{3}{c}{T=100} & \multicolumn{3}{c}{T=1000}\\ 
& $\kappa$ $\downarrow$ & $\rho$ $\downarrow$ & SR & $\kappa$ $\downarrow$ & $\rho$ $\downarrow$ & SR & $\kappa$ $\downarrow$ & $\rho$ $\downarrow$ & SR & $\kappa$ $\downarrow$ & $\rho$ $\downarrow$ & SR\\ \midrule
 ResNet18  &   100\%    &  8\%   &  9     & 88\% & 42\% & 9 & 81\% & 49\% & 9 & 81\%    & 52\%    &   9   \\
 ResNet20  &  100\%     &  10\%   &   9   &  94\% & 24\% & 9 & 83\% &  34\% & 9 & - &  -  &  -  \\ 
 \midrule
 \multirow{2}{*}{Model/Temp} & \multicolumn{3}{c}{T=1} & \multicolumn{3}{c}{T=10} & \multicolumn{3}{c}{T=50} & \multicolumn{3}{c}{T=100}\\ 
 & $\kappa$ $\downarrow$ & $\rho$ $\downarrow$ & SR & $\kappa$ $\downarrow$ & $\rho$ $\downarrow$ & SR & $\kappa$ $\downarrow$ & $\rho$ $\downarrow$ & SR & $\kappa$ $\downarrow$ & $\rho$ $\downarrow$ & SR\\ \midrule
 VGG19  &  53\%     &  64\%   &   9   &  59\% & 65\% & 9 & 59\% &  66\% & 9 & 59\% &  63\%  &  9  \\ 
 
 \bottomrule
\end{tabular}%
}
\caption{Supplementary results with fine-grained results on CIFAR-10. To compute $\rho$, we used SVHN as an OOD dataset. When results are not provided, the training did not finish successfully with the given hyperparameters.}
\label{tab:cifar10_finegrained_results}
\end{table}

\begin{table}[!h]
\centering
\resizebox{1\columnwidth}{!}{%
\begin{tabular}{@{}ccccccccccccc@{}}
\toprule
\multirow{2}{*}{Model/Temp} & \multicolumn{3}{c}{T=1} & \multicolumn{3}{c}{T=10} & \multicolumn{3}{c}{T=100} & \multicolumn{3}{c}{T=200}\\ 
& $\kappa$ $\downarrow$ & $\rho$ $\downarrow$ & SR & $\kappa$ $\downarrow$ & $\rho$ $\downarrow$ & SR & $\kappa$ $\downarrow$ & $\rho$ $\downarrow$ & SR & $\kappa$ $\downarrow$ & $\rho$ $\downarrow$ & SR\\ \midrule
 ResNet18  &   100\%    &  10\%   &  99     & 94\% & 30\% & 99 & 88\% & 48\% & 94 & 81\%    & 43\%    &   57    \\
 ResNet34  &  100\%     &  12\%   &   99   &  91\% & 37\% & 99 & 72\% &  49\% & 90 & 72\% &  50\%  &  61  \\ 
 \midrule
 \multirow{2}{*}{Model/Temp} & \multicolumn{3}{c}{T=1} & \multicolumn{3}{c}{T=10} & \multicolumn{3}{c}{T=50} & \multicolumn{3}{c}{T=100}\\ 
 & $\kappa$ $\downarrow$ & $\rho$ $\downarrow$ & SR & $\kappa$ $\downarrow$ & $\rho$ $\downarrow$ & SR & $\kappa$ $\downarrow$ & $\rho$ $\downarrow$ & SR & $\kappa$ $\downarrow$ & $\rho$ $\downarrow$ & SR\\ \midrule
 VGG19  &  53\%     &  69\%   &   74   &  53\% & 62\% & 55 & 53\% &  68\% & 46 & 53\% &  69\% & 45  \\ 
 
 \bottomrule
\end{tabular}%
}
\caption{Supplementary results with fine-grained results on CIFAR-100. To compute $\rho$, we used SVHN as an OOD dataset. When results are not provided, the training did not finish successfully with the given hyperparameters.}
\label{tab:cifar100_finegrained_results}
\end{table}

\begin{table}[!h]
\centering
\resizebox{1\columnwidth}{!}{%
\begin{tabular}{@{}ccccccccccccc@{}}
\toprule
\multirow{2}{*}{Architecture} & \multicolumn{3}{c}{T=1} & \multicolumn{3}{c}{T=10} & \multicolumn{3}{c}{T=100} & \multicolumn{3}{c}{T=1000}\\ 
& $\kappa$ $\downarrow$ & $\rho$ $\downarrow$ & SR & $\kappa$ $\downarrow$ & $\rho$ $\downarrow$ & SR & $\kappa$ $\downarrow$ & $\rho$ $\downarrow$ & SR & $\kappa$ $\downarrow$ & $\rho$ $\downarrow$ & SR\\ \midrule
 ResNet34       &    100\%   &  5\%   &   512   & 94\% & 2\% & 462 &  94\% & 11\%  & 262 & 82\%   &  23\%    &   122   \\
 ResNet50  &   100\%    &  5\%   &  947     & 100\% & 2\% & 509  & 90\% & 14\% & 249 & 78\%    & 22\%    &   128   \\
 \bottomrule
\end{tabular}%
}
\caption{Supplementary results with fine-grained results on ImageNet-1k. To compute $\rho$, we used CIFAR-100 as an OOD dataset. When results are not provided, the training did not finish successfully with the given hyperparameters.}
\label{tab:finegrained_main-table}
\end{table}

\newpage

\section{Additional experiments -- logits norm}\label{app:growing_logits_norm}

\begin{figure}[!h]
    \centering
    {
    \includegraphics[width=0.48\textwidth]{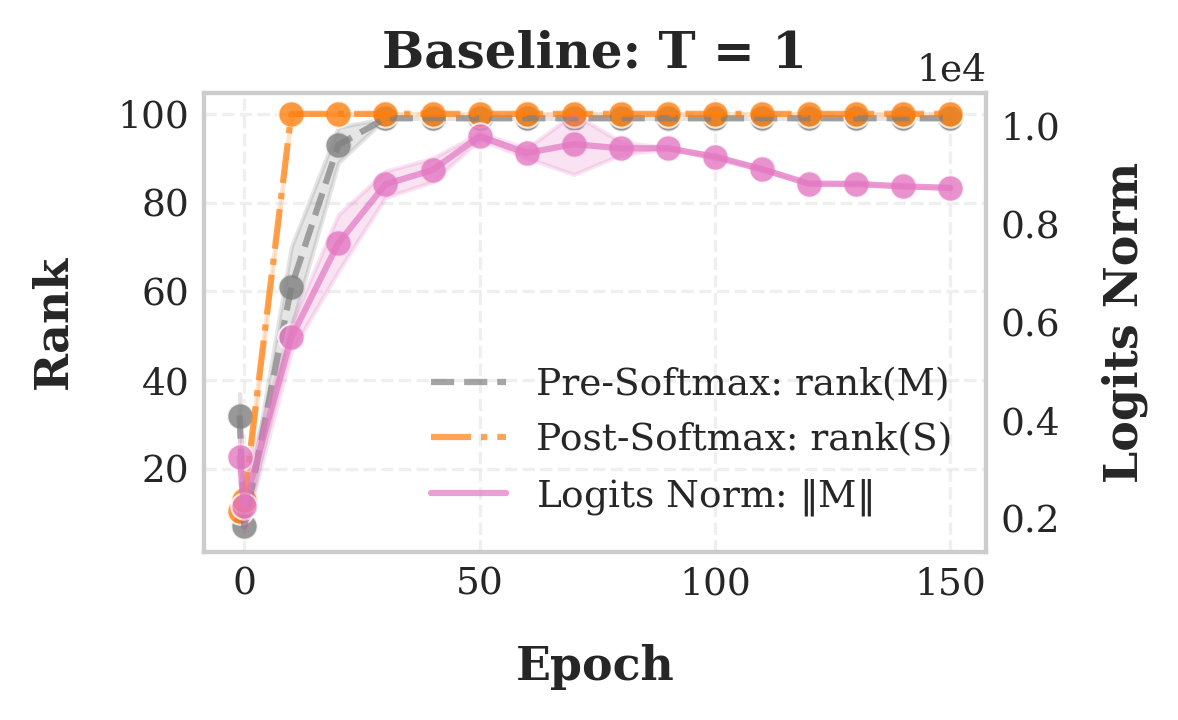}
    }
    {
    \includegraphics[width=0.48\textwidth]{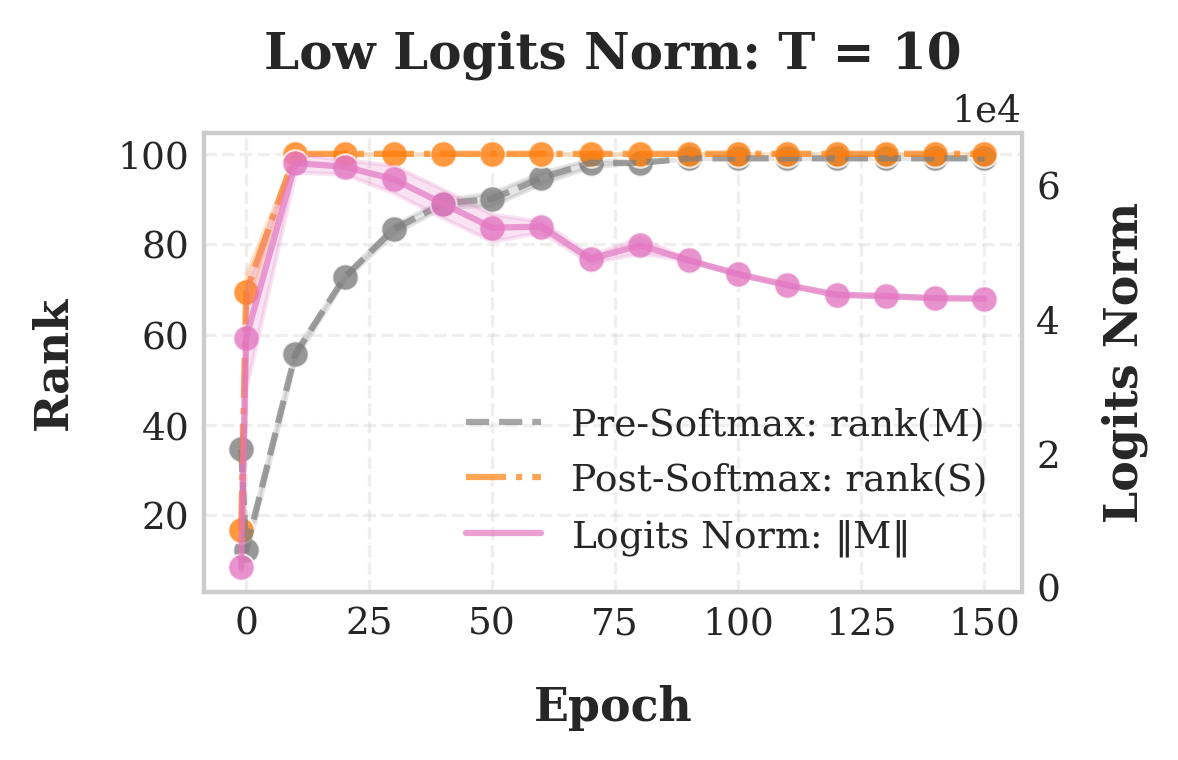}
    }
        {
    \includegraphics[width=0.48\textwidth]{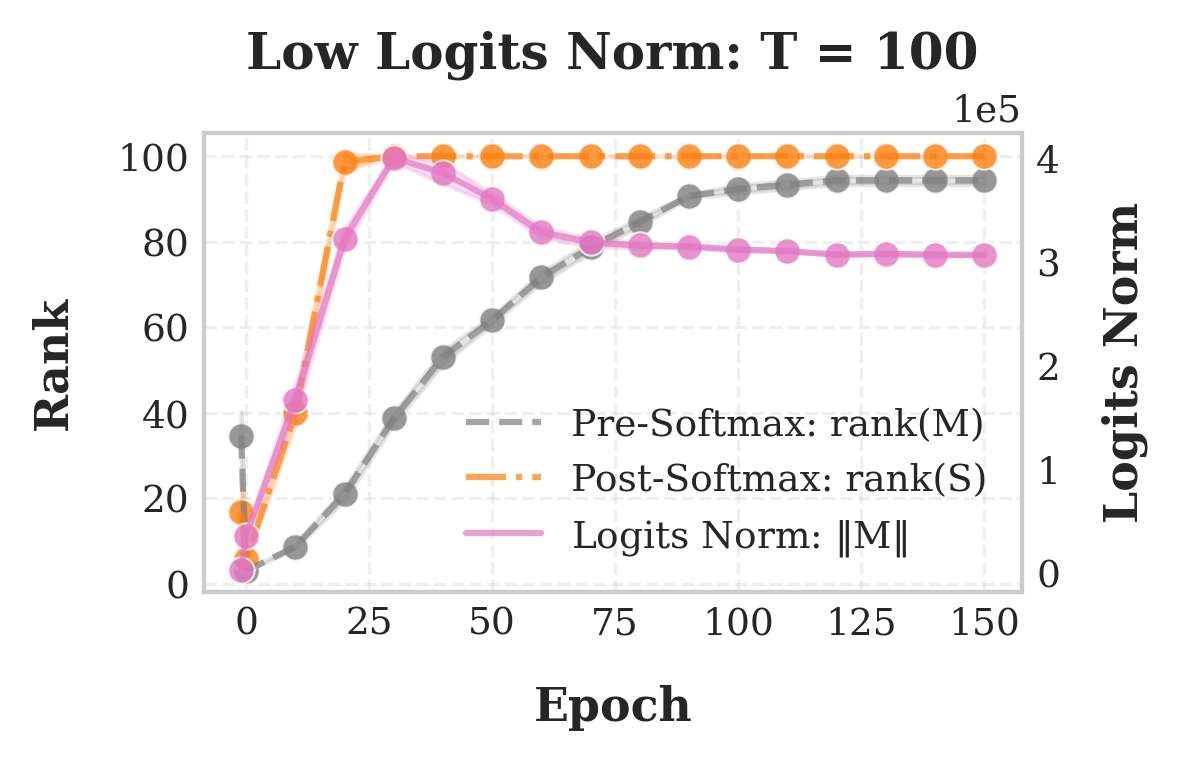}
    }
    {
    \includegraphics[width=0.48\textwidth]{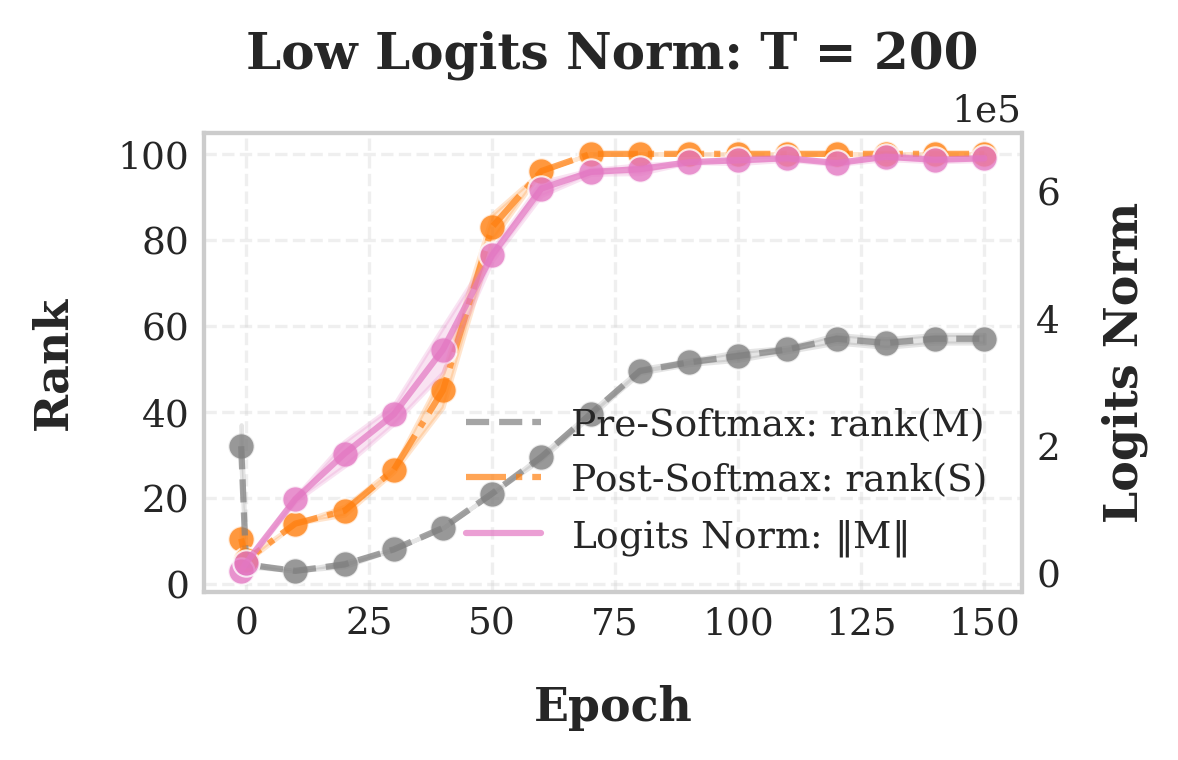}
    }
    \caption{\small{Plot presents the evolution of logits norm and its impact on pre and post \softmax rank leading to \textit{rank-deficit bias}. When training with high temperatures, the \textcolor{orange}{post-softmax rank growth} is triggered by \textcolor{appendix_pink}{the growth of the logits norm} not the \textcolor{gray}{pre-softmax rank}. Experiment: ResNet-18 trained on CIFAR-100.}} 
\end{figure}

\begin{figure}[!h]
    \centering
    {
    \includegraphics[width=0.48\textwidth]{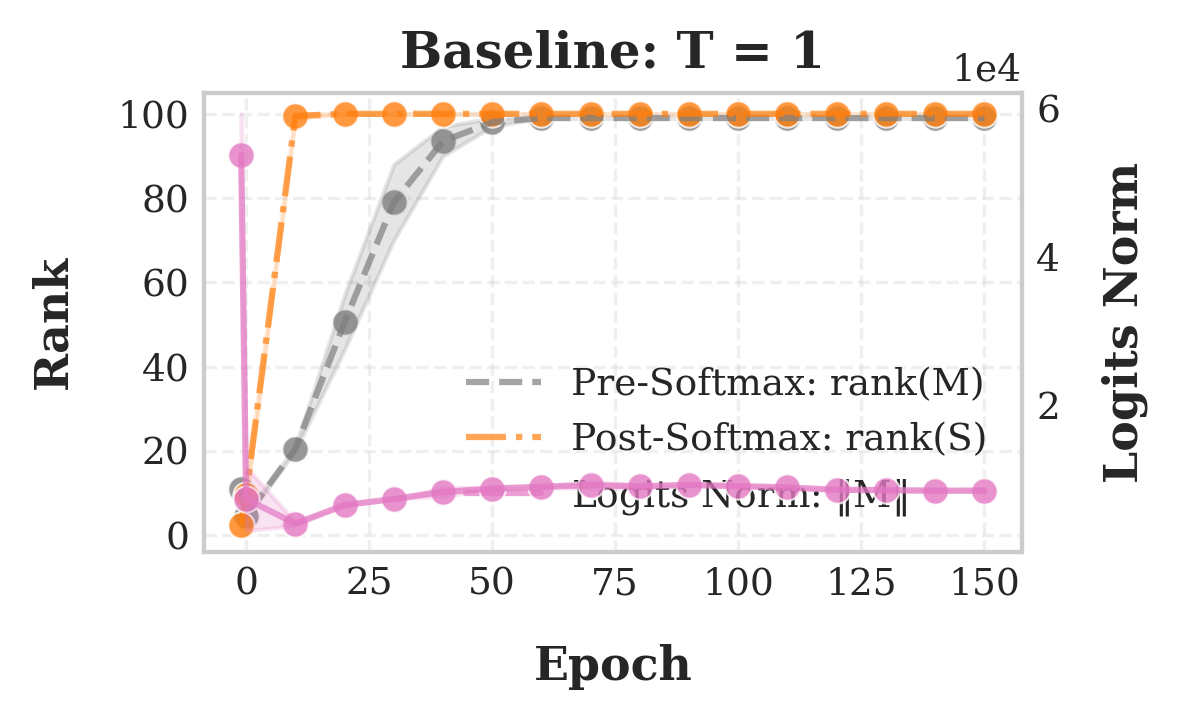}
    }
    {
    \includegraphics[width=0.48\textwidth]{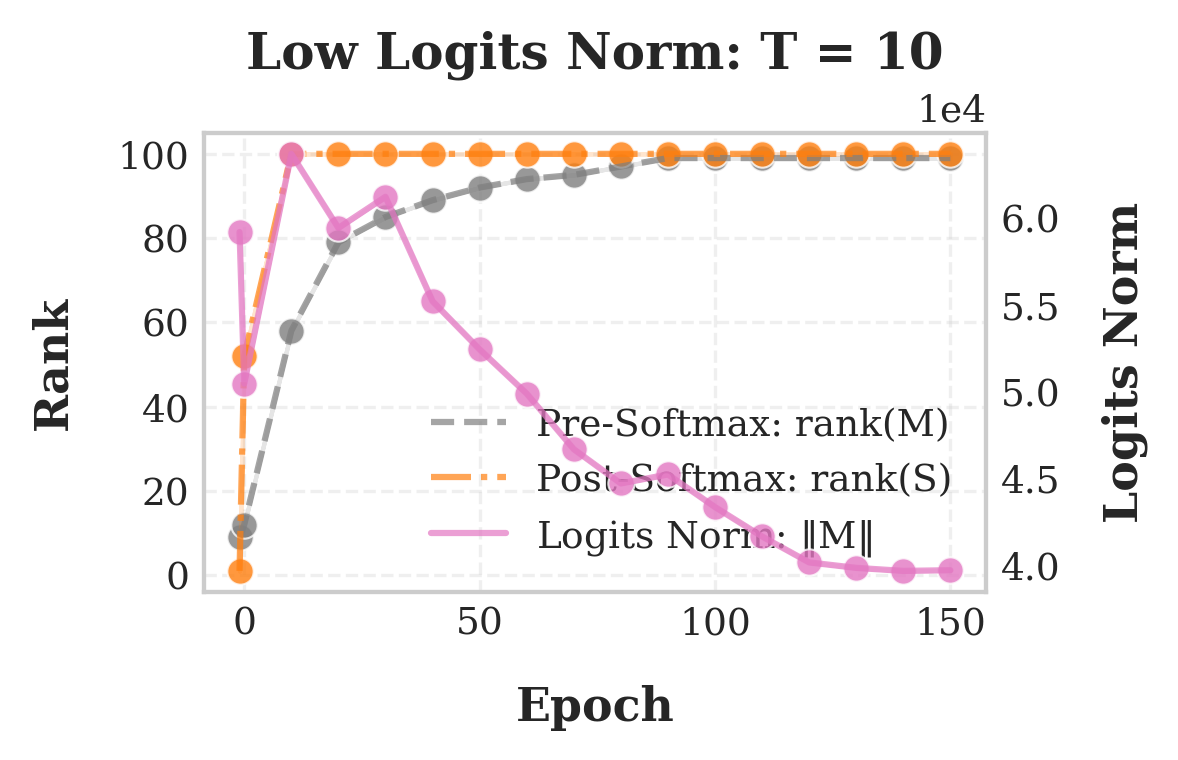}
    }
        {
    \includegraphics[width=0.48\textwidth]{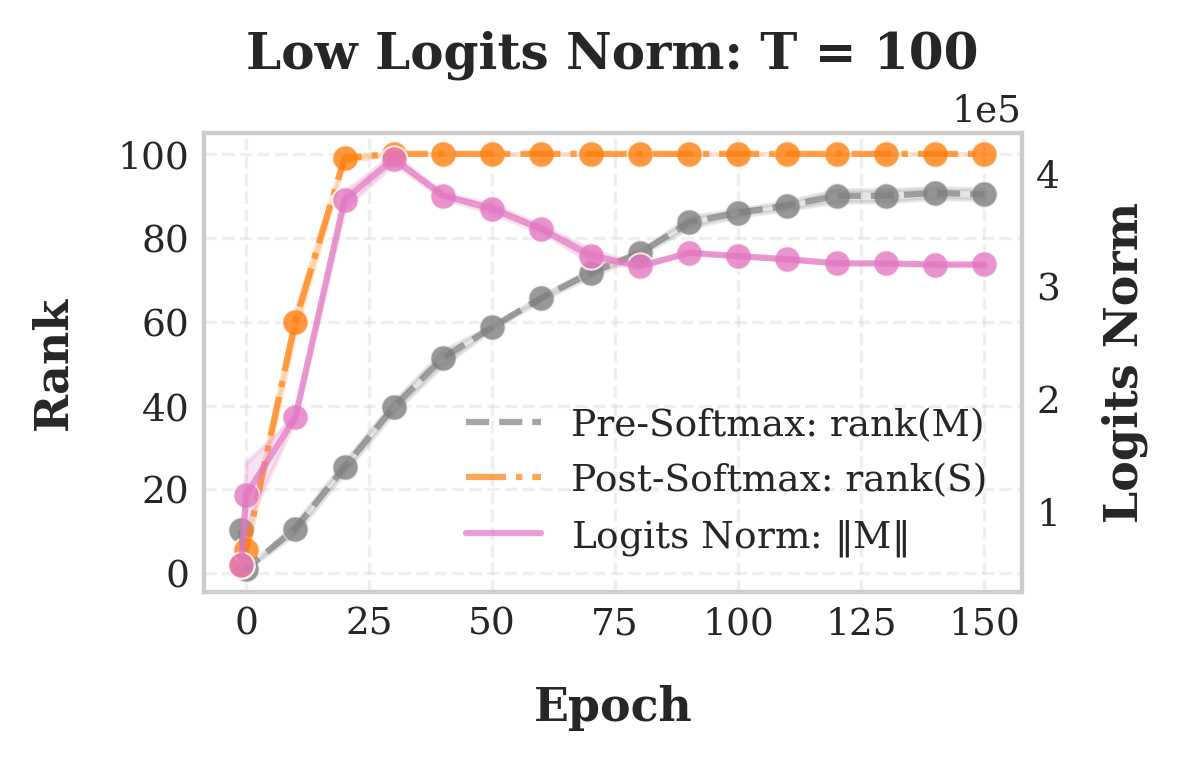}
    }
    {
    \includegraphics[width=0.48\textwidth]{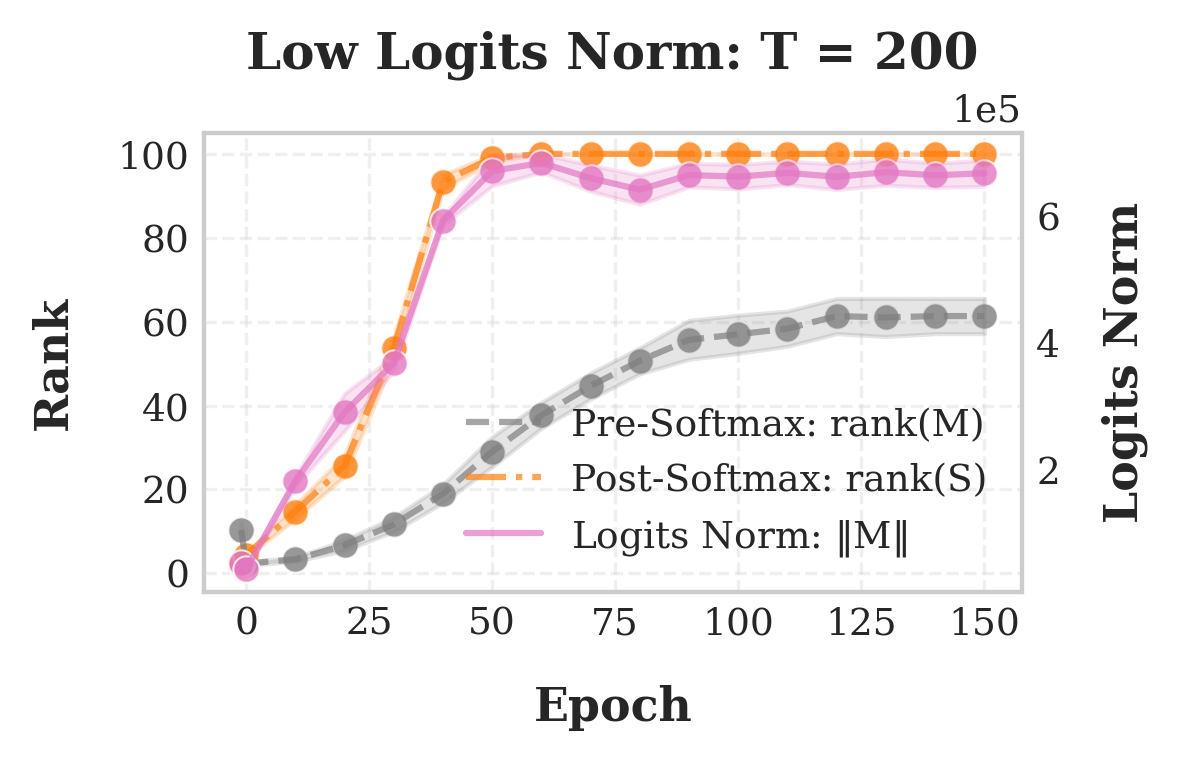}
    }
    \caption{\small{Plot presents the evolution of logits norm and its impact on pre and post \softmax rank leading to \textit{rank-deficit bias}. When training with high temperatures, the \textcolor{orange}{post-softmax rank growth} is triggered by \textcolor{appendix_pink}{the growth of the logits norm} not the \textcolor{gray}{pre-softmax rank}. Experiment: ResNet-34 trained on CIFAR-100.}}
\end{figure}

\begin{figure}[!h]
    \centering
    {
    \includegraphics[width=0.49\textwidth]{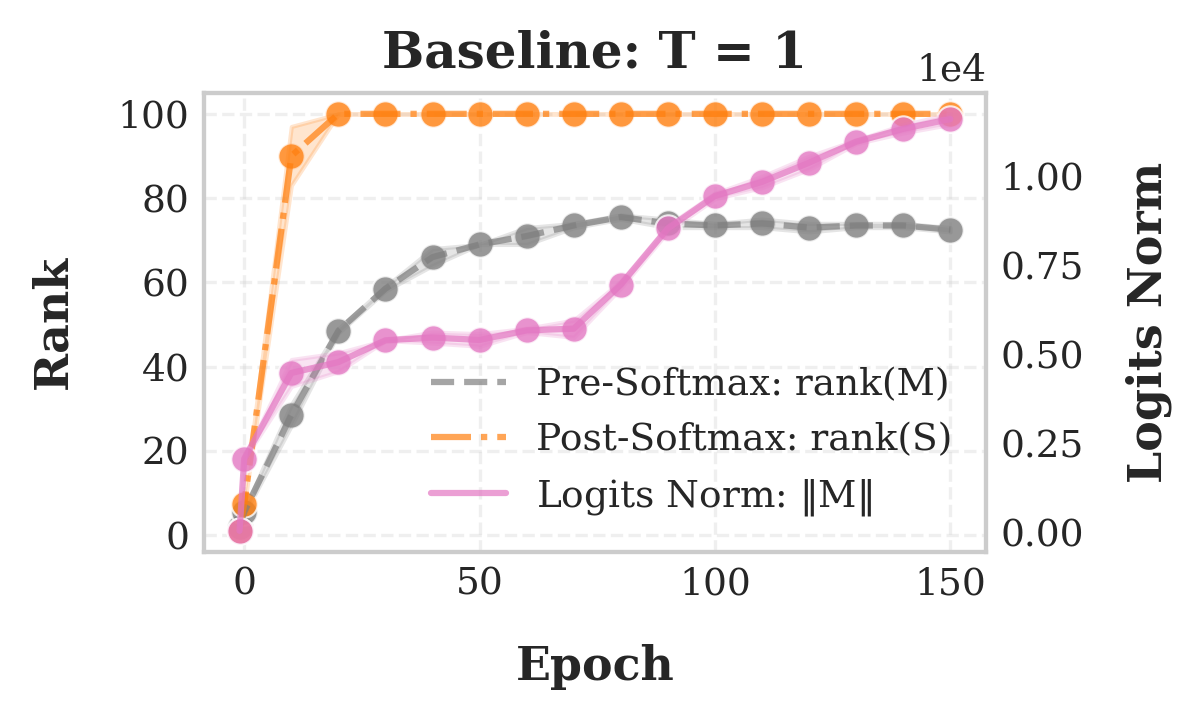}
    }
    {
    \includegraphics[width=0.49\textwidth]{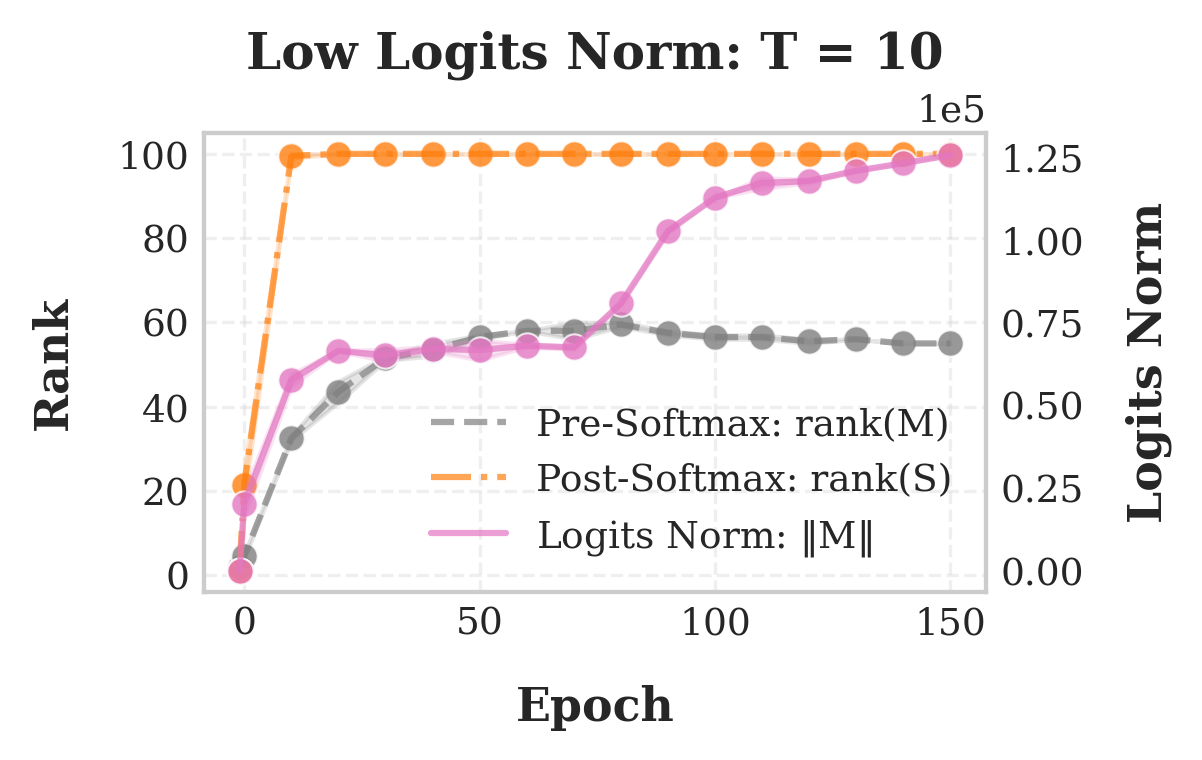}
    }
        {
    \includegraphics[width=0.49\textwidth]{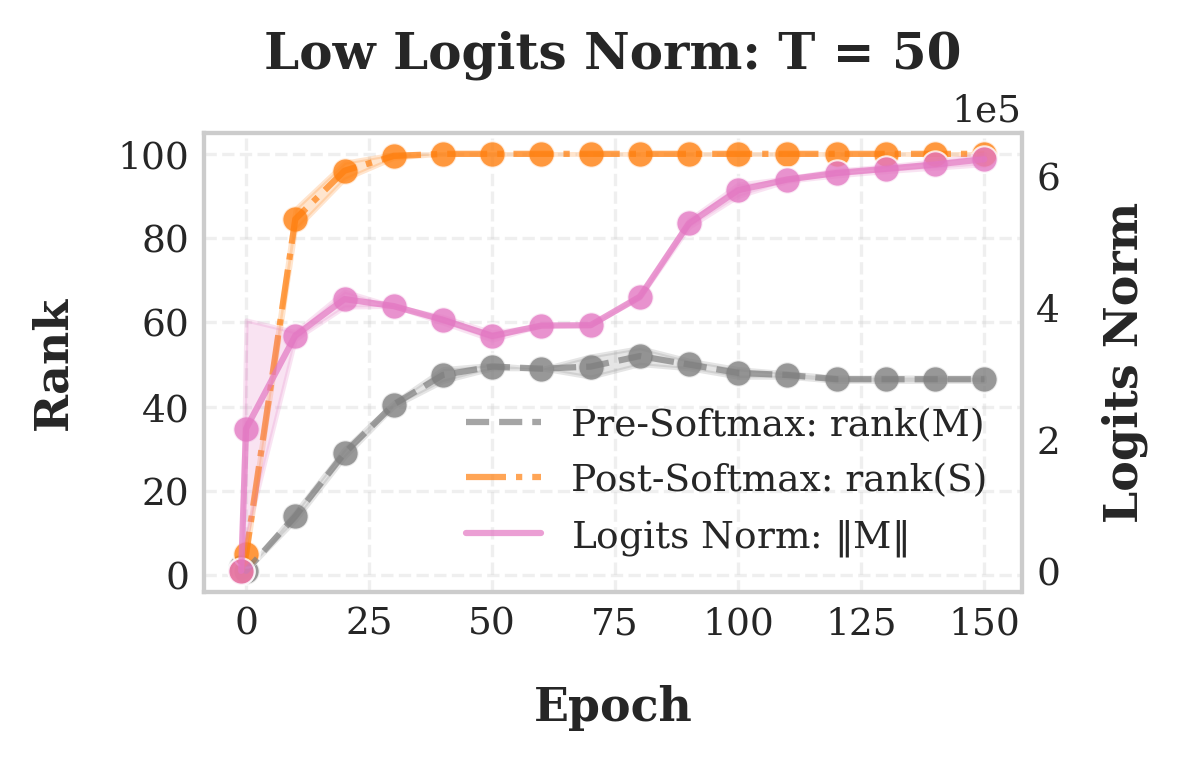}
    }
    {
    \includegraphics[width=0.49\textwidth]{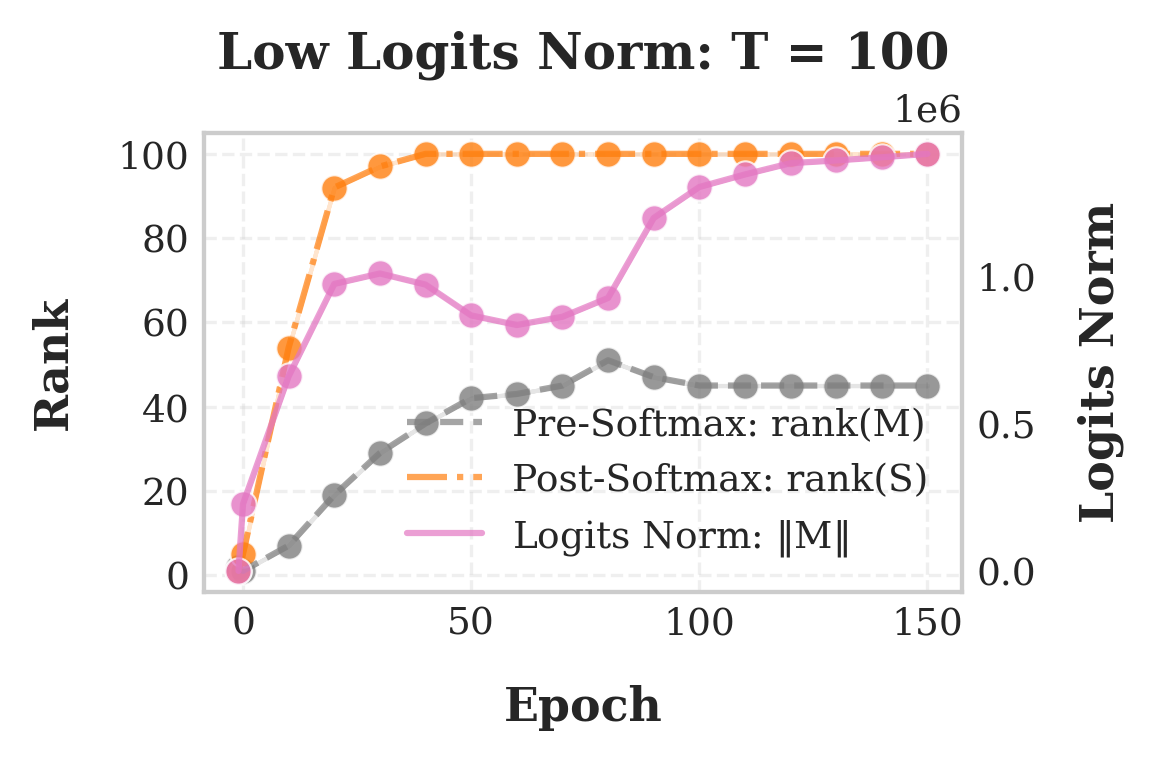}
    }
    \caption{\small{Plot presents the evolution of logits norm and its impact on pre and post \softmax rank leading to \textit{rank-deficit bias}. When training with high temperatures, the \textcolor{orange}{post-softmax rank growth} is triggered by \textcolor{appendix_pink}{the growth of the logits norm} not the \textcolor{gray}{pre-softmax rank}. Experiment: VGG-19 trained on CIFAR-100.}}
\end{figure}

\section{Neural Collapse}\label{app:neural_collapse}

Neural Collapse (NC) is a phenomenon observed in the terminal phase of training (TPT) of deep neural networks, where the learned representations and classifier weights exhibit a highly symmetric and simplified structure \cite{papyan2020neuralcollapse}. This behavior emerges when the model approaches zero training error, leading to the following four key conditions:

\begin{itemize}
    \item \textbf{NC1: Within-Class Feature Collapse}: The activations of samples from the same class converge to their class mean. Formally, for class $c$ with $N_c$ samples, the features $h_i^c$ in the penultimate layer satisfy $\|h_i^c - \mu_c\| \to 0$, where $\mu_c$ is the class mean.
    
    \item \textbf{NC2: Equiangular Class Means}: The class means $\mu_c$ become maximally separated and equiangular, forming a simplex equiangular tight frame (ETF). For $C$ classes, this means $\langle \mu_i, \mu_j \rangle = -\frac{1}{C-1}$ for $i \neq j$.
    
    \item \textbf{NC3: Classifier-Weight Alignment}: The classifier weights $w_c$ align with the class means, satisfying $w_c \propto \mu_c$.
    
    \item \textbf{NC4: Simplified Decision Boundaries}: The resulting decision boundaries become symmetric and equidistant, with the classifier behaving like a nearest class mean (NCM) decision rule.
\end{itemize}

\subsection{Intermediate Neural Collapse}
Recent work has extended the study of neural collapse beyond the final layer. \cite{parker2023neuralcollapseintermediatehidden,rangamani23intermediate} demonstrated that some degree of NC emerges in intermediate hidden layers across various architectures, with the degree of collapse typically increasing with layer depth. Key observations include:

\begin{itemize}
    \item Intra-class variance reduction occurs primarily in shallower layers
    \item Angular separation between class means increases consistently with depth
    \item Simple datasets may only require shallow layers to achieve collapse, while complex ones need the full network
\end{itemize}

However, as noted in~\cite{rangamani23intermediate}, not all architectures exhibit intermediate collapse uniformly. The conditions under which intermediate NC occurs remain an open question, requiring further study into architectural choices, optimization dynamics, and dataset characteristics.

\subsection{Assumptions and Limitations}
The neural collapse phenomenon comes with several important assumptions and limitations:

\paragraph{Training Phase Requirements}
NC typically emerges during the Terminal Phase of Training (TPT), when models achieve 100\% training accuracy. Achieving TPT often requires specific hyperparameter choices (e.g., extended training, particular learning rates) that may not correspond to those maximizing validation performance~\cite{papyan2020neuralcollapse}

\paragraph{Class Balance}
The original NC formulation assumed balanced classes, where each class has equal representation. Recent work~\cite{hong2024neural} has extended this to imbalanced settings, showing that:
\begin{itemize}
    \item Feature collapse still occurs within classes
    \item Class mean angles become dependent on class sizes
    \item Minority collapse (multiple minority classes collapsing to a single point) can occur below a certain sample size threshold
\end{itemize}

\paragraph{Data Augmentation}
Strong data augmentations often prevent models from reaching TPT, as they effectively create a harder optimization problem. This makes NC less likely to emerge in heavily augmented training regimes~\cite{papyan2020neuralcollapse}.

\paragraph{Regularization Effects}
Weight decay appears necessary for NC to emerge clearly. The unconstrained feature model (UFM) with cross-entropy loss and spherical constraints has been shown to provably lead to NC solutions, highlighting the role of implicit regularization \cite{ji2022an}.

\subsection{Neural Collapse finds solution of rank C-1}

\begin{proposition}\label{prop:nc_rank}
    Network $f_{\theta}$ exhibiting Neural Collapse on a dataset with $C$ classes has a solution rank equal to $C-1$. 
\end{proposition}

\begin{proof}
Under Neural Collapse (NC), we analyze the rank through three key properties:

\noindent\textbf{Step 1: Class Means Form ETF (NC1-2).} Let $\Km_C := [\mu_1, ..., \mu_C]^\top \in \mathbb{R}^{C \times d}$ be the class means matrix. By NC2, these means form a simplex equiangular tight frame (ETF) satisfying:
\begin{equation}
    \Km_C\Km_C^\top = \frac{C}{C-1}\mathbf{I}_C - \frac{1}{C-1}\mathbbm{1}\mathbbm{1}^\top
\end{equation}
This Gram matrix has rank $C-1$. Thus, $\text{rank}(\Km_C) = C-1$.

\noindent\textbf{Step 2: Classifier Alignment (NC3).} The classifier weights $\Wm = [w_1, ..., w_C]^\top$ satisfy $w_c \propto \mu_c$ from NC3. Therefore:
\begin{equation}
    \Wm = \alpha\Km_C \quad\text{for some } \alpha > 0
\end{equation}
This proportionality preserves the rank: $\text{rank}(\Wm) = \text{rank}(\Km_C) = C-1$.

\noindent\textbf{Step 3: Activation Matrix Structure.} Let $\Am \in \mathbb{R}^{d \times N}$ contain penultimate layer activations. By NC1, all examples collapse to their class means:
\begin{equation}
    \Am = \Km_C^\top \Sm
\end{equation}
where $\Sm \in \{0,1\}^{C \times N}$ is a binary selection matrix indicating class membership. Since $\Sm$ has full row rank for balanced classes, $\text{rank}(\Am) = \text{rank}(\Km_C) = C-1$.

\noindent\textbf{Step 4: Logit Matrix Decomposition.} The logits matrix $\Mm = \Wm\Am$ becomes:
\begin{align}
    \Mm &= \alpha\Km_C(\Km_C^\top\Sm) \\
        &= \alpha(\Km_C\Km_C^\top)\Sm
\end{align}
Using matrix rank properties:
\begin{equation}
    \text{rank}(\Mm) \leq \min\{\text{rank}(\Wm), \text{rank}(\Am)\} = C-1
\end{equation}
Since $\Km_C\Km_C^\top$ has rank $C-1$ and $\Sm$ has full row rank, the product maintains rank $C-1$. Thus, $\text{rank}(\Mm) = C-1$.
\end{proof}

\section{Theoretical analysis of softmax properties}\label{app:softmax_analysis}

This section aims to understand how applying softmax on a matrix $\Am$ column-wise changes the spectrum of the matrix and its rank. To this end, we apply the following theorems:

\begin{theorem}[Gershgorin Circle Theorem~\cite{gershgorin31}] \label{theorem:gershgorin}
Every eigenvalue of any real, symmetric matrix $\Km$ lies within at least one of the Gershgorin discs $D(k_{ii}, R_i)$, where $R_i = \Sigma_{j \neq i} |k_{ij}|$.
\end{theorem}

\begin{proposition} \label{proposition:B2_appendix}
For any matrix $\Sm\in\R^{n\times n}$ with each column $\m_j\in\R^n$ as a probability vector. Then the gap between the largest $\sigma_1(\Sm)$ and the smallest singular value $\sigma_n(\Sm)$ is bounded by the following tight inequality:
\[
0 \leq \sigma_1(\Sm) - \sigma_n(\Sm) \leq \sqrt{1+r}-\sqrt{\max\left\{\frac{1}{n}-r,0\right\}}
\]
where $r:= \max_i \sum_{j\neq i}\langle \s_i,\s_j\rangle$.
\end{proposition}
\begin{proof}
    Consider the matrix $\Gm:=\Sm^\top\Sm$. By Jensen's inequality and Cauchy-Schwartz inequality, we have 
    \[
    \frac{1}{n}
    =
    n\left(\frac{\sum_{i=1}^n\Sm_{ij}}{n}\right)^2
    \leq 
    \Gm_{jj} = \sum_{i=1}^n \Sm_{ij}^2 
    \leq 
    \left(\sum_{i=1}^n \Sm_{ij}\right)^2 = 1.
    \]
    By definition, $R_j:= \sum_{k\neq j}|\langle \s_k,\s_j\rangle|=\sum_{k\neq j}\langle \s_k,\s_j\rangle= \sum_{k\neq j}\Gm_{kj}$ is the radius of the $j$-th Gershgorin disc. Hence, Theorem \ref{theorem:gershgorin}, the eigenvalues of $\Gm$ lie on $\cup_j [\Gm_{jj}-R_j,\Gm_{jj}+R_j]\subset \left[\max\left\{\frac{1}{n}-r,0\right\},1+r\right]$, since $\Gm$ is positive semidefinite.
    Note that $\sigma_1(\Sm)=\sqrt{\lambda_{\max}(\Gm)}$ and $\sigma_n(\Sm)=\sqrt{\lambda_{\min}(\Gm)}$, we obtain the bound. The bound is tight by considering $\Sm=\mathbf{I}_n$ for the lower bound and  $\Sm=(\mathbf{1},0,...,0)$ for the upper bound.
\end{proof}

Intuitively, Proposition \ref{proposition:B2_appendix} shows that reducing the inner products between columns of $\Sm = \mathtt{softmax}(\Am)$ decreases the gap between its largest and smallest singular values. In particular, the numerical rank of the post-softmax matrix $\Mm$ can remain high even if the pre-softmax matrix $\Am$ is of low rank. Figure \ref{fig:inner_product_and_rank} illustrates that the smallest spectral gap occurs at an intermediate temperature leading to the highest post-softmax rank.

\begin{figure}[!h]
    \centering
    \includegraphics[width=\linewidth]{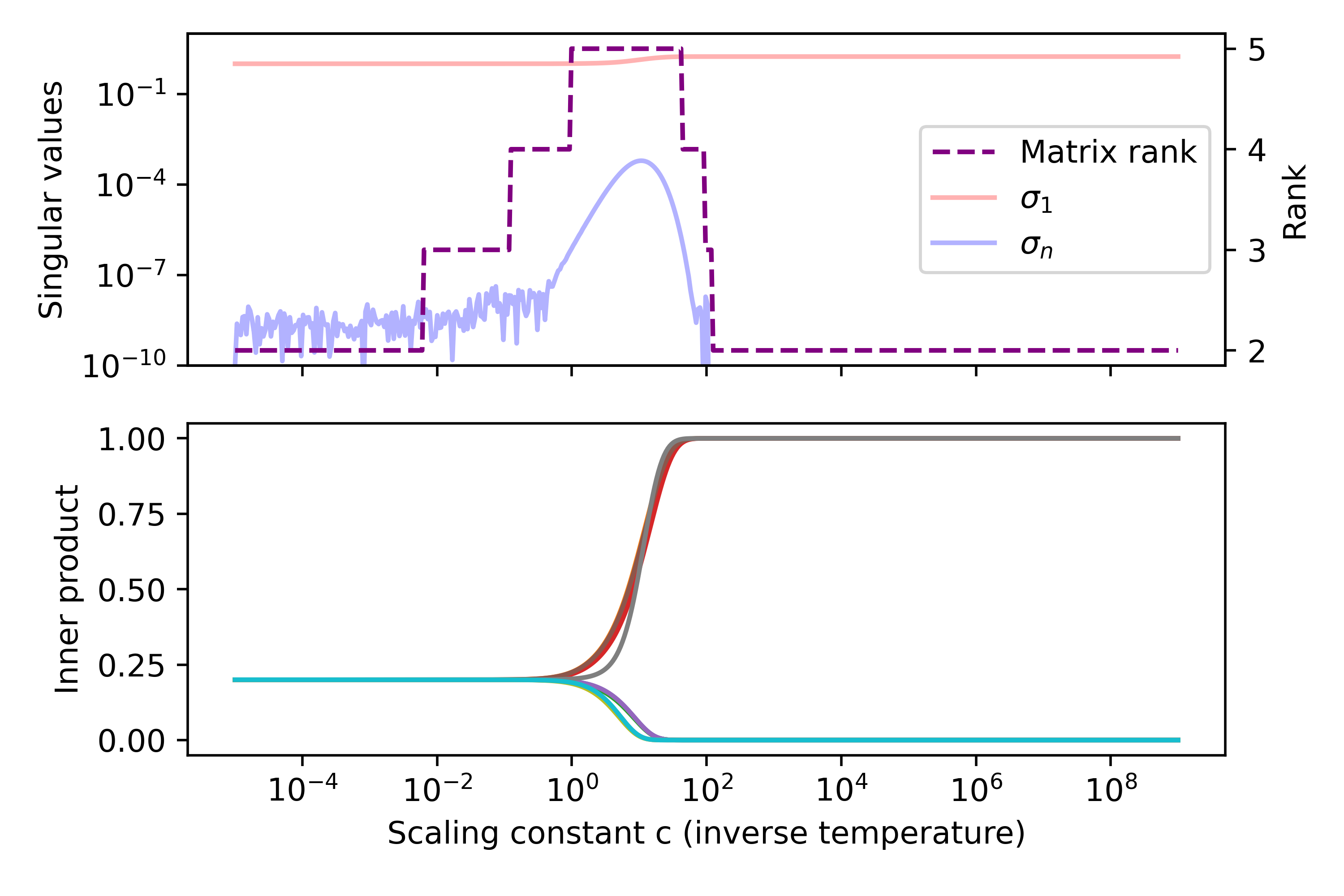}
    \caption{Applying \softmax column-wise on rank-1 matrix $\Am \in \mathbb{R}^{5\times5}$ with various temperatures decreases the inner products between the columns with different indices of the highest element and increases the inner products between the columns with the same indices (bottom). The temperature at which the bifurcation of inner products happens aligns with the temperature that shrinks the gap between the top and bottom singular values of the matrix and increases its rank (top). }
    \label{fig:inner_product_and_rank}
\end{figure}

\section{Theoretical limits of \textit{rank-deficit bias}}\label{app:rank2_to_full_rank}

Given the importance of the \textit{rank-deficit bias} discussed in the main paper and its great variability across different architectures and datasets, we ask ourselves the question: What is the theoretical limit of \textit{rank-deficit bias}? The following proposition shows that in theory, there exists a logits matrix $\Bm \in \mathbb{R}^{n \times n}$ of rank 2 that after applying a \softmax recovers full rank n. In other words, neural networks could successfully solve a classification task with $n$ classes by finding solutions for rank 2. This result shows that the solutions found by training current models with high temperatures are still far from the theoretical limit.

\begin{proposition}\label{prop:full_rank}
Let $n \geq 2$. For almost every random matrix $\Am \in \mathbb{R}^{n \times 2}$ with i.i.d.\ $\mathcal{N}(0, 1)$ entries, there exists a scaling parameter $c > 0$ such that $\Bm \coloneqq \tilde{\Am}\tilde{\Am}^\top$ satisfies:
\[
\rank(\Bm) = 2 \quad \text{and} \quad \rank(\operatorname{softmax}(c\Bm)) = n,
\]

where $\tilde{\Am}$ is the row-normalized version of $\Am$, and $\operatorname{softmax}$ denotes row-wise softmax.
\end{proposition}

\begin{proof}
\textbf{Step 1: Construction of $\Bm$ with rank 2.}\\
Let $\Am \in \mathbb{R}^{n \times 2}$ have i.i.d. standard normal entries. Define its row-normalized version:
\[
\tilde{\Am}_{i,:} \coloneqq \frac{\Am_{i,:}}{\|\Am_{i,:}\|_2} \quad \text{for } i = 1,\dots,n.
\]
Let $\Bm \coloneqq \tilde{\Am}\tilde{\Am}^\top \in \mathbb{R}^{n \times n}$. Since $\tilde{\Am}$ has rank at most 2, $\rank(\Bm) \leq 2$. Moreover, with probability 1, $\Am$ has full column rank, which implies $\rank(\tilde{\Am}) = 2$ and consequently $\rank(\Bm) = 2$.

\medskip
\textbf{Step 2: Distinctness of off-diagonal entries.}\\
For $i \neq j$, the entries satisfy:
\[
b_{ij} = \langle \tilde{\a}_i, \tilde{\a}_j \rangle = \cos\theta_{ij},
\]
where $\theta_{ij}$ is the angle between $\tilde{\a}_i$ and $\tilde{\a}_j$. Since the rows of $\Am$ are i.i.d. Gaussian vectors, $\tilde{\a}_i$ and $\tilde{\a}_j$ are independent and uniformly distributed on the unit circle in $\mathbb{R}^2$. The probability that $\tilde{\a}_i = \pm\tilde{\a}_j$ is zero, hence $b_{ij} \neq \pm1$ almost surely for all $i \neq j$.

\medskip
\textbf{Step 3: Behavior under softmax as $c \to \infty$.}\\
The row-wise softmax is defined by:
\[
[\operatorname{softmax}(c\Bm)]_{ij} = \frac{e^{c b_{ij}}}{\sum_{k=1}^n e^{c b_{ik}}}.
\]
For each row $i$, as $c \to \infty$:
\begin{itemize}
\item The diagonal term $e^{c b_{ii}} = e^{c}$ dominates (since $b_{ii} = 1$)
\item All off-diagonal terms $e^{c b_{ij}} \to 0$ (since $b_{ij} < 1$)
\end{itemize}
Thus:
\[
\lim_{c \to \infty} \operatorname{softmax}(c\Bm) = \mathbf{I}_n.
\]

\medskip
\textbf{Step 4: Existence of finite $c$ achieving full rank.}\\
Since $\operatorname{softmax}(c\Bm) \to \mathbf{I}_n$ as $c \to \infty$, there exists some finite $c_0 > 0$ such that $\operatorname{softmax}(c\Bm)$ has rank $n$ for all $c \geq c_0$.

\medskip
\textbf{Step 5: Rank of identity matrix.}\\

Since $\lim_{c \to \infty} \operatorname{softmax}(c\Bm) = \mathbf{I}_n,$ for all $c \geq c_0$, we have 

\[
\rank(\operatorname{softmax}(c\Bm)) = \rank(\mathbf{I}_n) = n.
\]

\end{proof}

\section{Data orthogonality experiments}\label{app:NC5}

\begin{figure}[!h]
    \centering
    {
    \includegraphics[width=0.49\textwidth]{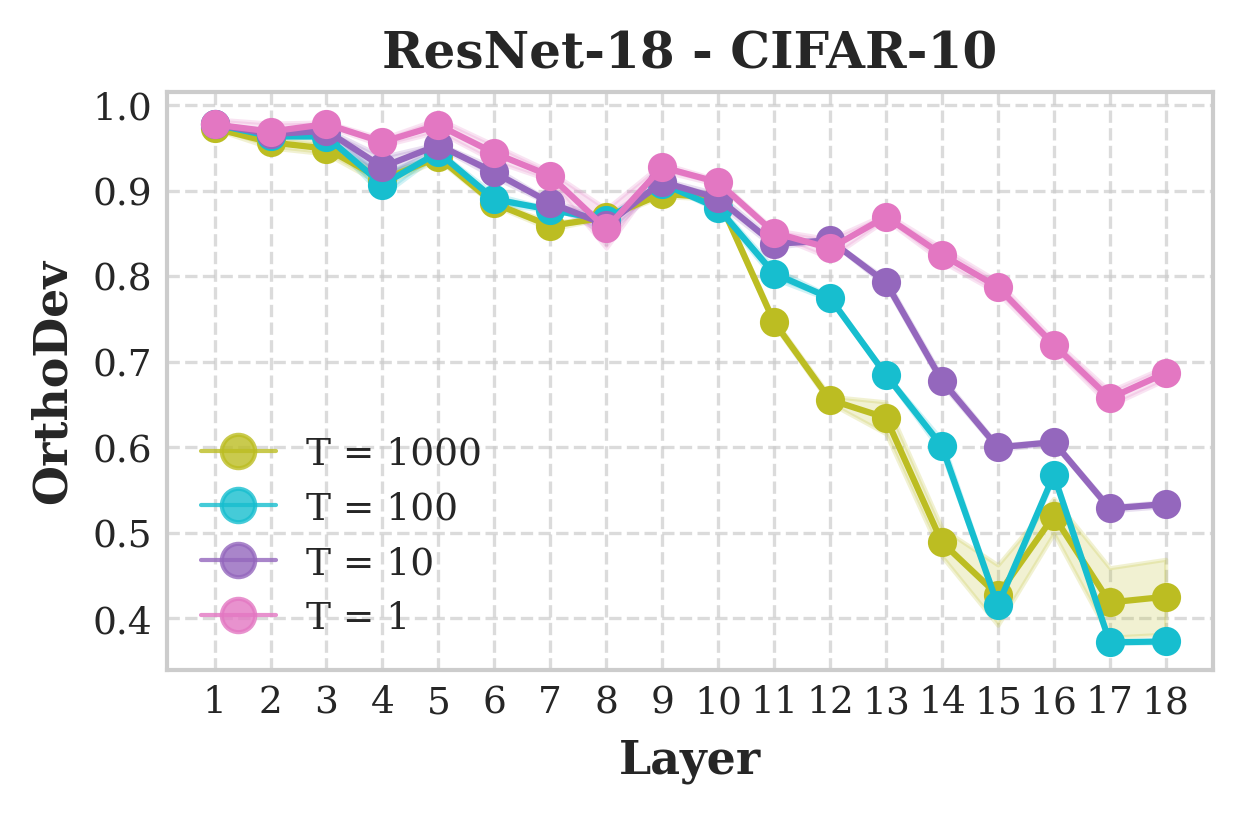}
    }
    {
    \includegraphics[width=0.49\textwidth]{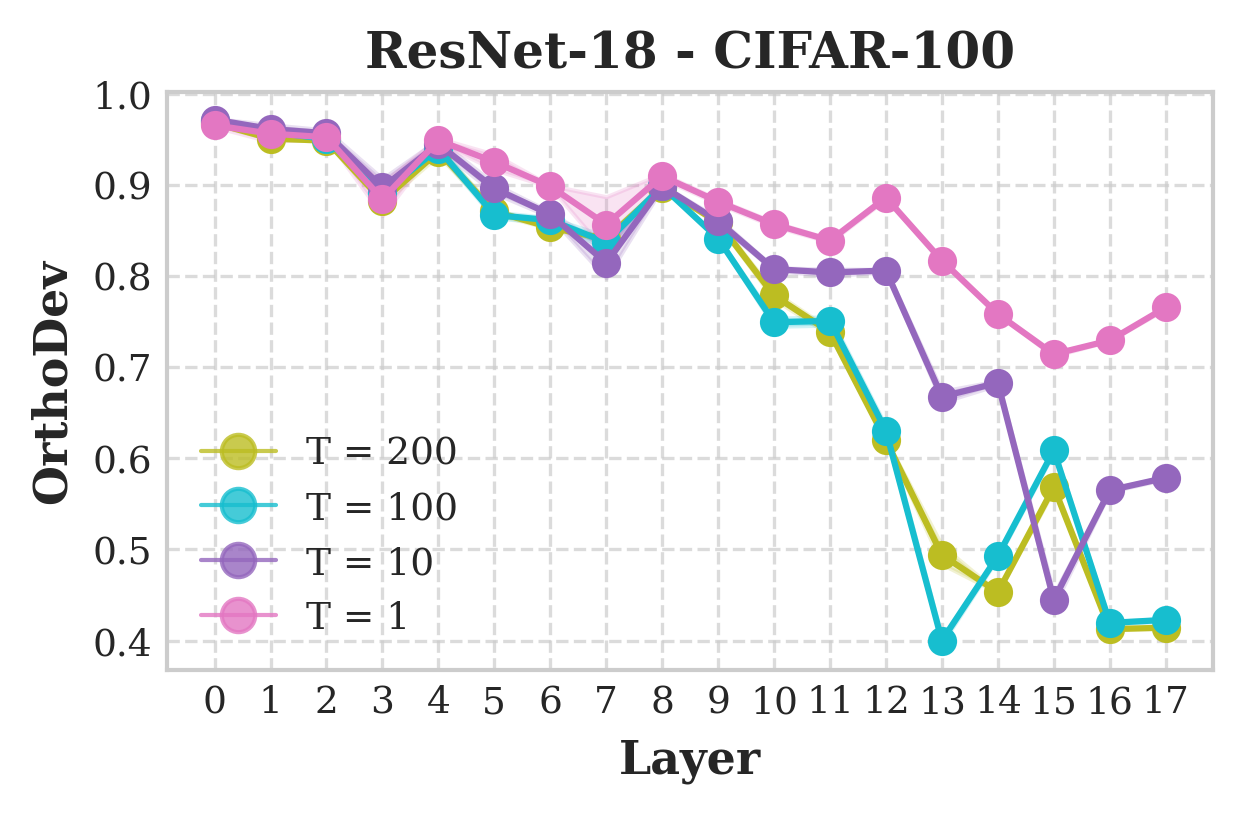}
    }
    {
    \includegraphics[width=0.49\textwidth]
    {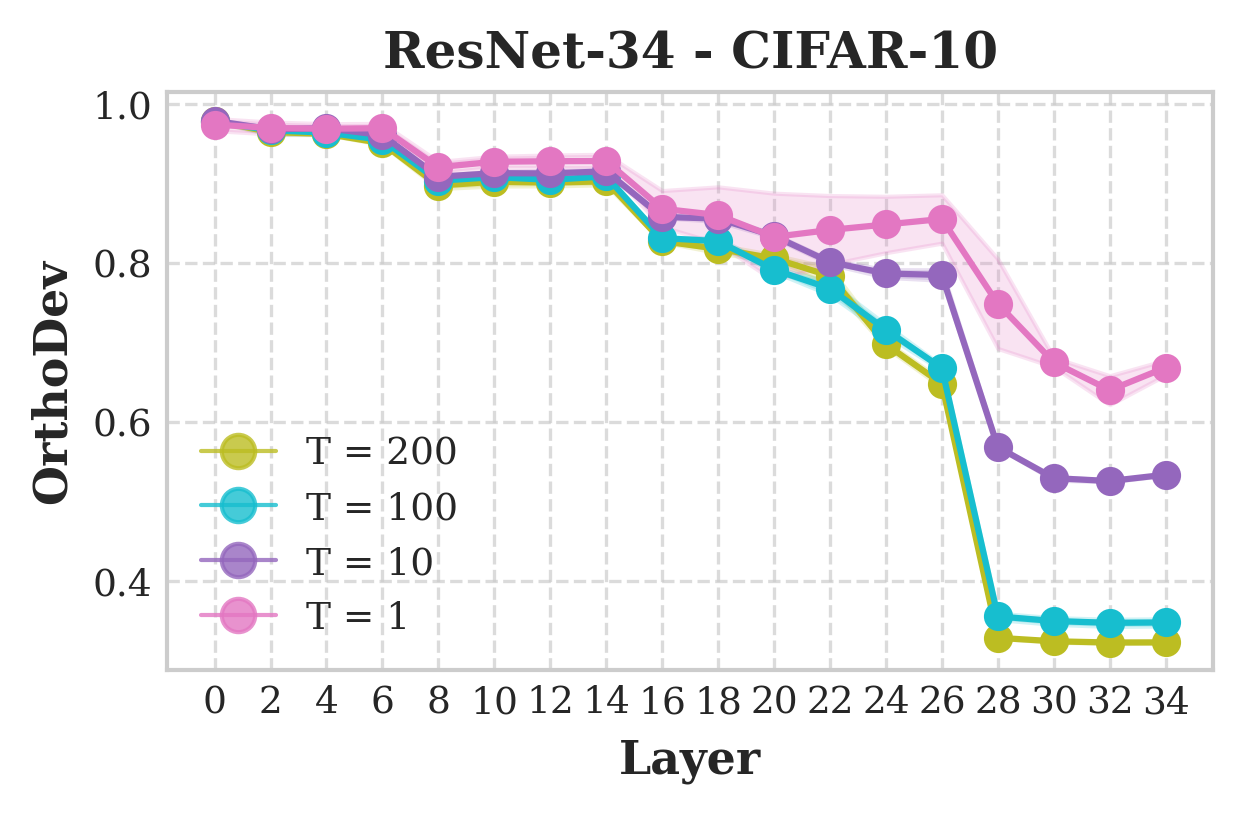}
    }
    {
    \includegraphics[width=0.49\textwidth]
    {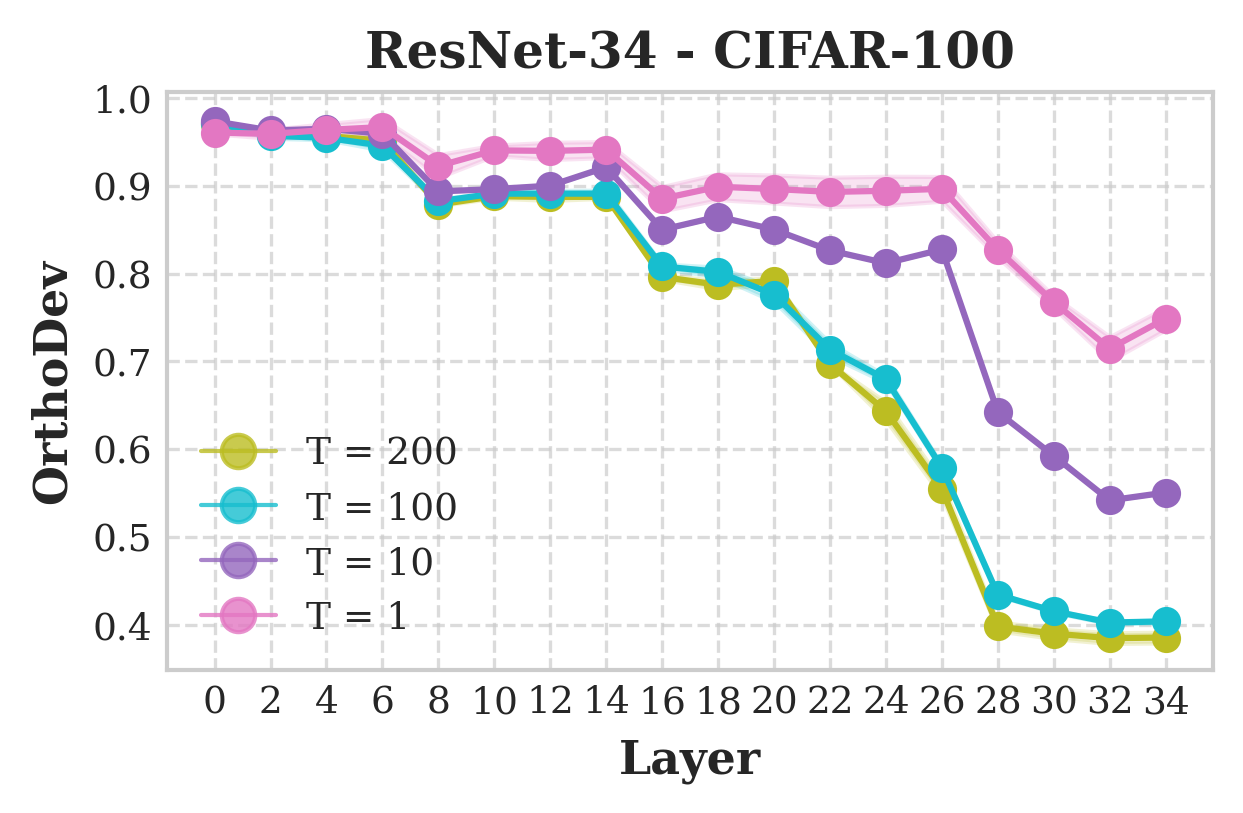}
    }
    \caption{Evolution of \texttt{OrthoDev} across layers for different temperatures. Consistently, networks trained with high temperature achieve low \texttt{OrthoDev} values.}
    \label{fig:orthodev_vs_temp}
\end{figure}

\newpage
\section{Generalization is at odds with detection}\label{app:ood_vs_generalization}

To examine the tension between out-of-distribution (OOD) generalization and OOD detection, we conduct a dedicated experiment in which multiple neural networks are evaluated on both tasks. OOD detection is assessed using the \texttt{NECO} method, while OOD generalization is measured via a linear probe applied to the penultimate layer representations. As shown in Figure~\ref{fig:generalization_vs_detection}, there is a clear linear relationship between OOD accuracy (y-axis) and OOD detection performance (x-axis), where improved generalization correlates with degraded detection. This relationship is governed by the \texttt{OrthoDev} metric, which, as demonstrated in Figure~\ref{fig:orthodev_vs_temp}, can be directly controlled by adjusting the temperature.

\begin{figure}[!h]
    \centering
    {
    \includegraphics[width=\textwidth]{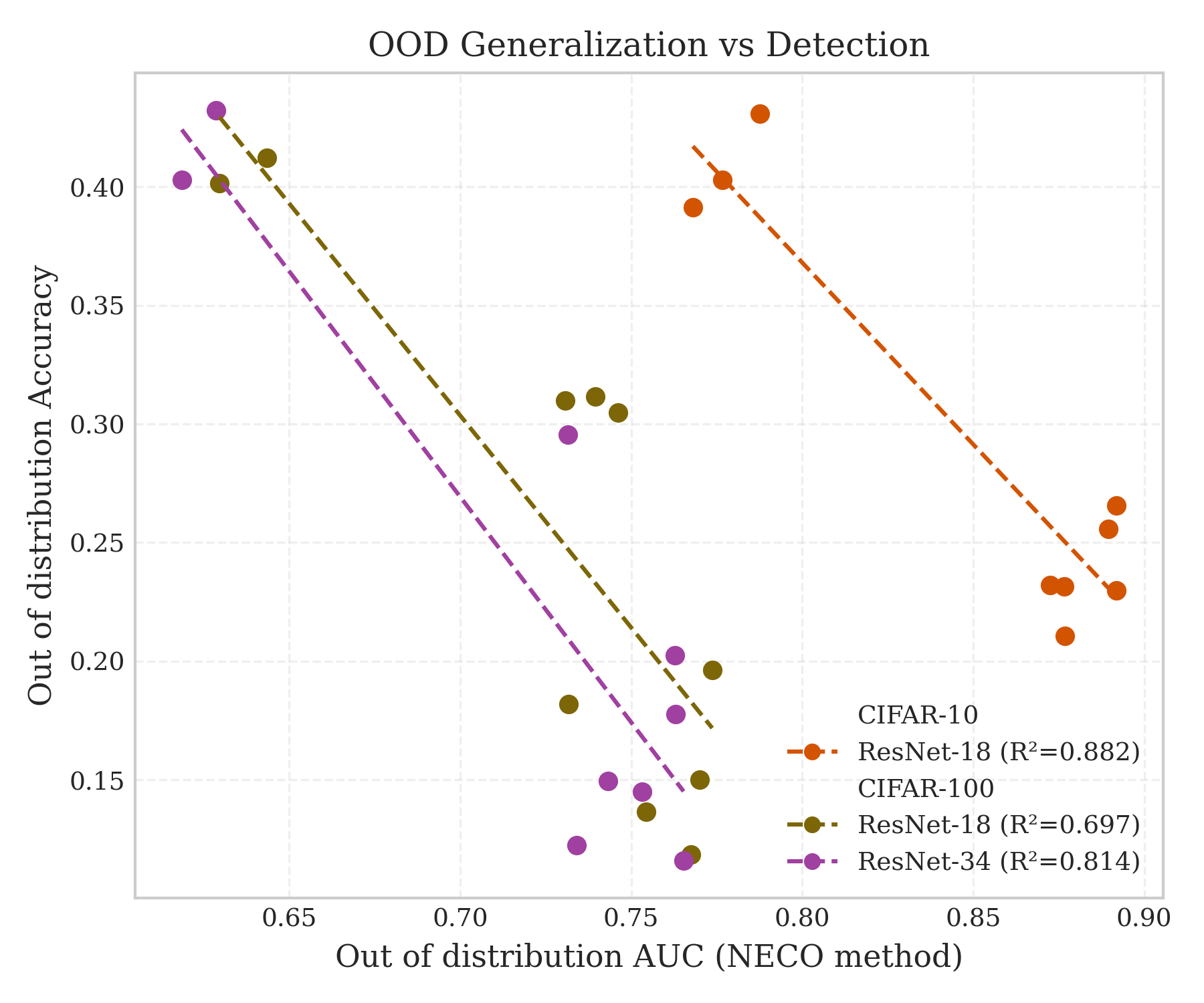}
    }
    \caption{OOD generalization is negatively correlated with OOD detection.}
    \label{fig:generalization_vs_detection}
\end{figure}

\section{Limitations of the work}

While our findings offer novel insights into the role of \softmax in representation learning, this work has several limitations that present valuable opportunities for future research.

\paragraph{Scope of Analysis:} Our study primarily focuses on supervised image classification, leaving open questions about how \emph{rank-deficit bias} manifests in other architectures and learning paradigms. For instance, investigating \softmax in intermediate Transformer layers could yield new insights into training stability and efficiency, particularly in NLP, where these layers share structural similarities with classification tasks but remain understudied in this context~\cite{wu2024linguisticcollapseneuralcollapse}. Similarly, self-supervised learning methods—many of which rely on \softmax-based losses (e.g., contrastive learning)—could benefit from our framework.

\paragraph{Theoretical Understanding:} A deeper theoretical analysis of the dynamics governing inner product evolution during training remains an important next step. While our empirical results highlight key trends, formalizing these observations could strengthen the theoretical foundations of our findings.

\paragraph{Hyperparameter Optimization:} Except for the ImageNet experiments, high-temperature models were trained using the same hyperparameters as baseline networks. While this controlled approach isolates the effect of temperature scaling, we anticipate that further performance gains could be achieved through systematic hyperparameter tuning tailored to high-temperature regimes. Exploring this direction remains an open research question.

\paragraph{Temperature Scheduling:} Our analysis examines fixed-temperature training, but dynamic temperature scheduling—potentially combining the benefits of both standard and high-temperature regimes—warrants further investigation as a means of optimizing model performance.

\section{Supplementary plots}
\begin{figure}[!h]
    \centering
    {
    \includegraphics[width=\textwidth]{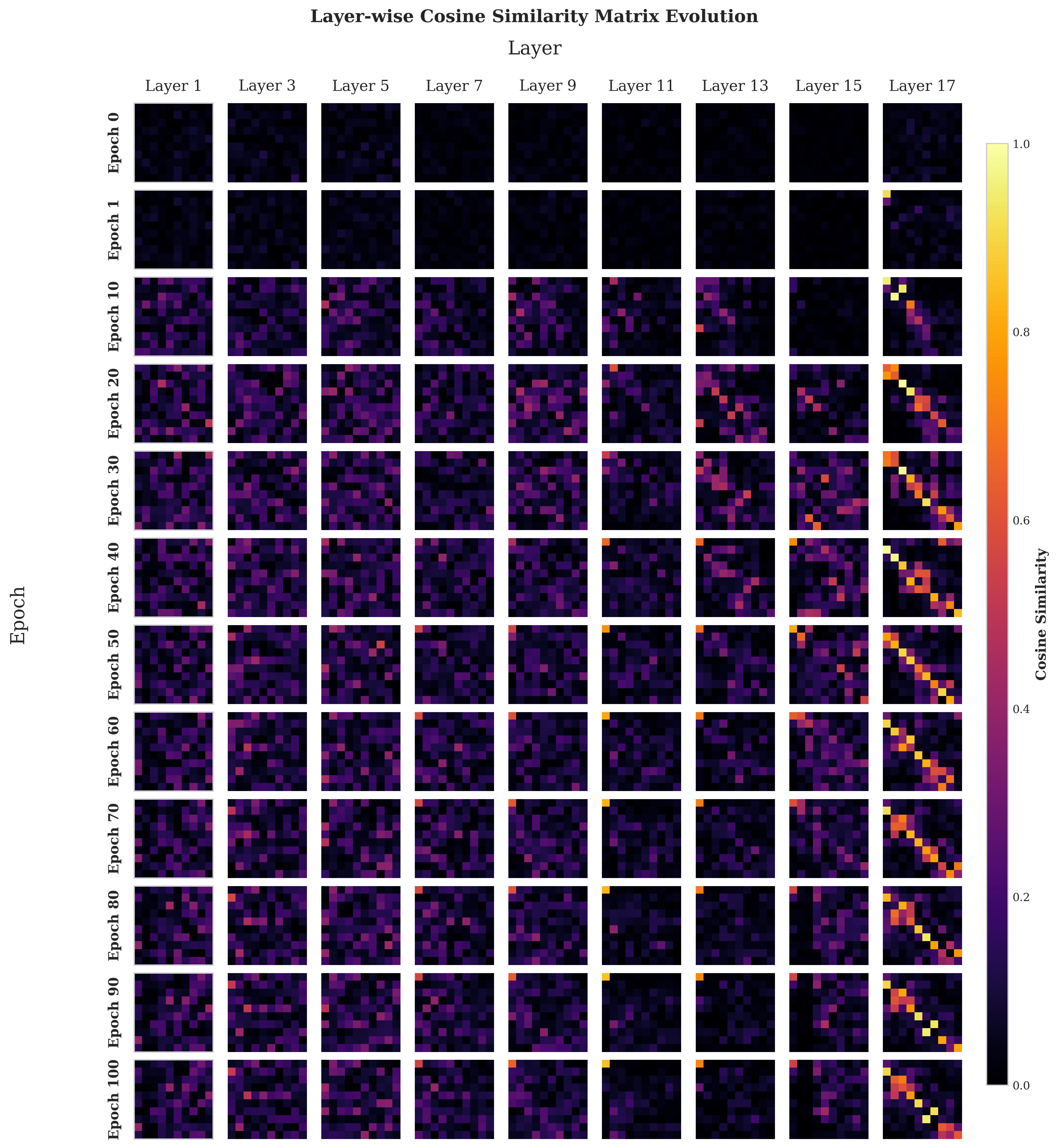}
    }
    \caption{\textbf{The alignment (cosine similarity) between singular vectors of weights and representations} of ResNet-18 trained on CIFAR-100 with high-temperature.} 
    \label{fig:finegrained_alignment_resnet_high}
\end{figure}

\begin{figure}[!h]
    \centering
    {
    \includegraphics[width=\textwidth]{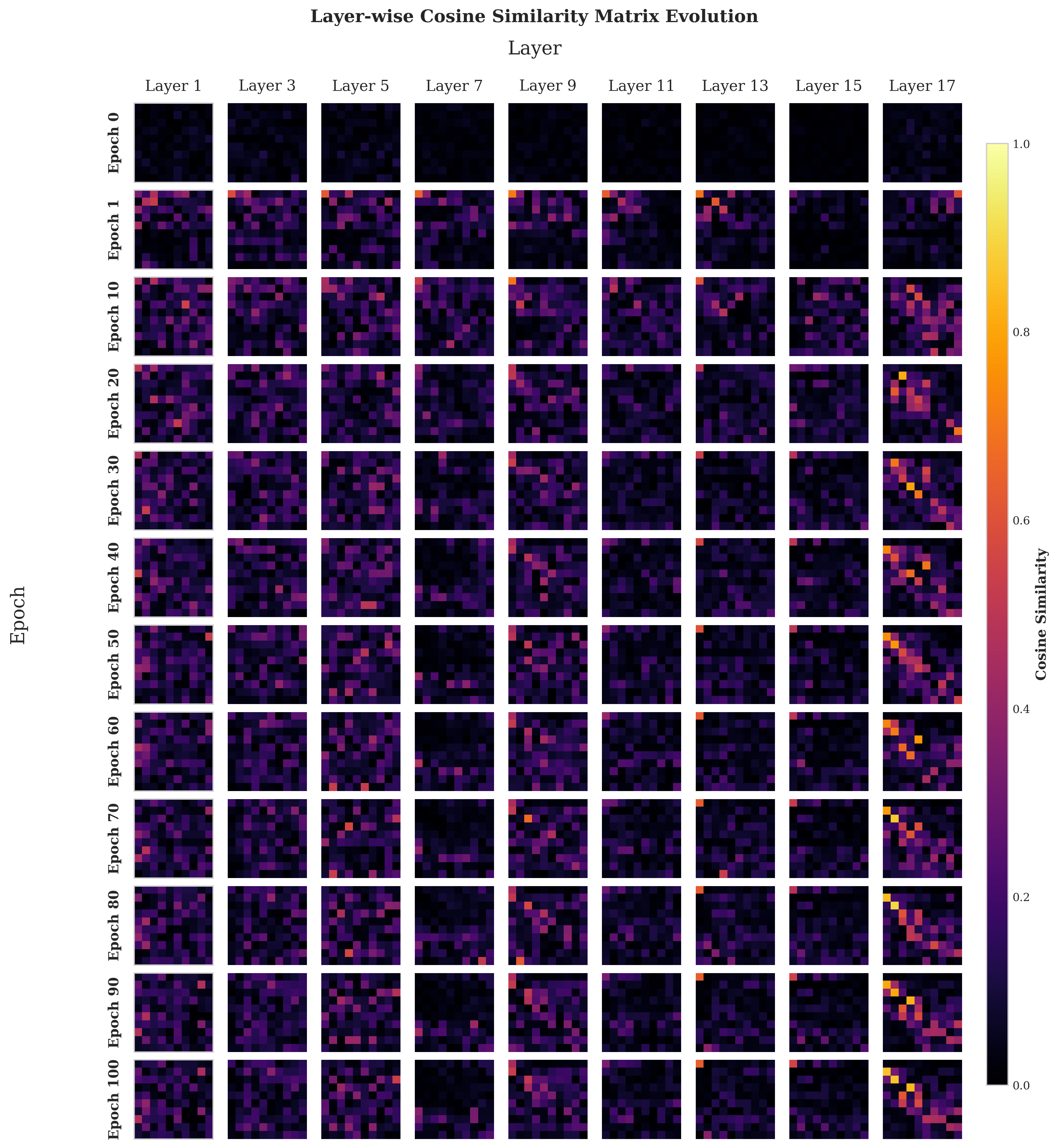}
    }
    \caption{\textbf{The alignment (cosine similarity) between singular vectors of weights and representations} of ResNet-18 trained on CIFAR-100 with temperature 1.} 
    \label{fig:finegrained_alignment_resnet_baseline}
\end{figure}

\end{document}